\documentclass{article} % For LaTeX2e
\usepackage{iclr2027_conference,times}
\usepackage{enumitem}
\usepackage{graphicx}
\usepackage{amsthm}
\usepackage{amssymb}
\usepackage{mathtools}
\usepackage{bbm}
\usepackage{booktabs}
\usepackage{amsmath,amsfonts,bm}

\def\eqref#1{equation~\ref{#1}}
\def\Eqref#1{Equation~\ref{#1}}
\def\1{\bm{1}}

\DeclareMathAlphabet{\mathsfit}{\encodingdefault}{\sfdefault}{m}{sl}
\SetMathAlphabet{\mathsfit}{bold}{\encodingdefault}{\sfdefault}{bx}{n}

\DeclareMathOperator*{\E}{\mathbb{E}}

\DeclareMathOperator*{\PROB}{\mathbb{P}}
\newcommand\expect[1]{\E\left[ #1 \right]}
\newcommand\expectwrt[2]{\E_{#2}\left[ #1 \right]}

\newcommand{\norm}[1]{\left\lVert#1\right\rVert}
\newcommand{\prob}[1]{\PROB\left(#1\right)}

\newcommand{\inner}[1]{\langle #1 \rangle}

\newcommand{\eg}{\textit{e.g.}}
\newcommand{\ie}{\textit{i.e.}}
\newcommand{\cf}{\textit{c.f.}}
\newcommand{\pistar}{\pi^{\star}}
\newcommand{\pirec}{\hat\pi_{\Omega, T}}
\newcommand{\frec}{f_T}
\newcommand{\fstar}{f^{*}}

\newcommand{\murec}{\hat{Q}_{\Omega,T}}
\newcommand{\Vrec}{\hat{V}_{\Omega, T}}
\usepackage{hyperref}
\usepackage{url}

\newtheorem{theorem}{Theorem}
\newtheorem{theorem_custom_num}{Theorem} % custom theorem counter for re-statements.
\newtheorem{lemma}[theorem]{Lemma} % Shares counter with 'theorem'
\newtheorem{proposition}[theorem]{Proposition} % Shares counter with 'theorem'
\newtheorem{corollary}[theorem]{Corollary} % Shares counter with 'theorem'
\newtheorem{assumption}{Assumption}
\newtheorem{remark}{Remark}

\title{Recommender regret for improving regularised RL estimation with bandits and MCTS}

\author{Danilo Vucetic\thanks{Equal contribution, correspondence to danilo.vucetic@mila.quebec and allar145@ulaval.ca}  $^{\text{  }1,3}$,
        Alexandre Larouche$^{* \text{ } 2,3}$,
        Mohammad Pedramfar$^{3, 4}$,\\
        \textbf{Audrey Durand}$^{2, 3, 5}$,
        \textbf{Gauthier Gidel}$^{1, 3, 5}$
        \\
        $^{1}$Université de Montréal \quad
        $^{2}$Université Laval \quad
        $^{3}$Mila - Québec AI Institute \quad
        $^{4}$McGill University \quad \\
        $^{5}$CIFAR AI Chair.
}

\iclrfinalcopy % Uncomment for camera-ready version, but NOT for submission.
\begin{document}

\maketitle

\begin{abstract}
Regularised reinforcement learning (RL) algorithms rely on estimators, like Monte Carlo tree search or one-step look-ahead, to provide policy, value, or Q-value estimates in training and inference.
However, existing bandit-based estimation frameworks evaluate sampling efficiency, not the quality of estimates.
We address this by formalising estimate generation as a recommendation process, treating bandit-based sampling as an improvement operator.
To evaluate the quality of a recommendation (estimate), we define regulariser-aware recommender regrets, analysing their lower bounds, and proving that simple algorithms like uniform sampling can achieve these bounds.
Finally, we unify bandit- and search-based target improvement under recommender MCTS.
Across synthetic trees, arcade games, and scientific discovery tasks, recommenders often yield faster target convergence and superior downstream performance. Our experiments characterise for the first time the potential of regularised MCTS in scientific applications. 
\end{abstract}

\section{Introduction}

% \todo{change ICLR2026 template to ICLR2027}

    Reinforcement learning (RL) algorithms often require policy, value, or Q-value estimators during training and inference. 
    % to act as targets in a loss function or during inference to provide improved action recommendations \citep{mnih2013_dqn, van_Hasselt_2016_DDQN, haarnoja2018soft, silver2017mastering_alphazero}.
    For example, deep Q-networks \citep[DQN]{mnih2013_dqn} and soft actor-critic \citep[SAC]{haarnoja2018soft} generate Q-value targets by computing the one-sample Monte Carlo estimates of the next state's value, %: $\hat Q(s, a) = r(s,a) + \gamma \hat V(s')$.
    while AlphaZero uses Monte Carlo tree search (MCTS) to generate improved value and policy estimates \citep{silver2017mastering_alphazero}.
    % Despite the central role these estimators play, the literature lacks a unified theoretical framework for optimizing how sampling budgets should be allocated to maximize the quality of the resulting target estimates.
    % Given the success of these methods, with their estimator-based targets, we study the problem of how to best generate these targets, working from the theoretical framework of multi-armed sampling \citep{pedramfar2026multiarmedsampling}.
    Despite the success of these estimator-based methods, the literature lacks a framework for optimally allocating sampling budgets to maximise the quality of target estimates.

    Using the theoretical framework of multi-armed sampling \citep{pedramfar2026multiarmedsampling}, which is based on entropy-regularised RL, we study how to efficiently estimate these targets for a more general class of regularisers.
    In multi-armed sampling, an algorithm sequentially chooses a sampling policy over an action space, samples an action from the sampling policy, and observes a reward. %Instead of minimizing the regret, the goal is to match some target distribution. Therefore, the observed rewards are used to update a \textit{performance function} that will be compared with the target policy.
    Using the information contained in the rewards, the goal is to efficiently match some target distribution, usually modelled as a softmax or Boltzmann distribution.
    Efficiency here refers to minimising the number of samples required to match the target distribution, thus motivating why prior formulations focused on evaluating the sampling policy against the target~\citep{wang2021fairness_fairx_ucb_ts, pedramfar2026multiarmedsampling}.
    % performance through \emph{sampling efficiency}: measuring how quickly the sampling policy or empirical sampling distribution converges to the target distribution \citep{pedramfar2026multiarmedsampling, jain2023gflownets_discovery}.
    However, in RL estimation, we care about the quality of the \emph{recommendation} of various target objects (\eg, policies, values, and Q-values) produced at the end of the sampling allocation, rather than the cost of the sampling policy itself.

    To bridge this gap, we extend the simple sampling regret of \citet{pedramfar2026multiarmedsampling} with the simple regret minimisation of \citet{bubeck2009pure_simpleregret}, to formalise target generation as an explicit \textbf{recommendation process}.
    % Here, the sampling policy serves merely as a data-gathering mechanism, while the recommendation represents the final estimate of a regularised RL target whose approximation error alone determines cost.
    % This distinction between the accumulated cost of samples in multi-armed sampling and the one-time cost of recommendations highlights a theoretical tension: recommendations require exploration to resolve uncertainty, while sampling discourages exploration of suboptimal actions. 
    % Inspired by multi-armed sampling~\citep{pedramfar2026multiarmedsampling}, which considers cumulative and simple regret~\citep{bubeck2009pure_simpleregret}, we extend simple sampling regret.
    Specifically, we 1) extend the simple sampling regret to allow for alternative regularisers, 2) expand the class of target objects to include values and Q-values, and 3) define divergences based on the choice regulariser and target object.
    This yields a regulariser-aware notion of regret we call the \textbf{recommender regret}.
    Our core contributions are as follows:
    \begin{itemize}[nosep,leftmargin=10pt]
        \item We introduce regulariser-aware notions of \textbf{recommender regret} for policies, values, and Q-values across general regularised RL settings, deriving theoretical lower bounds for each
        \item We prove upper bounds on recommender regret, demonstrating that uniform sampling achieves the tightest bounds of the algorithms analysed. This is a surprising result, given that multi-armed sampling requires almost no exploration \citep{pedramfar2026multiarmedsampling}
        \item We extend recommender regret to tree-search settings by proposing \textbf{recommender MCTS}, a generalisation of prior regularised tree-search work for planning-based target estimation
        \item We conduct extensive empirical evaluations on synthetic trees, arcade games, and scientific discovery domains, demonstrating that recommender-based algorithms often converge faster and achieve superior performance compared to prior work.
    \end{itemize}

\section{Background and related work}

    \subsection{Multi-armed sampling bandit problems}
        Multi-armed sampling bandits~\citep{pedramfar2026multiarmedsampling} are instances of stochastic multi-armed bandit problems with different notions of regret, and thus different goals for the policy.
        Let $\mathcal A$ denote a set of available actions, also called \textit{arms}. Stochastic multi-armed bandit instances can be represented as a collection of reward distributions $\phi = (\phi_a:a\in \mathcal{A})$, where $\phi_a$ is the reward distribution of action $a$, with mean $\mu_a = \expectwrt{r}{r \sim \phi_a} \in \mathbb{R}$, and $K = |\mathcal{A}|$ (\ie, there are $K$ actions) \citep{lattimore2020_bandit_book}.
        A sampler interacts with the environment over $T$ rounds sequentially, receiving feedback in the form of rewards which inform future interactions.  
        At each time step $t \in [T]$ the sampler chooses a sampling policy $d_t \in \Delta_\mathcal{A}$, and samples an action $a_t \sim d_t$, receiving a randomly sampled reward $r_{t,a_t} \sim \phi_{a_t}$ from the environment. 
        The history of interactions, $\mathcal{H}_t = (d_1, a_1, r_{1,a_1}, \dots , d_{t-1}, a_{t-1}, r_{t-1, a_{t-1}})$, is maintained by the sampler and used to choose the next sampling policy: $d_{t+1}(\cdot | \mathcal{H}_{t+1})$.\footnote{With the history containing all the randomness of the environment (due to random sampling of the rewards from their distributions) and the policy (due to random sampling of actions from the policy).} 
        
        In multi-armed sampling, the sampler maintains a performance distribution, $f_t = f(\mathcal{H}_t)$, which is a deterministic function of $\mathcal{H}_t$. Given a divergence $D$, the sampling regret at time $t$ is defined as: %The sampling regret is defined in \autoref{equ:back_sampling_regret}, for a divergence $D$ at time $t$, and can be extended to expected or cumulative regret \citep{pedramfar2026multiarmedsampling}.
        \begin{equation}
        \label{equ:back_sampling_regret}
            \mathcal{R}^{f, D}_t = D(f_t, \pi^\star) \quad \text{with~} \pi^\star = \text{softmax}(\mathbf{\mu}) \quad \text{and~} \mathbf{\mu} = (\mu_a: a \in \mathcal{A}),
        \end{equation}
        and can be extended to expected or cumulative regret \citep{pedramfar2026multiarmedsampling}. 
        % Note how the sampling regret is measured against the performance distribution, meaning that a sampler seeks to drive the performance distribution close to the optimal policy. Multi-armed sampling can thus be seen as an indirect estimation problem, where the directly observed objects (rewards) are used to update a sampling policy $d_t$, which affects the performance distribution $f_t$.
        By setting the performance distribution-generating function $f(\mathcal{H}_t)$, we can alter the properties of the sampling policy, $d_t$, that optimises the regret. This problem is similar to FairX bandits~\citep{wang2021fairness_fairx_ucb_ts}, where the considered policies are $\text{softmax}(\ln(g(\mu)))$, with $g: \mathbb{R}^K \to \mathbb{R}^K_{>0}$ being an arbitrary function, and the regret is defined specifically with $D$ as the total variation distance ($L_1$ norm).
        \citet{pedramfar2026multiarmedsampling} found that exploration was barely required for sampling.

    \subsection{Regularised Markov decision processes}
        Regularised Markov decision processes (MDPs) were introduced by \citet{geist2019theory_regularisers}, extending the usual MDP tuple with a strongly convex regularisation function.
        The regularised MDP tuple is $(\mathcal{S}, \mathcal{A}, r, P, \gamma, H, \Omega)$, where $\mathcal{S}$ is the state space, $\mathcal{A}$ is the action space, $r: \mathcal{S}\times\mathcal{A} \to \mathbb{R}$ is the reward function, $P: \mathcal{S}\times\mathcal{A} \to \Delta_\mathcal{S}$ is the environment's dynamics function, $\gamma \in [0,1)$ is the discount factor, $H$ is the horizon, and $\Omega: \Delta_\mathcal{A} \to \mathbb{R}$ is the regulariser.
        $\Delta_\mathcal{S}$ and $\Delta_\mathcal{A}$ denote the probability simplices over the state and action spaces, respectively.
        
        Regularisation with $\Omega$ penalises the policy for deviating from some desirable behaviour. This can be seen in the regularised policy evaluation functions $V^\pi_\Omega$ and $Q^\pi_\Omega$ over states and state-action pairs respectively, where the regulariser is subtracted from the value function:
        \vspace{-6pt}
        \begin{equation}
        \label{equ:regularised_evaluation_funcs}
            V^\pi_\Omega(s) = \expectwrt{Q^\pi_\Omega(s,a)}{a\sim \pi(\cdot|s)} - \frac{1}{\omega}\Omega(\pi) \quad \quad Q^\pi_\Omega(s,a) = r(s,a) + \gamma \expectwrt{V^\pi_\Omega(s')}{s' \sim p(\cdot|s,a)}.
        \end{equation}
        % This can be seen in the policy evaluation functions of \autoref{equ:regularised_evaluation_funcs}, where the regulariser is subtracted from the value function. 
        For instance, when $\Omega$ is negative entropy, the optimal policy will be ``more uniform,'' in that probability mass will be allocated throughout the action space rather than to a single optimal action \citep{neu2017unified, haarnoja2018soft, geist2019theory_regularisers, zhang2025when}.
        % \begin{equation}
        % \label{equ:regularised_evaluation_funcs}
        %     V^\pi_\Omega(s) = \expectwrt{Q^\pi_\Omega(s,a)}{a\sim \pi(\cdot|s)} - \frac{1}{\omega}\Omega(\pi) \quad \quad Q^\pi_\Omega(s,a) = r(s,a) + \gamma \expectwrt{V^\pi_\Omega(s')}{s' \sim p(\cdot|s,a)}.
        % \end{equation}
        Solving for the optimal policy by maximising $V^\pi_\Omega$ with respect to the policy $\pi$ yields the Legendre-Fenchel dual. The optimal regularised value and policy functions evaluate to the dual of the regulariser and its gradient, respectively: $V^\star_\Omega(s) = \Omega^*(Q^\star_\Omega(s, \cdot))$ and $\pi^\star_\Omega(\cdot | s) = \nabla \Omega^*(Q^\star_\Omega(s, \cdot))$ \citep{rockafellar1997_convex_analysis_conjugates, geist2019theory_regularisers}. \autoref{table:back_regularised_RL_equations} lists common regularisers and their optimality equations where $\omega$ and $\alpha$ are inverse temperatures, $\pi_\text{ref}$ is an arbitrary reference policy, $\pi^\alpha_Q = \text{softmax}(\alpha Q^\star_\Omega(s))$, and Tsallis entropy\footnote{We always use $q=2$ as the entropic index for Tsallis entropy, unless otherwise stated.} yields sparse policies via the sparsemax functions \citep{martins2016softmax_sparsemax, veillard_lev_avg_kl_regu, asadi2017alternative_mellowmax, gan2021stabilizing_soft_mellowmax, lee2018_sparse_MDP_tsallis, lee2020_tsallis_actor_critic_robotics}.        

        \begin{table}[b]
            \centering
            \caption{Optimality equations for various regularisers.}
            % Neat trick to get the table to fit into the text width of the page!
            \resizebox{\textwidth}{!}{\begin{tabular}{llll}\toprule 
               \textbf{Name} & \textbf{Regulariser $\Omega(\pi)$} & \textbf{Value $\Omega^*(Q^\star_\Omega(s))$} & \textbf{Policy $\nabla \Omega^*(Q^\star_\Omega(s))$} \\     \midrule \midrule 
                Negative Entropy          & $\sum_a \pi(a|s) \log \pi(a|s)$ & $\frac{1}{\omega}\text{LSE}\left( \omega Q^\star_\Omega(s) \right)$ & $\text{softmax}(\omega Q^\star_\Omega(s))$ \\ \midrule 
                % \begin{tabular}[c]{@{}l@{}}Relative\\  Entropy\end{tabular} 
                Relative Entropy
                & $\text{KL}(\pi(s) || \pi_\text{ref}(s))$ & $\frac{1}{\omega} \log\sum_a\pi_\text{ref}(a|s) e^{\omega Q^\star_\Omega(s,a)}$ & $\dfrac{\pi_\text{ref}(a|s)e^{\omega Q^\star_\Omega(s,a)}}{\sum_b\pi_\text{ref}(b|s)e^{\omega Q^\star_\Omega(s,b)}}$ \\ \midrule 
               % \begin{tabular}[c]{@{}l@{}}Soft\\  Mellowmax\end{tabular}  
                Soft Mellowmax
                & $\text{KL}(\pi(s) || \pi^\alpha_Q(s))$ & 
                $\frac{1}{\omega} \log\sum_a \pi^\alpha_Q(a|s) e^{\omega Q^\star_\Omega(s,a)}$
                % $\frac{1}{\omega} \log \dfrac{\sum_a  e^{(\alpha+\omega) Q^\star_\Omega(s,a)}}{\sum_b  e^{\alpha Q^\star_\Omega(s,b)}}$
                & $\text{softmax}((\alpha+\omega) Q^\star_\Omega(s))$ \\\midrule 
                % $\dfrac{e^{(\alpha+\omega) Q^\star_\Omega(s,a)}}{\sum_b  e^{(\alpha+\omega) Q^\star_\Omega(s,b)}}$ \\ \hline
                Tsallis Entropy & $\frac{1}{2}\sum_a \pi(a|s)^2$ & $\frac{1}{\omega}\text{sparsemax-val}\left(\omega Q^\star_\Omega(s)\right)$ & $\text{sparsemax}\left(\omega Q^\star_\Omega(s)\right)$ \\
                \bottomrule
            \end{tabular}}
            \label{table:back_regularised_RL_equations}
        \end{table}

    \subsection{Prior work in recommender regret} \label{sec:prior_work_rec_reg}
        \citet{xiao2019maximum_MENTS} introduced the stochastic softmax bandit problem which sought to minimise the exponentiated value regret in the entropy-regularised RL setting: $\mathcal{E}_t = \mathbb{E} \big[(\exp(\hat V_t) - \exp(V^\star))^2 \big]$. 
        % By analysing the statistical properties of the estimator $\exp(\hat V_t) = \sum_a \exp(\hat \mu_{t,a})$, it is shown that $\lim_{t \to \infty} t \mathcal{E}_t \geq C$, for any algorithm achieving $\mathcal{E}_t \in O(1/t)$. 
        % In Section~\ref{sec:rec_reg_lower_bounds}, we provide a lower bound with a similar $1/\sqrt{t}$ rate on the absolute value difference between values (instead of the squared difference of exponentiated values).
        Importantly, \citet{xiao2019maximum_MENTS} did not study the Q-value or policy regrets of the stochastic softmax bandit problem. 
        % The closest analyses are \citep[Theorem 5]{xiao2019maximum_MENTS} and \citep[Theorem 2]{convex_regular_tree_search} which demonstrate that action recommendations concentrate in probability to the highest Q-value action at the root of an MCTS search tree. 
        In contrast, our work provides lower and upper bounds on norm-based Q-value regrets as well as policy regrets with Bregman divergence and norm-based divergences.
        Thus, our setting is novel both in the notions of regret and the bounds on those regrets. 

        Generalising the algorithm of \citet{xiao2019maximum_MENTS}, \citet{convex_regular_tree_search} proposed extended empirical exponential weight (E3W), which selects actions according to the convex combination of the regularised policy and a uniform exploration term ($\epsilon > 0$):
        \begin{equation}
        \label{equ:E3W_action_selection}
            d^\text{E3W}_t(a|s) = (1 - \lambda_t) \nabla \Omega^* \left( \hat{Q}_t(s, \cdot) \right)(a) + \frac{\lambda_t}{|\mathcal{A}|} \quad \text{with~} \lambda_t = \frac{\epsilon}{\log(t+1)}.
        \end{equation}
        In Section~\ref{sec:rec_reg_algos_upper_bounds} we show that E3W is not optimal in the recommender regret, in that uniform action selection can achieve tighter upper bounds.
        
        % Generalising the algorithm of \citet{xiao2019maximum_MENTS}, \citet{convex_regular_tree_search} proposed extended empirical exponential weight (E3W, \cf, \autoref{equ:E3W_action_selection} for an exploration hyperparameter $\epsilon>0$), a stochastic algorithm which selects actions according to the convex combination of the regularised policy and a uniform exploration term. In Section~\ref{sec:rec_reg_algos_upper_bounds} we show that E3W is not optimal in the recommender regret, in that uniform action selection can achieve tighter upper bounds. 
        % \begin{equation}
        % \label{equ:E3W_action_selection}
        %     d^\text{E3W}_t(a|s) = (1 - \lambda_t) \nabla \Omega^* \left( \hat{Q}_t(s, \cdot) \right)(a) + \frac{\lambda_t}{|\mathcal{A}|}; \quad \lambda_t = \frac{\epsilon}{\log(t+1)}.
        % \end{equation}
        
        \citet{xiao2019maximum_MENTS} also introduced the maximum entropy tree search (MENTS) algorithm, later extended by \cite{convex_regular_tree_search} to strongly convex regularisers like relative and Tsallis entropy (RENTS and TENTS, respectively) and to the soft mellowmax (SMM) operator by \citet{vucetic_SMM_TS}. 

\section{Recommender regret of sampling bandits in regularised RL}
    We define recommenders in order to generate improved policy, value, and Q-value estimates for use in regularised RL training and inference.
    Recommender regret appends estimation problems to the sampling bandit framework. While sampling bandits update a general performance distribution $f_t = f(\mathcal{H}_t)$, recommenders specifically update estimates of the regularised RL optimality equations: 
    \begin{equation}
    \label{equ:reg_rl_optimality_bandit_equs}
        \hat{\pi}_{\Omega, t} = \nabla \Omega^*(\hat{Q}_{\Omega,t}); \quad \hat{V}_{\Omega, t} = \Omega^*(\hat{Q}_{\Omega,t}); \quad \hat{Q}_{\Omega,t} = \hat{r}_t,
    \end{equation}
    where $\Omega$ is a strongly convex regulariser, and $\hat{r}_t = (\sum_{i=1}^{t} r_{i, a_i} \mathbb{I}(a_i = a) / N_t(a) : a \in \mathcal{A})$. From sampling bandits, we naturally arrive at the recommender policy regret by setting $f_t = \text{softmax}(\hat{r}_t)$. However, this only covers one policy function (the softmax) out of many possible through regularised RL, and does not extend to other objects such as values and Q-values. Lastly, the divergences $D \in \{\text{KL}, \norm{\cdot}_1 \}$ of the sampling bandit policy regret are not aligned with arbitrary regularisers, whose primal norm and Bregman divergence may be different (let alone divergences for objects which are not policies) \citep{bregman_divergence_paper, shalev2007primal_dual_smooth_1}. This highlights the need for a new, \textbf{regulariser-aware notion of regret} such that the geometry of $D$ matches the regularised RL setting, along with a suitable analysis of lower and upper bounds for this regret.
    % In the following sections we define \textbf{regulariser-aware notions of regret} for sampling bandits, such that the geometry of $D$ matches the regularised RL setting. We then analyse the lower bounds on these regrets, and demonstrate that simple algorithms can achieve these bounds.
    
    % the sampling policy can be uniform (or any distribution with full support for that matter) and still lead to the convergence of $p_t \to \pi^\star$.

    % Efficiently estimating policies, values, or Q-values is the central goal of a recommender agent. 

    \subsection{Notions of recommender regret}
        The goal of a recommender is to match a target object $f^\star \in \{\pi^\star_\Omega, V^\star_\Omega, Q^\star_\Omega \}$. We formalise this idea using the notation of \citet{pedramfar2026multiarmedsampling}, choosing performance functions $f_t \in \{\hat{\pi}_{\Omega, t}, \hat{V}_{\Omega, t}, \hat{Q}_{\Omega, t} \}$, which are deterministic functions of the history $\mathcal{H}_t$. A regulariser-aware divergence function $D$ is also required, with its set of admissible functions changing depending on the object of interest. The simple recommender regret on an object of interest $f$ with respect to a divergence $D$ at time step$t \in [1, T]$, is: 
        \begin{equation}
        \label{equ:SRR_f_d_definition}
            \mathcal{SRR}^{f, D}_t(\mathcal{H}_t) = D(f_t, f^\star),
        \end{equation}
        where $f^\star$ is the target of the object $f \in \{ \pi, V, Q \}$. Algorithms that attempt to minimise the simple recommender regret must make $f_t$ close to $f^\star$ with respect to the divergence $D$. We use the notation $\mathcal{SRR}^{f, D}_t$ when there is no ambiguity. The simple recommender regret is a random variable of the policy $d_t$ and the environment $\phi = (\phi_a : a \in \mathcal{A})$ since $f_t$ is updated according to the sampled rewards. The expected simple recommender regret is defined as:
        \begin{equation}
        \label{equ:expected_SRR_f_d_definition}
            \expect{\mathcal{SRR}^{f, D}_t} = \expectwrt{D(f_t, f^\star)}{\substack{a_i \sim d_i \\ r_{i, a_i} \sim \phi_{a_i} \\ \forall i \in [1, t]}}.
        \end{equation}

        % Minimising the error of a policy, value, or Q-value recommender (estimator) leads to an optimisation-like notion of regret. Stating this in the notation of \citet{pedramfar2026multiarmedsampling}, 

        \paragraph{Policies}
        When the object of interest is $f = \pi$, then $D(\pi_1,\pi_2) \in \{\norm{\pi_1-\pi_2}, B_\Omega(\pi_1,\pi_2) \}$, where $\norm{\cdot}$ is the primal norm for which $\Omega$ is strongly convex, and $B_\Omega$ is the Bregman divergence under $\Omega$: $B_\Omega(\pi_1, \pi_2) = \Omega(\pi_1) - \Omega(\pi_2) - \inner{\nabla \Omega(\pi_2), \pi_1-\pi_2}$. For example, negative entropy is strongly convex with respect to the $L_1$ norm, while the Bregman divergence is equal to the KL divergence. 

        \paragraph{Values}
        Value functions output scalars, so any non-negative scalar error function would be appropriate. Let $f=V$ and choose $D(V_1, V_2) = |V_1 - V_2|$, the absolute value function, as the distance for value-based simple recommender regret.

        \paragraph{Q-values} For Q-value-based recommender regret, we define $D(Q_1, Q_2) = \norm{Q_1-Q_2}_*$, which is the dual norm of the policy's primal norm. 
        % The dual norm is a natural choice of divergence since it... lmfao I actually don't know why I chose the dual norm hahaha
        %This primal-dual relationship arises due to the Legendre-Fenchel dual representation of regularised RL, and the dual-smoothness lemma. 
        % That is, policies are primal variables, while Q-values are dual variables, with mappings between the spaces represented by $Q = \nabla \Omega(\pi)$ and $\pi = \nabla \Omega^*(Q)$ \citep{rockafellar1997_convex_analysis_conjugates, shalev_shwartz_phd_thesis_online_learning_convexity}.

    \subsection{Lower bounds on the recommender regret} \label{sec:rec_reg_lower_bounds}
        % Combine all the notions of regret into a single theorem, showing all the bounds at once, somehow. Alexandre thinks there's a decent way to present it.
        We first establish minimax lower bounds for the expected recommender regret under norm-based statistical distances when the objects of study are values and Q-values. The results are presented in Theorem~\ref{thm:lb_norm_q_v} and hold for any arbitrary regulariser.
        \begin{theorem}[Lower bound for state and action values recommender regret]
        \label{thm:lb_norm_q_v}
        Let $T$ and $K$ be the horizon and the number of actions respectively, $\Omega$ be a regulariser, and $\mathcal{SG}(\sigma^2)$ denote the class of $\sigma$-sub-Gaussian bandits. Then, for any algorithm, we have
        \begin{equation*}
            \begin{aligned}
                \sup_{\phi \in \mathcal{SG}(\sigma^2)} \expect{\left|\hat V_{\Omega, T} - V^{*}_{\Omega}\right|} &\geq \sqrt{\frac{2}{T}}\frac{\sigma}{8e} \\
                \max_{\phi \in \mathcal{SG}(\sigma^2)} \E_\phi\left[\norm{\murec - Q^*_{\Omega, \phi}}_*\right] &\geq \sqrt{\frac{K}{2T}}\frac{\sigma \min_{a,a' \in [K], a\neq a'}\norm{\mathbf{e}_a - \mathbf{e}_{a'}}_*}{8e},
            \end{aligned}
        \end{equation*}
        where $\mathbf{e}_a$ is a $K$-dimensional $a$-th indexed canonical basis vector.
        \end{theorem}

        The result is obtained by creating a base instance and a perturbed instance, averaging their respective expected recommender regret as a lower bound, and applying Bretagnolle-Huber. %To obtain the result, we follow typical bandit proof structure. We create a base instance and a perturbed instance, take the mean of the expected recommender regret on both, apply Bretagnolle-Huber. 
        We leverage the triangle inequality to get a quantity that depends on the perturbation, which highlights the $\sqrt{T^{-1}}$ rate. Details are deferred to Appendix~\ref{app:lb_proofs}. Similarly to simple regret minimisation~\citep[Remark 2]{bubeck2009pure_simpleregret}, the Q-value recommender regret scales with $\sqrt{K/T}$, i.e. larger horizons make the problem easier, but more actions make it harder. Interestingly, the value lower bound does not depend on $K$, i.e., regardless of how many actions are in the problem, the difficulty is constant.
        
        For policy recommender regret, we note that the convexity of $\Omega^*$ by itself does not provide us with meaningful lower bounds because there may exist multiple perturbations that provide the same policy, thus trivially setting the regret to zero.  We provide a sufficient condition in Proposition~\ref{prop:suf_cond_lb_pol}. Before stating the propostion, let $E:= \{v \in \mathbb{R}^A \;|\; \sum_{a\in \mathcal{A}}v_a = 0 \}$ be the vector space parallel to the simplex and let us define the dual norm restricted to $E$ as $\|\cdot\|_{\star,E} := \sup_{v \in E\, \|v\|\leq 1} \langle \cdot,v\rangle$. We will also need to make the following assumption
        \begin{assumption} \label{ass:smooth}
        Given $\mathcal{Q} \subset \mathbb{R}^\mathcal{A}$ the set of targets $Q^\star_\Omega$, we assume that $\Pi := \nabla \Omega^*(\mathcal{Q})$ is a subset of $\text{ri}(\Delta_\mathcal{A})$,and that for all $\pi \in \text{hull}(\Pi)$ and $v \in E$, we have $\langle\nabla^2_E \Omega(\pi)v,v\rangle \leq L \|v\|^2 $.\footnote{The notation $\text{ri}(C)$ and $\text{hull}(C)$ respectively stand for the relative interior and the convex hull of $C$, for and $\nabla_E$ correspond to the gradient restricted to the vector space $E$.}
        \end{assumption}
        This assumption holds for instance for the entropy function for which, given that $\mathcal{Q}$ is compact we have that $\nabla \Omega^*(\mathcal{Q})$ is in the relative interior of the simplex and  $\Omega^*$ has the smoothness property described above. 
        \begin{proposition}
        \label{prop:suf_cond_lb_pol} Under Ass.~\ref{ass:smooth},
            % If $\Omega$ is a uniformly smooth regulariser, then there exists a function $\nu_D: \mathbb{R}_{\geq0} \to \mathbb{R}_{\geq0}$, where $\nu_D(c) = 0 \iff c = 0$, such that $\norm{\hat \pi_{\Omega, t} - \pi^*_\Omega} \geq \nu_{\norm{\cdot}}\left(\norm{\murec - Q^*_\Omega}_*\right)$ and $\expect{B_\Omega(\hat \pi_{\Omega, t}, \pi^*_\Omega)} \geq \nu_B\left(\norm{\murec - Q^*_\Omega}_*\right)$. 
              $\forall Q,Q' \in \mathcal{Q}$ we have $\langle\nabla \Omega^*(Q)-\nabla \Omega^*(Q'), Q- Q'\rangle \geq \frac{1}{L} \|Q -Q'\|_{*,E}^2$.
        \end{proposition}
        The proof is in Appendix~\ref{app:suff_cond_lb_pol} as well as a discussion on why strict convexity of $\Omega^*$ is not enough. It relies on a simplex-contrained version of the fact that this property of the regulariser $\Omega$ implies the dual $\Omega^*$ is $\frac{1}{L}$-strongly convex, which we use to lower-bound the expected recommender regret. As a concrete example, it follows from Proposition~\ref{prop:suf_cond_lb_pol} that negative entropy is $L$-smooth on the policy space $\Pi = \left\{ \pi: \pi(a) \geq \epsilon>0, \forall a \in \mathcal{A}\right\}$. Note that Assumption~\ref{ass:smooth} is made solely to obtain an interesting analysis of a broader class of regularisers on a subset of their domain. Indeed, compact sets of policies may be overly restrictive for some regularisers. We obtain the following lower bound on the policy recommender regret.
        % The specific form of the modulus of convexity is known for special cases of uniformly smooth regularisers.
        % A simple corollary of this result gives us a lower bound of the policy with the Q-values.
        %         \begin{corollary} For $\hat \pi_{\Omega,T} :=  \nabla \Omega^*(\hat Q_{\Omega,T})$ and $\pi^*_\phi := \nabla \Omega^*(Q^*_{\Omega,\phi})$ we have $$\|\hat \pi_{\Omega,T} - \hat \pi^*_\phi\| \geq \tfrac{1}{L} \|\hat Q_{\Omega,T}- Q^*_{\Omega,\phi}\|_{*,E}\,.$$
        % \end{corollary}
        % With this corollary, we can provide minimax lower bounds for the norm-based and Bregman divergence-based policy recommender regret.
        \begin{theorem}[Lower bound for policy recommender regret]
        \label{thm:lower_bound_smooth}
            Let $T$, $K$ and $\omega =1$ be the horizon, the number of actions and the inverse temperature respectively and let us consider the class of $\sigma$-sub-Gaussian bandits, $\mathcal{SG}(\sigma^2)$ with target $Q^\star_\Omega \in \mathcal{Q}$, we call this set $\mathcal{SG}(\sigma^2, \mathcal{G})$. Then if,  Ass~\ref{ass:smooth} holds we have 
            \begin{equation*}
                \begin{aligned}
                    \max_{\phi \in \mathcal{SG}(\sigma^2,\mathcal{G})} \E_\phi\left[\norm{\pirec- \pistar_\phi }\right] &\geq \sqrt{\frac{K}{2T}}\frac{\sigma \min_{a,a' \in [K], a\neq a'}\norm{\mathbf{e}_a - \mathbf{e}_{a'}}_{*,E}}{8L} \\
                     \max_{\phi \in \mathcal{SG}(\sigma^2,\mathcal{G})} \E_\phi\left[B_\Omega(\pirec, \pistar_\phi)\right] &\geq \frac{K}{T}\frac{\beta \sigma^2  \min_{a,a' \in [K], a\neq a'}\norm{\mathbf{e}_a- \mathbf{e}_{a'}}_{*,E}^2}{64L^2e},
                \end{aligned}
            \end{equation*}
            where $\mathbf{e}_a$ is a $K$-dimensional $a$-th indexed canonical basis vector.
    \end{theorem}

    The proof can be found in Appendix~\ref{app:lb_proofs}, and reuses the same proof structure from Theorem~\ref{thm:lb_norm_q_v}, where the Bregman divergence is lower bounded using the $\beta$-strong convexity of $\Omega$.
    Once again, these lower bounds indicate that larger horizons make the problem easier, whereas a larger number of actions and smaller smoothness constants $L$ make the problem harder. Indeed, when $L$ is small, vastly different policies can have quasi-identical Q-values. However, this result provides evidence that the multi-armed sampling lower bound of~\citet{pedramfar2026multiarmedsampling} on the simple sampling regret using the total variation distance may be loose, since our lower bound is greater by a factor of $K$.

    \subsection{Algorithms for the recommender regret} \label{sec:rec_reg_algos_upper_bounds}
        % E3W and uniform analysed, showing that uniform seems to be better.
        We begin with an analysis of the notions of regret, assuming no knowledge of the form of the regulariser, other than its being strongly convex. Theorem~\ref{theorem:general_upper_bound} provides general upper bounds on each notion of simple recommender regret for any sampling policy $d_t$ with non-zero probability mass on each action. 
        Importantly, the upper bounds always achieve near-optimal dependence of on $T$, but only achieve optimal dependence on $K$ when the dual norm is $\norm{\cdot}_\infty$, or when $f = V$. This theorem allows us to conclude that uniform sampling is an effective algorithm for regularised RL estimation. % problems. 

        \begin{theorem}[Sampling policies with non-zero probability mass achieve optimal regret bounds] \label{theorem:general_upper_bound}
            Suppose $\Omega$ is $\beta$-strongly convex with dual $V_\Omega(s) = \Omega^*(Q(s, \cdot))$, with $\sigma$-sub-Gaussian reward distributions $\phi_a, \forall a \in \mathcal{A}$. For any $T \geq \frac{2\ln(2K/\delta)}{\zeta^2}$, with $\delta \in O(1/T^2)$, and any algorithm with sampling policy $d_t$ satisfying $d_t(a) \geq \zeta > 0$ $\forall a \in \mathcal{A}, \forall t \in [1, T]$, the following expected regret bounds hold:
            % \begin{equation} \label{equ:theorem_regret_bounds}
            %     \expect{\mathcal{SRR}^{f, D}_T} \in \Tilde{O} \left( \sqrt{\frac{1}{\zeta T}}\right);\quad \quad \expect{\mathcal{SRR}^{\pi, B_\Omega}_T} \in \Tilde{O} \left(\frac{1}{\zeta T}\right),
            % \end{equation}

            \begin{enumerate}[leftmargin=*, nosep, before=\setlength{\abovedisplayskip}{1pt}\setlength{\belowdisplayskip}{1pt}]
                \item When $\norm{\cdot}_\star = \norm{\cdot}_\infty$ or when $f = V$:
                \begin{equation} \label{equ:theorem_regret_bounds_infty}
                    \expect{\mathcal{SRR}^{f, D}_T} \in \tilde{O} \left( \sqrt{\frac{1}{\zeta T}}\right), \qquad \expect{\mathcal{SRR}^{\pi, B_\Omega}_T} \in \tilde{O} \left(\frac{1}{\zeta T}\right).
                \end{equation}
                \item When $\norm{\cdot}_\star$ is a general dual norm with $\norm{ \mathbf e'}_\star = \max_{a} \norm{\mathbf{e}_a}_\star$:
                \begin{equation} \label{equ:theorem_regret_bounds_general}
                    \expect{\mathcal{SRR}^{f, D}_T} \in \tilde{O} \left( K \norm{\mathbf e'}_\star \sqrt{\frac{1}{\zeta T}}\right), \qquad \expect{\mathcal{SRR}^{\pi, B_\Omega}_T} \in \tilde{O} \left( K^2 \norm{\mathbf e'}_\star^2 \frac{1}{\zeta T}\right),
                \end{equation}
            \end{enumerate}
            where the left and right terms bound norm-based and Bregman divergence-based regrets, respectively. $\Tilde O$ leaves out slow-growing logarithmic terms, following  \citet{pedramfar2026multiarmedsampling}.
        \end{theorem}
    
        % \begin{proof}[Proof of Theorem~\ref{theorem:general_upper_bound}]
        %     The proof is completed by upper bounding each notion of regret by Q-value error using Lemma~\ref{lemma:q_val_convergence_implies_value_policy}. Lemma~\ref{lemma:any_policy_achieves_optimal_bounds} then allows us to upper bound the Q-value error using concentration arguments on specific problem instances. Finally, Lemma~\ref{lemma:upper_bound_conversion_to_expectation} allows for bounds in expectation on each notion of regret.
        % \end{proof}

        Theorem~\ref{theorem:general_upper_bound}, proof in \autoref{appendix:proof_bandit_theorem}, makes it clear that an optimal sampling policy should maximise $\zeta$, to minimise the upper bound. Intuitively, \textbf{uniform sampling} should achieve better rates than other sampling policies, like E3W (\Eqref{equ:E3W_action_selection}). 
        Corollary~\ref{cor:uniform_beats_e3w_Q_upper_bounds} makes this intuition concrete, demonstrating that E3W achieves upper bounds that are looser by a $\log(T)$ term.
        Our results match those of \citet{pedramfar2026multiarmedsampling} in terms of dependence on $T$ (they don't report dependence on K) and highlight a theoretical tension between sampling and recommendation: while sampling requires minimal exploration, recommendation requires pure exploration, as in \citet{bubeck2009pure_simpleregret}.
        % Our results highlight a theoretical tension between sampling and recommendation: as in \citet{bubeck2009pure_simpleregret}, we show that uniform sampling is optimal for simple regret minimisation, and unlike \citet{pedramfar2026multiarmedsampling}, 
        %This is similar to the finding in \citet{bubeck2009pure_simpleregret}, which showed that uniform sampling tightened simple regret upper bounds, but differs from \citet{pedramfar2026multiarmedsampling}, which showed that sampling barely required exploration. Our results thus demonstrate a theoretical tension between sampling and recommendation.
        \vspace{-7pt}
        \paragraph{Limitations of the analysis}
            We sought to produce general upper bounds of the simple recommender regret that would hold regardless of the regulariser, and consequently, the form of the policy or value functions. As such, though our upper bounds fairly tight in $T$, they are still (sometimes) loose in $K$. Future work will propose regulariser-specific algorithms (\eg, like \citet{pedramfar2026multiarmedsampling} for entropy-regularised RL in sampling regret) to achieve tighter upper bounds.
            % Note that uniform sampling is optimal in the sense that it minimises the upper bound, as we will see in \autoref{sec:sythetic_trees}, uniform sampling does not always achieve lower error than E3W.
\vspace{-15pt}
\section{Monte Carlo planning for recommendations}
\vspace{-5pt}
    Monte Carlo tree search is a planning algorithm that naturally extends recommender bandit problems to generating recommendations in multi-state RL environments. Following the algorithmic setup of \citet{xiao2019maximum_MENTS} and \citet{convex_regular_tree_search}, we consider a search tree $\mathcal{T}$ rooted at $s_0$ of depth $H$ with deterministic transition function $\tau: \mathcal{S}\times\mathcal{A} \to \mathcal{S}$ and deterministic non-leaf rewards $r(s, a)$. 
    Each non-leaf node $s \in \mathcal{T}$ acts as a bandit instance operating over actions $a \in \mathcal{A}$.\footnote{MCTS can be intuitively viewed as a sequence of bandits stacked one on top of another.}
    Each node $s$ of the tree maintains value estimates $\hat{V}_t(s)$, $\hat Q_t(s,a)$, and visit counts $N_t(s,a)$ across total simulation steps $t$.

    At each simulation step $t$, a sampling policy $d_t(\cdot|s)$ navigates from $s_0$ through $\mathcal{T}$ along trajectory $(s_0, a_0) \to (s_1, a_1) \to \dots \to s_L$ of length $L \leq H$, with $a_i \sim d_t(\cdot|s_i)$ and $s_{i+1} = \tau(s_i, a_i)$, adding $s_L$ to the tree. 
    Visit counts are updated recursively via $N_t(s_i, a_i) \leftarrow N_t(s_i, a_i) + 1$ along the trajectory. 
    While standard MCTS algorithms like UCT \citep{kocsis2006bandit_UCT} and AlphaZero \citep{silver2017mastering_alphazero} use empirical average backups, regularised tree search algorithms use the Legendre-Fenchel dual $\Omega^*(\hat{Q})$ directly, without averaging:
    \begin{equation} \label{equ:mcts_value_q_value_updates}
        \hat{V}_t(s_i) = \Omega^*\left(\hat{Q}_t(s_i, \cdot)\right); \quad \hat{Q}_t(s_i,a_i) = \begin{cases}
            r(s_i, a_i) + \gamma R(s_L) & i = L - 1 \\
            r(s_i, a_i) + \gamma \hat{V}_t(\tau(s_i, a_i)) & i < L - 1
        \end{cases},
    \end{equation}
    where $R(s_L)$ is a value estimate generated via a Monte Carlo rollout, or a learned value estimator \citep{silver2017mastering_alphazero, xiao2019maximum_MENTS, convex_regular_tree_search, boltzmann_CRTS_simple_regret}. 
    
    % since our main contribution is the theoretical analysis of regularised tree search under simple recommender regret. 
    
    % Monte Carlo tree search (MCTS) incrementally constructs a search tree $\mathcal{T}$ rooted at an initial state $s_0$ to generate value, Q-value, and policy recommendations.

    % Consider a search tree $\mathcal{T}$ of depth $H$ with deterministic transition function $\tau: \mathcal{S}\times\mathcal{A} \to \mathcal{S}$ and deterministic non-leaf rewards $r(s, a)$.
    % Each non-leaf node $s \in \mathcal{T}$ acts as a bandit instance operating over actions $a \in \mathcal{A}$.\footnote{A nice perspective: MCTS is just stacking bandits, one on top of another.} 

    % Search trees are usually explained through a four-step process: 1) recursive action selection, 2) tree expansion to store new nodes, 3) value estimation of the expanded node(s), and 4) propagation of value estimates up the tree \citep{Mandziuk2018_MCTS_real_life}. This view is effective at explaining practical MCTS algorithms, but lacks explanatory power from a theoretical perspective. 

    % A search tree can instead be viewed as a stack of bandit problems. At each non-terminal and non-root node $s$ of a search tree, a bandit problem unfolds: 1) an action is sampled from algorithm $a \sim d(\cdot|s)$, 
    % %balances exploration and exploitation according to the notion of regret being optimised,
    % 2) reward signals are attained by visiting the child state-action pair, 3) the bandit's estimators are updated, and passed up the tree.

    Our main contribution is the theoretical analysis of regularised tree search \citep{xiao2019maximum_MENTS, convex_regular_tree_search} for generating recommendations.
    Theorem~\ref{thrm:MCTS_SRR_high_prob_bound} shows that the simple recommender regret concentrates at the root of the tree, indicating that recommender MCTS can effectively generate regularised policy, value, and Q-value recommendations (proof in \autoref{appendix:MCTS_THEOREM_PROOF}).

    \begin{theorem} \label{thrm:MCTS_SRR_high_prob_bound}
        For any notion of simple recommender regret $\mathcal{SRR}^{f, D}_T$ at the root node $s_0$ of a search tree of depth $H$ with deterministic transitions and deterministic non-leaf rewards, a sampling policy $d_t(a|s) \geq \zeta > 0, \forall a \in \mathcal{A}, \forall s\in\mathcal{S}, \forall t \in [1, T]$, and $\epsilon \leq C_3 \zeta$, then $\exists C_1, C_2, C_3 \in \mathbb{R}_{>0}$, such that
        \begin{equation} \label{equ:MCTS_high_prob_regret_bound}
            \prob{\mathcal{SRR}^{f, D}_T \geq \epsilon} \leq C_1 (KT)^H \exp \left(- \frac{\zeta^H T \epsilon^2}{K^{2H} C_2} \right).
        \end{equation}
    \end{theorem}

    Corollary~\ref{cor:MCTS_expected_regret} integrates these concentration inequalities to achieve the first expected regret bounds under regularised MCTS (proof in \autoref{appendix:proof_of_mcts_Expected_regret}).\footnote{Prior work could not generate such bounds since they specified restrictive errors $\epsilon$.}

    % \begin{corollary}[Recommender MCTS optimal regret bound] \label{cor:MCTS_uniform_regret}
    %     Under the hypotheses of Theorem~\ref{thrm:MCTS_SRR_high_prob_bound}, the uniform sampling policy $d_t(a|s) = 1/K$ maximizes $\zeta^* = 1/K$. This choice minimizes both $C_2$ and the search horizon $T$. In particular, for error probability $\delta \in (0, \min(C_1 (TK)^H,1))$ and search horizon $T \geq \frac{2C_2 K^2}{C_3^2} \left( 2H \log \left( \frac{2C_2 K^2 H}{C_3^2}  \right) + \ln \left( \frac{C_1 K^H}{\delta} \right) \right)$, uniform sampling guarantees that with probability at least $1-\delta$:
    %     \begin{equation} \label{equ:MCTS_optimal_regret_bound}
    %         \mathcal{SRR}^{f, D}_T \leq \sqrt{\frac{C_2}{T} \ln \left( \frac{C_1 (TK)^H}{\delta} \right)}.
    %     \end{equation}
    % \end{corollary}

    \begin{corollary}[Recommender MCTS expected regret bound] \label{cor:MCTS_expected_regret}
        % Let the samples from reward distributions be bound by $|r_{t, a_t}| \leq R_{\max} < \infty$. 
        Let the Q-value error be bounded: $\norm{\hat Q_T(s_0, \cdot) - Q^\star(s_0, \cdot)}_\infty \leq 2R_{\max} < \infty$.
        Under the hypotheses of Theorem~\ref{thrm:MCTS_SRR_high_prob_bound}, the uniform sampling policy $d_t(a|s) = 1/K$ maximizes $\zeta^\star= 1/K$. For $\delta \in (0,1]$ this choice minimizes the search horizon $T \geq \frac{C_2 K^{2+2H}}{\zeta^H C_3^2} \ln \left( \frac{C_1 (TK)^H}{\delta} \right)$, for which the following hold in expectation for MCTS:
        \begin{equation} \label{equ:MCTS_expected_regret_bound}
            \expect{\mathcal{SRR}^{f, D}_T} \in \Tilde{O} \left( \sqrt{\frac{K^{3H}}{ T}}\right);\quad \quad \expect{\mathcal{SRR}^{\pi, B_\Omega}_T} \in \Tilde{O} \left(\frac{K^{3H}}{T}\right),
        \end{equation}
        where the terms bound norm-based and Bregman divergence-based regrets, respectively.
    \end{corollary}

    While algorithms like E3W also achieve exponential convergence in probability, they introduce  logarithmic penalties in the exponent, as noted in Corollary~\ref{cor:uniform_beats_e3w_Q_upper_bounds} \citep{xiao2019maximum_MENTS, convex_regular_tree_search}. 
    Uniform sampling thus achieves tighter MCTS rates in both concentration (Theorem~\ref{thrm:MCTS_SRR_high_prob_bound}) and expected regret 
    (Corollary~\ref{cor:MCTS_expected_regret}), again following the results of \citet{bubeck2009pure_simpleregret} for simple regret minimisation. Remark~\ref{remark:mcts_bounds_kinda_suck} comments on the different rates achieved in prior work and our bound.
    These results make it clear that uniform sampling is optimal for recommender MCTS, with sample complexity scaling exponentially with tree depth, and polynomially with action space size.

\section{Experiments}
    We evaluate recommender bandits and MCTS algorithms in three domains to showcase their convergence efficiency, and utility in diverse decision-making contexts. We first study E3W and uniform sampling in synthetic trees to empirically verify our theoretical claims (\autoref{sec:sythetic_trees}). We then demonstrate the potential of recommenders for finding high-value candidate solutions for arcade games in OCTAX (\autoref{sec:octax_experiments}) and for biological sequence-design tasks (\autoref{sec:sequax_experiments}). 
    % In particular, we compare the performance of uniform sampling and E3W to verify our theoretical claims.
    Note that we are the first to employ regularised MCTS in either OCTAX or scientific discovery tasks, demonstrating for the first time the efficacy of regularised planning with recommenders in these domains.
    In all domains, the tree search algorithms described in \autoref{sec:prior_work_rec_reg} are employed, including 
    the max-entropy (MENTS), relative-entropy/KL (RENTS), Tsallis-entropy (TENTS), and soft-mellowmax (SMM) backup operators, alongside AlphaZero's variant of UCT.
    % Our tree search algorithms are implemented in MCTX, a JAX-based MCTS library \citep{deepmind2020jax}.
    % We evaluate the performance of uniform sampling versus E3W, to experimentally confirm the results of Corollaries \ref{cor:uniform_beats_e3w_Q_upper_bounds} and \ref{cor:MCTS_uniform_regret}.
    % In the latter two domains, performance is compared against a learned policy. Search is applied at inference time on top of the learned policy's value estimator (\cf, $R$ in \autoref{equ:mcts_value_q_value_updates}), and we sweep the number of simulations, the inverse-temperature $\omega$, and the tree-policy action-selection rule (E3W vs. uniform) to characterise when soft tree search improves on the base policy and when different regularisers are preferable.

\subsection{Synthetic tree experiments} \label{sec:sythetic_trees}
    \citet{xiao2019maximum_MENTS} proposed a synthetic tree environment as a testbed for regularised MCTS convergence. We follow their experimental setup, analysing the Q-value error since by Lemma~\ref{lemma:q_val_convergence_implies_value_policy}, Q-value error upper bounds all forms of simple recommender regret. Each tree is parameterised by a branching factor $k$ and depth $d$. At all nodes of the tree, the reward is drawn from a normal distribution upon each visit i: $r_i(s, a) \sim \mathcal{N}(\mu_{s,a}, 1)$. For each tree, we randomly select $\mu_{s,a} \sim \mathcal{N}(0, 1)$ for each action at each node, which is held constant for the whole simulation budget. 
    Regret is averaged over five such $\mu$ seeds. 
    % This ensures a high degree of stochasticity in the Q-value, value, and policy estimators, making convergence significantly harder.
    Q-value error is measured in accordance with each regulariser's dual norm: MENTS, RENTS, and SMM use the $L_\infty$ norm, while TENTS and UCT use the $L_2$ norm.
    
    \autoref{fig:synthetic_tree_main_fig} illustrates the differences in convergence speed between the MCTS methods. \textbf{Uniform sampling outperforms E3W across the presented trees}. This trend continues to hold for Bregman divergence-based regrets (\autoref{fig:appendix_policy_bregman_theoretical}) across all trees, but not for value regret (\autoref{fig:appendix_value_error_mse}), nor norm-based policy regrets (\autoref{fig:appendix_policy_error_theoretical}), where uniform sampling does not always achieve the lowest error curve for each MCTS algorithm, see \autoref{appendix:synthetic_tree_plots}.
    % However, this trend ceases to hold for value and policy errors, see \autoref{appendix:synthetic_tree_plots}, where E3W achieves smaller errors in smaller action spaces. 
    This discrepancy is interesting, and reveals a gap in our theoretical analysis: the algorithms' dependence on $K$ is not perfectly modelled, and the bounds on value and policy errors are somewhat loose.
    The higher error of TENTS an UCT in the figure is due to the difference in norms ($L_2$ vs. $L_\infty$) between the regularised methods (\cf, \autoref{fig:appendix_Q_vals_MSE}).
    %In \autoref{fig:appendix_Q_vals_MSE} we plot the Q-value mean squared error, which shows similar performance between TENTS and the other regularised methods.
    Finally, we note that our results differ from \citet{bubeck2009pure_simpleregret}, where UCB outperforms uniform sampling, despite the latter being theoretically optimal. In our experiments, uniform sampling often outperforms UCT and greedier sampling (E3W) in many tree and algorithmic settings.
    
    \begin{figure}[ht]
        \centering
        \includegraphics[width=\linewidth]{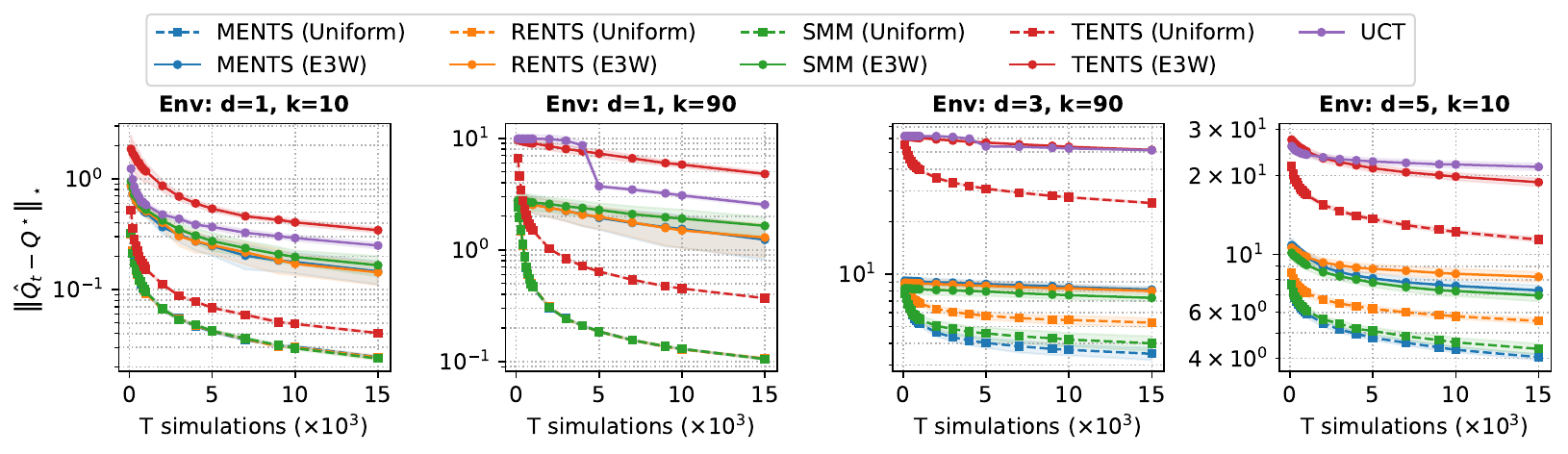}
        \caption{Q-value estimation error in recommender MCTS over 15k simulations and trees of varying depth $d$, and action space size $k$, with 100 random tree search seeds, for each of the five $\mu$ seeds.}
        \label{fig:synthetic_tree_main_fig}
    \end{figure}
    
\subsection{Arcade games in OCTAX} \label{sec:octax_experiments}
    % Done. Just need to clean them up.
    % At each step of a game, recommender MCTS was run, generating a policy recommendation. We took the argmax of the policy, to ensure good performance of the agent.  
    
    % Experiments in OCTAX demonstrate that recommender MCTS is able to generate effective policy recommendations in complex decision-making settings, like arcade games. 
    
    % For more plots, including execution times, episode lengths, and full task lists, see \autoref{appendix:octax_plots}.

    OCTAX, \citep{radji2026octax}, is a massively parallelized arcade-game environment that emulates a suite of classic arcade-style games (like Atari). 
    For each of 19 OCTAX games, we train a separate Q-network with a parallelised Q-network (PQN)-style algorithm \citep{gallici2025simplifying_pqn}. 
    The frozen Q-network is used to estimate values inside regularised MCTS (\cf, \autoref{equ:mcts_value_q_value_updates}, $R \leftarrow \Omega^*(Q_\theta(s', \cdot))$).
    We compare the learned policy and the MCTS-improved policies, greedily selecting the highest Q-value action for each.
    The inverse temperature parameter $\omega = 5$ for all regularised tree search methods, and $\alpha = 1$ for SMM. 
    The E3W exploration constant is fixed at $\epsilon=0.1$. We sweep over simulation budgets to evaluate performance as simulations increase. 
    Each configuration is evaluated over 100 random seeds.
    
    In \autoref{fig:octax_MCTS_long_sims_improvement}, we report mean episodic return for a representative subset of the games.
    This subset was chosen by filtering out the games that saturated quickly (to study performance with $T$), games that didn't improve with search (\eg, search is worse in Tank in \autoref{fig:appendix_octax_episode_returns}), and choosing a wide range of episode lengths (100 for Brix, 4000 for Filter, \cf, \autoref{fig:appendix_octax_episode_lengths}).
    Regularised MCTS clearly improves the performance of the PQN agent (baseline) in these games, especially as the number of simulations increases. Interestingly, uniform sampling and E3W perform similarly at high simulation counts, but E3W often outperforms uniform sampling at lower simulation counts.
    We speculate that this is because E3W preferentially samples greedily (with $\hat \pi_t$), so higher-reward actions are likely better-explored earlier than with uniform sampling. 
    This would explain why uniform sampling achieves similar performance with more simulations.
    These results support the results of prior work, namely \citet{xiao2019maximum_MENTS} and \citet{convex_regular_tree_search}, which showed that regularised MCTS algorithms could achieve superior performance to UCT in Atari games.%\footnote{These papers also do not report performance with $T$.}
    See \autoref{appendix:additional_octax_details} for more experiment details and \autoref{appendix:octax_plots} for plots of execution times, episode lengths, and mean returns for all 19 games.

    \begin{figure}[ht]
        \centering
        \includegraphics[width=\linewidth]{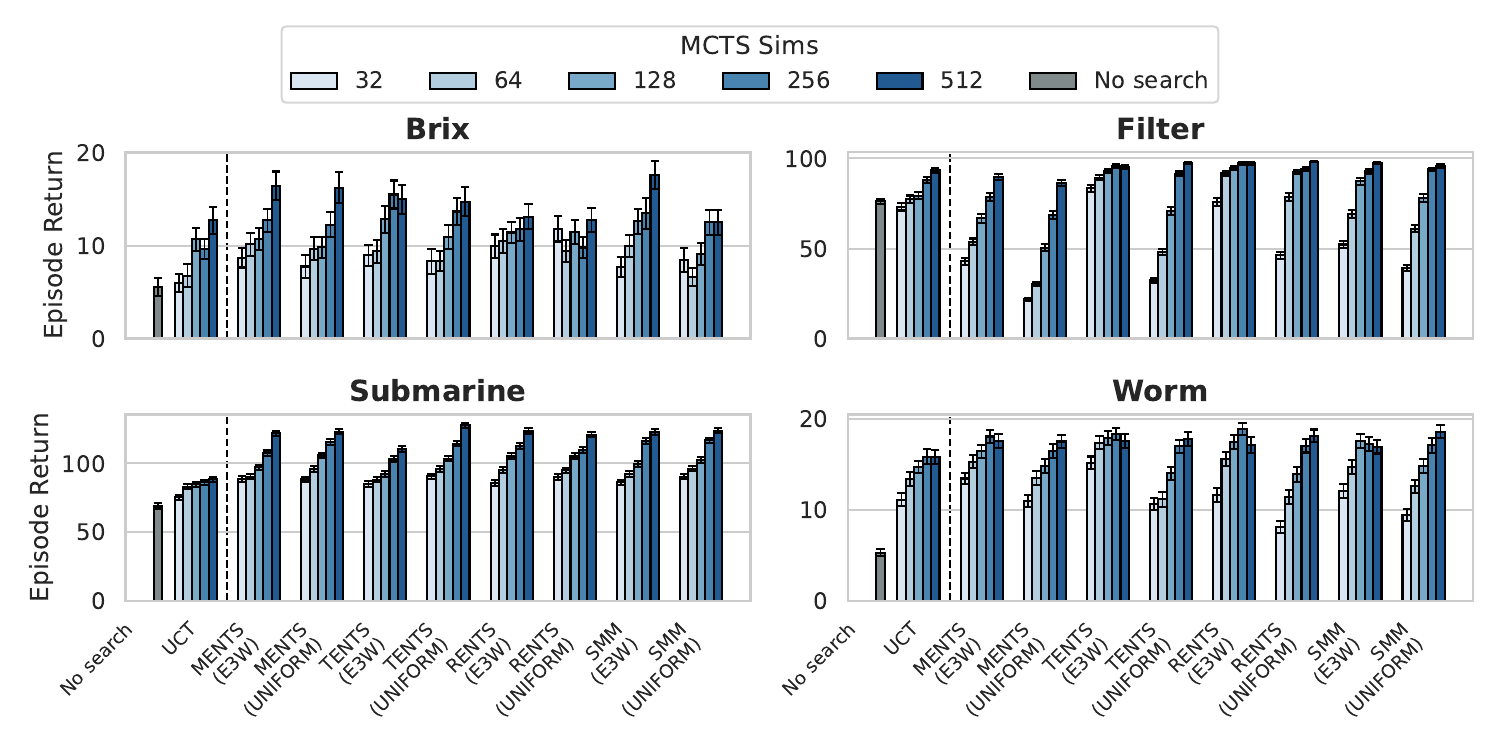}
        \caption{Episodic returns in select OCTAX tasks showing improved performance with increased MCTS simulations. The error bars represent the standard error, with 100 random seeds per setting.}
        \label{fig:octax_MCTS_long_sims_improvement}
    \end{figure}

\subsection{Scientific discovery} \label{sec:sequax_experiments}
    Following the experimental setup of \citet{TGM_2025_jiralerspong}, we now consider a number of biological sequence design tasks popular in the generative flow network (GFlowNet) literature \citep{trajectory_balance_GFN_2022, jain2023gflownets_discovery}. Using the protein- and DNA-sequence datasets of \citet{trabucco2022_sequence_design}, a transformer model is trained as an autoregressive proxy reward model. The proxy reward model is then fixed, and used to train a trajectory general mellowmax \citep[TGM]{TGM_2025_jiralerspong} model, which we use as the baseline, and to estimate values inside regularised MCTS. TGM is chosen over other baselines like GFlowNets, because it was shown to achieve state-of-the-art performance in the three tasks we evaluate on: anti-microbial peptide (AMP), green fluorescent protein (GFP), and 5’ untranslated region sequence (UTR), see \autoref{appendix:details_sci_disc} for details. 

    \textbf{Search evaluation.}
    We compare the TGM baseline model against MCTS-improved policies, sampling actions stochastically according to the policy estimator. This is because in biological sequence design, diversity is prioritised over return maximisation. 
    Simulation budgets are swept from 32 to 1024 (the 1024-simulation tier is evaluated only for AMP and UTR).
    We sweep the inverse temperature parameter $\omega$ over from $0.025$ to $40$, for each configuration in each task, including the baseline model (whose logits are scaled by $\omega$), and UCT (whose policy output temperature is set to $1/\omega$).
    We again set $\alpha=1$ for SMM.
    For each omega, we generate 100 sequences, over five random seeds. 
    
    \textbf{Metrics.} 
    Using the metrics in \citet{TGM_2025_jiralerspong} and \citet{trajectory_balance_GFN_2022}, 
    we evaluate the diversity of high-reward sequences generated by each configuration, reporting a diversity-aware top-$k$ reward: sequences are ranked by reward and greedily added to a selected set only if they are at least a task-specific Levenshtein distance $\delta$ from every sequence already selected. 
    This measures each configuration's coverage of distinct modes the reward landscape rather than its best sample.
    % Because a policy that repeatedly regenerates one high-reward sequence is far less useful than one that discovers many distinct high-reward sequences, we report a diversity-aware top-$k$ reward: sequences are ranked by reward and greedily added to a selected set only if they are at least a task-specific Levenshtein distance $\delta$ from every sequence already selected ($\delta=10$ for AMP, $\delta=60$ for GFP, $\delta=13$ for UTR); we report the mean reward of the selected set at $k \in \{1, 10, 20, 50, 100\}$. This measures each method's ability to cover distinct modes of the reward landscape rather than only its best-case sample.

    \textbf{Results.} \autoref{fig:MCTS_sci_disc_box_plots} illustrates the performance distribution of each configuration in the scientific discovery experiments.
    Notably, tree search-improved policies, particularly Tsallis entropy-based tree search (TENTS), were able to match or exceed the top-100 mean of the baseline in all tasks, especially with uniform sampling. 
    In fact, uniform sampling outperforms E3W for almost all search algorithms, in almost all tasks.
    This claim is further supported by \autoref{tab:top100_diverse_reward}, which lists the top-100 mean reward for all methods. While prior work \citep{morozov2024improving} improved GFlowNet policies with MCTS, they only reported similarity metrics between reward functions and MCTS-improved models on synthetic GFlowNet tasks. Thus, our results indicate for the first time that recommender-based regularised MCTS can be useful for real-world scientific discovery applications.
    
    \begin{figure}[th]
        \centering
        \includegraphics[width=\linewidth]{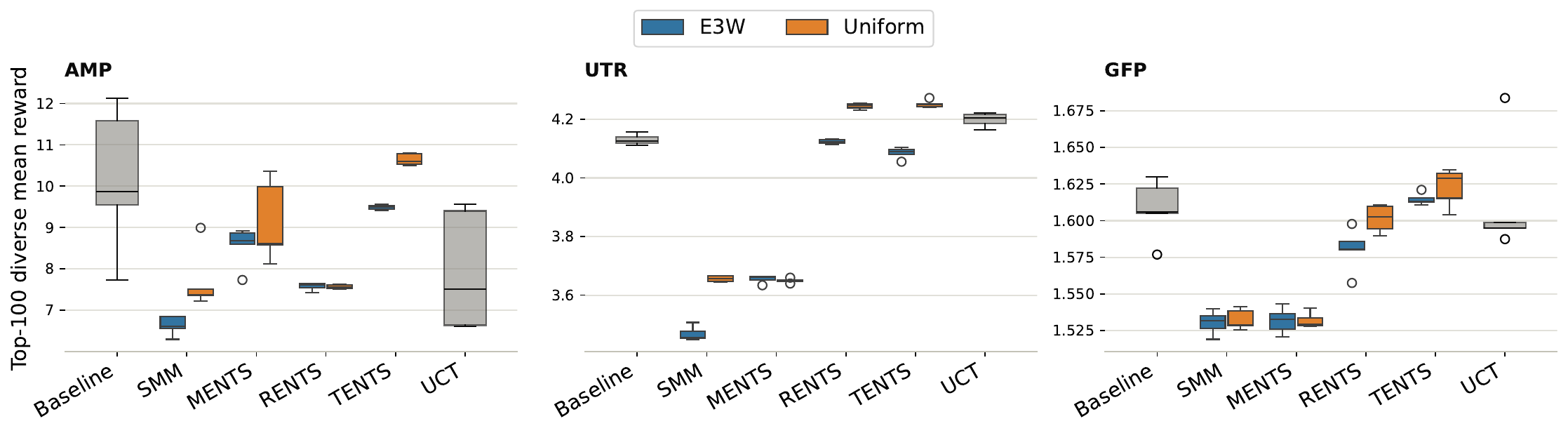}
        \caption{Top-100 diverse reward plots for the three scientific discovery tasks, with the TGM model as baseline, and as the value predictor in the MCTS tree.
        %Inverse temperatures were swept over 17 values, from $0.025$ to $40$, generating 100 sequences each, 1700 in total. After filtering out all sequences that were $\delta$-similar in the Levenshtein distance, the remaining top 100 were averaged to achieve the score. This process was repeated across 5 random seeds, to produce the box plots.
        Tree search is competitive with, or beats, the mean performance of the baseline. Uniform sampling outperforms E3W for almost all algorithms.
        }
        \label{fig:MCTS_sci_disc_box_plots}
    \end{figure}
\vspace{-10pt}
\section{Conclusion}
    We present recommender bandits, and new class of bandit-based estimators for regularised RL based on the multi-armed sampling framework. Theoretically, we show that our notions of recommender regret are well-posed, in that they produce non-trivial lower bounds, and that these bounds can be achieved with uniform policies, unlike the greedier policies required by multi-armed sampling. Experimentally, we demonstrate that recommender-motivated uniform policies for bandits and MCTS match or outperform prior work in synthetic trees, arcade games, and most notably, scientific discovery. Future work will tighten the upper bounds, using regulariser-specific algorithms.

\subsection*{AI use statement}

% (This section is \textbf{required} and does not count toward the page limit.)

In this work, we used generative AI tools for a number of tasks including
\begin{itemize}[nosep]
    \item proof checking
    \item proof ideation, especially proposing certain tools to use, or techniques
    \item to help with learning certain mathematical tools
    \item to transcribe hand-written notes to Latex
    \item to review the theorem and lemma statements for correctness
    \item to check code for correctness
    \item to write plotting and latex table generation code
    \item to help write some experiment code, particularly boilerplate evaluation scripts and bash scripts
    \item reviewing the writing of the paper, especially to focus narratives and reduce fluff
\end{itemize}

We have not used generative AI tools for writing the paper, and took great care to ensure that all the written work presented in this manuscript was human-generated, but perhaps AI-assisted. We have also not blindly followed the written code, nor blindly used any theoretical tools that were presented in the early phases of this work. We again, took great care to ensure that our work was sensible, and that it was honestly reported and conducted. The other AI disclosures are not relevant to this work.

We have reviewed all AI-assisted work. In particular, the core codebase has been written entirely by humans, and has been thoroughly checked, tested, and debugged. The plotting code was already written, but AI was used to modify it for different plots and objectives. These were always thoroughly vetted for correctness. All latex transcriptions were treated as a starting point, and modified to fit the math we wrote down. 

We take responsibility for the final content of this work,
including text, claims or artifacts produced with the aid of generative AI.

\bibliography{ref.bib}
\bibliographystyle{iclr2027_conference}

\appendix
\section{Proof of the Lower bounds}
\label{app:lb_proofs}
Let $\mathcal{SG}(\sigma^2)$ be the class of $\sigma^2$-sub-Gaussian bandits, let $\mu:=(\mu_1, \mu_2, \dots, \mu_K)$ denote a vector of means for a Gaussian $K$-Armed Bandit with variance $\sigma^2$ for all arms (i.e. the reward $r_{t,a} \sim \mathcal{N}(\mu_i, \sigma^2)$ for each arm $a \in [K]$) and $N_a(t)$ denote the number of times arm $a$ has been played at time $t \in [1,\dots, T]$. We also introduce a perturbation vector $\epsilon \in \mathbb{R}^K \neq \mathbf{0}$, and a \textit{perturbed} Gaussian instance whose means are perturbed by $\epsilon$, with means $\mu + \epsilon$. We refer to these two instances as $\mu$ and $\mu+\epsilon$ directly. Note that the mean vectors of each instance should be interpreted as the Q-values for these instances. As we are looking for a minimax lower bound on the class of $\sigma$-sub-Gaussian instances $\mathcal{SG}(\sigma^2)$, investigating Gaussian instances yields a valid lower bound. The instances $\mu$ and $\mu+\epsilon$ are in $\mathcal{SG}(\sigma^2)$. We now introduce a lemma on the number of times a pair of arm can be played. 

\begin{lemma}
Let $K \ge 2$ be the number of arms and $T \ge 1$ be the time horizon. For any sequence of arm selections generated by an algorithm $\mathcal{F}$, let $N_a(T)$ denote the total number of pulls of arm $a \in [K]$ up to time $T$. Then, there exists a distinct pair of arms $(i, j)$ with $i \neq j$ such that
\[
N_i(T) + N_j(T) \le \frac{2T}{K}.
\]
\end{lemma}

\begin{proof}
    Let $N_1(T) \leq N_2(T) \leq \dots \leq N_K(T)$ denote the number of pulls of each arm in a $K$-armed bandit. Then, $\sum_{a = 1}^K N_a(T) = T$.

    Now, suppose $N_1(T) + N_2(T) > \frac{2T}{K}$, then since $N_1(T) \leq N_2(T)$ it follows that 

    \[
    2 N_2(T) \geq N_1(T) + N_2(T) > \frac{2T}{K} \implies N_2(T) > \frac{T}{K}.
    \]

    For every arm $i \geq 3$, we have 
    \[
    N_i(T) \geq N_2(T) > \frac{T}{K}.
    \]

    Then,
    \begin{equation*}
    \begin{aligned}
        \sum_{a=1}^K N_a(T) &= N_1(T) + N_2(T) + \sum_{a=3}^K N_a(T) \\
        &> \frac{2T}{K} + (K-2)\frac{T}{K} \\
        &= T. \\ 
    \end{aligned}
    \end{equation*}
    This is a contradiction, as $\sum_{a=1}^K N_a(T) = T > T$.
\end{proof}

% Therefore, for the sake of the proof, we consider that $\epsilon$ has a single non-zero component, $i$. 
Note that although the indices rely on the sampling policies $d_t$ generated by $\mathcal{F}$, such a pair always exists regardless of $\mathcal{F}$ and policies $d_t$. 

% Notation details
For clarity, we denote $\E_{\mu}$ and $\PROB_{\mu}$ as the expected value and probability under the probability measure induced by the $T$ round interaction between instance $\phi$, algorithm $\mathcal{F}$ and its associated policies $d_t$. Similarly, $f^*_\phi$, where $f \in \{V_{\Omega}, Q_{\Omega}, \pi_{\Omega}\}$, refers to the target object when the instance is $\phi$. Lastly, we denote $\Delta_{\norm{\cdot}, f} = \norm{\fstar_\mu - \fstar_{\mu + \epsilon}}$ as a shorthand.

We provide two lemmas that will be used throughout the proofs. The first concerns recommender regret when the divergence $D$ is a norm $\norm{\cdot}$. The second is specifically when $D$ is the Bregman divergence $B_\Omega$. Throughout the proofs, $\mathbf{e}_a$ denotes a $K$-dimensional $a$-th indexed canonical basis vector and $\mathbbm{1}$ denotes a $K$-dimensional all-ones vector.
\begin{lemma}
\label{lemma:shared_norm_lb}
    Let $f$ be the object of interest, $T$ be the horizon, and denote $\Delta_{\norm{\cdot}, f} = \norm{\fstar_\mu - \fstar_{\mu + \epsilon}}$ for some arbitrary norm $\norm{\cdot}$. Given perturbation vectors $\epsilon \in \left\{ \sqrt{\frac{K}{2T}} \sigma \left(\mathbf{e}_i - \mathbf{e}_j\right), \sqrt{\frac{2}{T}}\sigma\mathbbm{1}\right\}$, where $(i,j)$ are the indices of the least played pair of arms under some arbitrary algorithm $\mathcal{F}$, the following holds for any subgaussian instance $\mathcal{SG}(\sigma^2)$: %then in the class of subgaussian instances, $\mathcal{SG}(\sigma^2)$, we have, for perturbation vectors $\epsilon \in \left\{ \sqrt{\frac{2K}{T}} \sigma \mathbf{e}_i,  \sqrt{\frac{2}{T}}\sigma\mathbbm{1}\right\}$
    \[
    \min_{\mathcal{F}}\max_{\phi \in \mathcal{SG}(\sigma^2)} \E_\phi\left[\norm{\frec - \fstar_\phi}\right] \geq \frac{\Delta_{\norm{\cdot}, f}}{8e}. 
    \]
\end{lemma}
\begin{proof}
     Consider the event $A:=\left\{\norm{\fstar_\mu - \frec}  \geq g_T\right\}$. Then, choosing any algorithm $\mathcal{F}$, we have

\begin{equation}
\begin{aligned}
  \max_{\phi \in \mathcal{SG}(\sigma^2)} \E_\phi\left[\norm{\frec - \fstar_\phi}\right] &\geq  \frac{\E_\mu\left[\norm{\frec - \fstar_\mu}\right] + \E_{\mu + \epsilon}\left[ \norm{\frec - \fstar_{\mu + \epsilon}}\right]}{2}. \\
  \end{aligned}
\end{equation}
Focusing on the numerator, we obtain
\begin{equation}
\begin{aligned}
  \E_\mu\left[\norm{\frec - \fstar_\mu}\right] + \E_{\mu + \epsilon}\left[ \norm{\frec - \fstar_{\mu + \epsilon} }\right]&\geq  \E_\mu\left[\norm{\frec - \fstar_\mu }\bigg|A\right]\mathbb{P}_\mu(A)  \\ &+ \E_{\mu + \epsilon}\left[\norm{\frec - \fstar_{\mu + \epsilon} }\bigg|A^c\right]\mathbb{P}_{{\mu + \epsilon}}(A^c) \\
&\geq  g_T\mathbb{P}_\mu(A) + (\Delta_{\norm{\cdot}, f} - g_T)\mathbb{P}_{{\mu + \epsilon}}(A^c), \\
\end{aligned}
\end{equation}
where the last inequality comes from the fact that $A,A^c$ occur to replace the expected values, combined with the triangle inequality. By choosing $g_T = \frac{\Delta_{\norm{\cdot}, f}}{2}$, we obtain
% , such that if Event $A$ occurs, we would believe we are in $\mu$ and in ${\mu + \epsilon}$ otherwise (interpreting the event as a hypothesis test).
% Then,

\begin{equation}
\begin{aligned}
  \E_\mu\left[\norm{\frec - \fstar_\mu}\right] + \E_{{\mu + \epsilon}}\left[\norm{\frec - \fstar_{{\mu + \epsilon}} }\right] &\geq \frac{\Delta_{\norm{\cdot}, f}}{2} \left(\mathbb{P}_\mu(A) +\mathbb{P}_{{\mu + \epsilon}}(A^c)\right) \\
    &\geq \frac{\Delta_{\norm{\cdot}, f}}{4} \left(\exp\left(-D_{KL}\left(\PROB_\mu || \PROB_{\mu + \epsilon}\right)\right) \right) \\
    &= \frac{\Delta_{\norm{\cdot}, f}}{4} \left(\exp\left(-\sum_{a=1}^K \expectwrt{N_a(T)}{\mu} \frac{\epsilon_a^2}{2\sigma^2}\right)\right). \\
    % &= \frac{\Delta_{\norm{\cdot}, f}}{4} \left(\exp\left(-\expectwrt{N_i(T)}{\mu} \frac{\epsilon_i^2}{2\sigma^2}\right)\right) \\
    % &\geq \frac{\Delta_{\norm{\cdot}, f}}{4} \left(\exp\left(-\frac{T}{K} \frac{\epsilon_i^2}{2\sigma^2}\right)\right). \\\
\end{aligned}
\end{equation}

The second line is obtained from the Bretagnolle-Huber inequality. The third line applies the KL-divergence decomposition lemma for bandits (Lemma 15.1, \citet{suttonReinforcementLearningSecond2020}), then take the KL-divergence for Gaussian distributions. Now, if we choose
$\epsilon = \sqrt{\frac{K}{2T}} \sigma (\mathbf{e}_i - \mathbf{e}_j)$, we have
\begin{equation}
\begin{aligned}
  \E_\mu\left[\norm{\frec - \fstar_\mu}\right] + \E_{{\mu + \epsilon}}\left[\norm{\frec - \fstar_{{\mu + \epsilon}} }\right] &\geq \frac{\Delta_{\norm{\cdot}, f}}{4} \left(\exp\left(-\sum_{a=1}^K \expectwrt{N_a(T)}{\mu} \frac{\epsilon_a^2}{2\sigma^2}\right)\right) \\
    &= \frac{\Delta_{\norm{\cdot}, f}}{4} \left(\exp\left(-\expectwrt{N_i(T)}{\mu} \frac{\epsilon_i^2}{2\sigma^2} - \expectwrt{N_j(T)}{\mu} \frac{\epsilon_j^2}{2\sigma^2}\right)\right) \\
     &\geq  \frac{\Delta_{\norm{\cdot}, f}}{4} \left(\exp\left(-\frac{2T}{K} \left(\frac{\epsilon_i^2}{2\sigma^2} + \frac{\epsilon_j^2}{2\sigma^2}\right)\right)\right) \\
    &= \frac{\Delta_{\norm{\cdot}, f}}{4} \left(\exp\left(-\frac{T}{K} \left(\frac{K}{2T} + \frac{K}{2T}\right)\right)\right) \\
    % &\geq \frac{\Delta_{\norm{\cdot}, f}}{4} \left(\exp\left(-\frac{T}{K} \frac{\epsilon_i^2}{2\sigma^2}\right)\right). \\
    &\geq \frac{\Delta_{\norm{\cdot}, f}}{4e}.
\end{aligned}
\end{equation}

% {\color{blue}
% $\epsilon$ such that $|\epsilon_i| \le \sqrt{\frac{2K}{T}} \sigma$ and $\min_j |\epsilon_j| \geq \sqrt{\frac{2}{KT}}$, we have
% \begin{align*}
%   \E_\mu\left[\norm{\frec - \fstar_\mu}\right] + \E_{{\mu + \epsilon}}\left[\norm{\frec - \fstar_{{\mu + \epsilon}} }\right] &\geq \frac{\Delta_{\norm{\cdot}, f}}{4} \left(\exp\left(-\sum_{a=1}^K \expectwrt{N_a(T)}{\mu} \frac{\epsilon_a^2}{2\sigma^2}\right)\right) \\
%     &= \frac{\Delta_{\norm{\cdot}, f}}{4} \left(\exp\left(-\expectwrt{N_i(T)}{\mu} \frac{\epsilon_i^2}{2\sigma^2}\right)\right) \\
%     &\geq \frac{\Delta_{\norm{\cdot}, f}}{4} \left(\exp\left(-\frac{T}{K} \frac{\epsilon_i^2}{2\sigma^2}\right)\right). \\
%     &\geq \frac{\Delta_{\norm{\cdot}, f}}{4e}.
% \end{align*}
% }

Similarly, using $\epsilon = \sqrt{\frac{2}{T}}\sigma\mathbbm{1}$ we obtain 
\begin{equation}
\begin{aligned}
  \E_\mu\left[\norm{\frec - \fstar_\mu}\right] + \E_{{\mu + \epsilon}}\left[\norm{\frec - \fstar_{{\mu + \epsilon}} }\right] &\geq \frac{\Delta_{\norm{\cdot}, f}}{4} \left(\exp\left(-\sum_{a=1}^K \expectwrt{N_a(T)}{\mu} \frac{\epsilon_a^2}{2\sigma^2}\right)\right) \\
  &= \frac{\Delta_{\norm{\cdot}, f}}{4} \left(\exp\left(-\frac{1}{T}\sum_{a=1}^K \expectwrt{N_a(T)}{\mu} \right)\right) \\
    &= \frac{\Delta_{\norm{\cdot}, f}}{4e}.
\end{aligned}
\end{equation}

Therefore, when $\epsilon \in \left\{ \sqrt{\frac{K}{2T}} \sigma \left(\mathbf{e}_i - \mathbf{e}_j\right),  \sqrt{\frac{2}{T}}\sigma\mathbbm{1}\right\}$, since all of the above hold for any arbitrary algorithm $\mathcal{F}$, we have that
    \[
    \min_{\mathcal{F}}\max_{\phi \in \mathcal{SG}(\sigma^2)} \E_\phi\left[\norm{\frec - \fstar_\phi}\right] \geq \frac{\Delta_{\norm{\cdot}, f}}{8e}.
    \]
\end{proof}

\begin{lemma}
\label{lemma:breg_lemma}
    Let $f$ be the object of interest, $T$ be the horizon, and denote $\Delta_{B, f} = \norm{\fstar_\mu - \fstar_{\mu + \epsilon}}^2$ for some arbitrary norm $\norm{\cdot}$. Given inverse temperature $\omega = 1$ and $\epsilon = \sqrt{\frac{K}{2T}} \sigma \left(\mathbf{e}_i - \mathbf{e}_j\right)$, where $(i,j)$ are the indices of the least played pair of arms under some arbitrary algorithm $\mathcal{F}$, the following holds for any subgaussian instance $\mathcal{SG}(\sigma^2)$: %then in the class of subgaussian instances, $\mathcal{SG}(\sigma^2)$, we have the following for $\epsilon = \sqrt{\frac{2K}{T}} \sigma \mathbf{e}_i$
    \[
    \min_{\mathcal{F}}\max_{\phi \in \mathcal{SG}(\sigma^2)} \E_\phi\left[B_{\Omega}(\frec, \fstar_\phi)\right] \geq \frac{\beta\Delta_{B, f}}{32e},
    \]
    where $\beta$ is the strong convexity parameter of regularizer $\Omega$.
\end{lemma}
\begin{proof}
By $\beta$-strong convexity of $\Omega$ on $\Pi$, $B_\Omega(p, q) \ge \frac{\beta}{2} \norm{p - q}^2$. Thus,
\begin{equation}
\begin{aligned}
  \max_{\phi \in \mathcal{SG}(\sigma^2)} \E_\mu\left[B_\Omega(\frec, \fstar_\phi)\right] &\geq \frac{\E_\mu\left[B_\Omega(\frec, \fstar_\mu)\right] + \E_{\mu + \epsilon}\left[B_\Omega(\frec, f^*_{\mu+\epsilon})\right]}{2} \\
  &\ge \frac{\beta}{4} \left( \E_\mu\left[\norm{\frec -  \fstar_\mu}^2 \right] + \E_{\mu + \epsilon}\left[\norm{\frec - \fstar_{\mu + \epsilon}}^2\right] \right).
\end{aligned}
\end{equation}

Let $\Delta_{B, f} = \norm{\fstar_\mu - \fstar_{\mu + \epsilon}}^2 = \Delta_{\norm{\cdot}, f}^2$. Starting from $\Delta_{B, f}$ and using the triangle inequality, followed by $(a+b)^2 \leq 2a^2 + 2b^2$, we find $\frac{\Delta_{B, \pi}}{2} \le \norm{\fstar_\mu - \frec}^2 + \norm{\fstar_{\mu + \epsilon} - \frec}^2$. Consider the event $A := \left\{\norm{\fstar_\mu - \frec}^2 \le g_T \right\}$, then we get
\begin{equation}
\begin{aligned}
  \E_\mu\left[\|\fstar_\mu - \frec\|^2\right] + \E_{\mu + \epsilon}\left[\|\fstar_{\mu + \epsilon} - \frec\|^2\right] 
  &\ge   g_T \PROB_\mu(A^c) + \left(\frac{\Delta_{B, f}}{2} - g_T\right)\PROB_{\mu + \epsilon}(A)  \\
  &\ge  \frac{\Delta_{B, f}}{4} \left( \PROB_\mu(A^c) +\PROB_{\mu + \epsilon}(A)\right)  \\
  &\ge \frac{\Delta_{B,f}}{8} \exp\left(-D_{KL}\left(\PROB_\mu || \PROB_{\mu + \epsilon}\right)\right) \\
  &\ge \frac{\Delta_{B, f}}{8e}.
\end{aligned}
\end{equation}

The second line results from choosing $g_T = \frac{\Delta_{B, \pi}}{4}$, the third applies Bretagnolle-Huber, and the last line uses the KL decomposition lemma then the fact that $\epsilon = \sqrt{\frac{K}{2T}} \sigma \left(\mathbf{e}_i - \mathbf{e}_j\right)$. As a result,
\begin{equation}
    \min_{\mathcal{F}}\max_{\phi \in \mathcal{SG}(\sigma^2)} \E_\phi\left[B_{\Omega}(\frec, \fstar)\right] \geq \frac{\beta\Delta_{B, f}}{32e}. 
\end{equation}
\end{proof}

\subsection{Values Recommender Regret (First half of Theorem~\ref{thm:lb_norm_q_v})}
\label{app:sec:values_proof}
We now establish the proof for the first half of Theorem~\ref{thm:lb_norm_q_v}: Value recommender regret. For this regret, the divergence $D(\Vrec, V^*_\Omega) = |\Vrec - V^*_\Omega|$ is the absolute value of the difference between the recommendation and the target value. 

\begin{proof}
For values, $f^* =V^*_{\Omega}=\Omega^*(Q^*_\Omega) $,  $f_T = \Vrec$, and $\norm{\cdot} = |\cdot|$. Therefore, $\Delta_{|\cdot|, V} := |V^*_{\Omega,\mu} - V^*_{\Omega,\mu+\epsilon}|$. Note that the values are dependent on the instance Q-values, as $V^*_{\Omega, \mu} = \Omega^*(Q^\star_{\Omega, \mu}) = \Omega^*(\mu)$ and similarly for the perturbed instance. By choosing $\epsilon = \sqrt{\frac{2}{T}}\sigma \mathbbm{1}$, we have that
% 
% \begin{equation}
%   \begin{aligned}
%     V^*_{\Omega, \mu+\epsilon} &= \max_{\pi} \langle \pi, \mu+\epsilon \rangle - \Omega(\pi) \\
%     &\geq  \langle \pistar_\mu, \mu+\epsilon \rangle - \Omega(\pistar_\mu) \\
%     &= \pistar_\mu(i) (\mu_i + \epsilon_i) + \sum_{j\neq i} \pistar_\mu(j) \mu_j - \Omega(\pistar_\mu). \\
%   \end{aligned}
% \end{equation}
% 
% Notice that 
\begin{equation}
  \begin{aligned}
    V^*_{\Omega, \mu+\epsilon} - V^*_{\Omega, \mu} &= \max_{\pi} \left[\langle \pi, \mu+\epsilon \rangle - \Omega(\pi)\right] - \max_{\pi} \left[\langle \pi, \mu \rangle - \Omega(\pi)\right] \\
    &\geq  \langle \pistar_\mu, \mu+\epsilon \rangle - \Omega(\pistar_\mu) - \langle \pistar_\mu, \mu \rangle + \Omega(\pistar_\mu) \\
    &=  \langle \pistar_\mu, \mu+\epsilon \rangle - \langle \pistar_\mu, \mu \rangle  \\
    &=  \sum_{a=1}^K \pistar_\mu(a) \epsilon_a  \\
    &=  \sqrt{\frac{2}{T}}\sigma.
    % &\geq  \langle \pistar_\mu, \mu+\epsilon \rangle - \Omega(\pistar_\mu) - \langle \pistar_\mu, \mu \rangle - \Omega(\pistar_\mu) \\
    % &=  \langle \pistar_\mu, \mu+\epsilon \rangle - \langle \pistar_\mu, \mu \rangle  \\
    % &=  \pistar_\mu(i) (\mu_i + \epsilon_i) + \sum_{j\neq i} \pistar_\mu(j) \mu_j - \pistar_\mu(i)\mu_i - \sum_{j\neq i} \pistar_\mu(j)\mu_j  \\
    % &=  \pistar_\mu(i) (\mu_i + \epsilon_i)  - \pistar_\mu(i)\mu_i \\
    % &=  \pistar_\mu(i) \epsilon_i. \\
  \end{aligned}
\end{equation}

Leveraging Lemma~\ref{lemma:shared_norm_lb}, we obtain
\begin{equation}
    \begin{aligned}
\frac{\Delta_{|\cdot|,f}}{8e} &= \frac{|V^*_{\Omega, \mu+\epsilon} - V^*_{\Omega, \mu}|}{8e}  \\
        &\geq \frac{\sqrt{2}\sigma}{\sqrt{T}8e}.
        % &\geq \frac{\pistar_\mu(i)\epsilon_i}{4e}  \\
        % &= \sqrt{\frac{2K}{T}}\frac{\sigma}{8e}.\\
    \end{aligned}
  \end{equation}
Therefore, we have the following lower bound on the values Recommender Regret:
\begin{equation}
\begin{aligned}
  \min_{\mathcal{F}} \max_{\phi \in \mathcal{SG}(\sigma^2)} \E_\phi\left[\norm{\Vrec - V^*_{\Omega, \phi}}\right] &\geq \frac{\sqrt{2}\sigma}{\sqrt{T}8e}. \\
  \end{aligned}
\end{equation}
\end{proof}
\subsection{Q-values Recommender Regret (Second half of Theorem~\ref{thm:lb_norm_q_v})}
For Q-values recommender regret we have $\fstar = Q^*_\Omega$ and $\frec = \murec$, $f^*_{\mu}=Q^\star_{\Omega,\mu} = \mu$, and $f^*_{\mu + \epsilon}=Q^\star_{\Omega,\mu+\epsilon} = \mu + \epsilon$. Then, $\Delta_{\norm{\cdot}_*, Q} = \norm{Q^*_\mu - Q^*_{\mu+\epsilon}}_*= \norm{\epsilon}_*$. 

\begin{proof}
 Recall that $(i,j)$ are the indices of least played pair of arms under some algorithm $\mathcal{F}$ and that such a pair is guaranteed to exist regardless of the algorithm. By choosing $\epsilon = \sqrt{\frac{K}{2T}} \sigma \left(\mathbf{e}_i - \mathbf{e}_j\right)$, and applying Lemma~\ref{lemma:shared_norm_lb}, we obtain

  \begin{equation}
    \begin{aligned}
       \frac{\Delta_{\norm{\cdot}_*,Q}}{8e} 
        &= \frac{\norm{\epsilon}_*}{8e}  \\
        % &= \frac{\sqrt{2K} \norm{\mathbf{e}_i \odot \sigma}}{\sqrt{T}4e}  \\
        &= \frac{\sqrt{K} \sigma \norm{\mathbf{e}_i - \mathbf{e}_j}_*}{\sqrt{2T}8e}.  \\
    \end{aligned}
  \end{equation}

Thus, we have the following lower bound on the Q-values Recommender Regret:
\begin{equation}
\begin{aligned}
  \min_{\mathcal{F}} \max_{\phi \in \mathcal{SG}(\sigma^2)} \E_\phi\left[\norm{\murec - Q^*_{\Omega, \phi}}_*\right] &\geq \sqrt{\frac{K}{2T}}\frac{\sigma \min_{a,a' \in [K], a\neq a'}\norm{\mathbf{e}_a - \mathbf{e}_{a'}}_*}{8e}. \\
  \end{aligned}
\end{equation}   

This is a consequence of the fact that any algorithm $\mathcal{F}$ has a least played pair of arms arm $(i,j)$, and so the lower bound for the $\max$ over instances must hold for the algorithm minimizing the right hand side.

\end{proof}

\subsection{Norm-based policy recommender regret (First half of Theorem~\ref{thm:lower_bound_smooth})}
\label{app:sec:norm_based_policy_regret_proof}
We now demonstate the lower bound for the norm-based policy recommender regret when Assumption~\ref{ass:smooth} is respected. For policy recommender regret we have $\fstar = \pistar_\Omega$ and $\frec = \pirec$. We now provide a lemma lower bounding $\Delta_{\norm{\cdot}, \pi}$.

\begin{lemma}
\label{lemma:delta_norm_pi_lb}
Let $(i,j)$ denote the indices of the least played pair of arms under some algorithm $\mathcal{F}$. When the object of interest is $\pi$, and the divergence for recommender regret is $\norm{\cdot}$, then the following lower bound for $\Delta_{\norm{\cdot}, \pi}$ with perturbation $\epsilon = \sqrt{\frac{K}{2T}}\sigma \left(\mathbf{e}_i - \mathbf{e}_j\right)$ holds for $L$-smooth, $\beta$-strongly convex regulariser $\Omega$
    \[
    \Delta_{\norm{\cdot}, \pi} \geq \sqrt{\frac{K}{2T}}\frac{\sigma \min_{a, a' \in [K], a\neq a'}\norm{\mathbf{e}_a - \mathbf{e}_{a'}}_*}{L}.
    \]
\end{lemma}
\begin{proof}
%For this proof, we choose $\epsilon = \sqrt{\frac{2K}{T}}\sigma \mathbf{e}_i$. Notice 
We first observe that
\[
\Delta_{\norm{\cdot}, \pi} = \norm{\pistar_{\mu + \epsilon} - \pistar_{\mu}} = \norm{\nabla \Omega^*(Q^*_{\Omega, \mu}) - \nabla \Omega^*(Q^*_{\Omega, \mu+\epsilon})}= \norm{\nabla \Omega^*(\mu + \epsilon) - \nabla \Omega^*(\mu)}.
\]
From Proposition~\ref{prop:suf_cond_lb_pol}, $\Omega^*$ is $\frac{1}{L}$-strongly convex on the space of policies $\Pi$. Using $\epsilon = \sqrt{\frac{K}{2T}}\sigma \left(\mathbf{e}_i -\mathbf{e}_j\right)$, and recalling that $(i,j)$ are the indices of the least played pair of arms under some algorithm $\mathcal{F}$, we obtain
\begin{equation}
\begin{aligned}
    \inner{\nabla \Omega^*(\mu) - \nabla \Omega^*(\mu + \epsilon), \epsilon} &\geq \frac{1}{L}\norm{\epsilon}_{*,E}^2 \quad & \text{(Proposition~\ref{prop:suf_cond_lb_pol})}\\
    \norm{\nabla \Omega^*(\mu) - \nabla \Omega^*(\mu + \epsilon)}\norm{\epsilon}_{*,E}&\geq \frac{1}{L}\norm{\epsilon}_{*,E}^2  \quad & \text{(Dual norm inequality)}\\
    \Delta_{\norm{\cdot}, \pi} &\geq \frac{1}{L}\norm{\epsilon}_{*,E} \\
    &\geq \sqrt{\frac{K}{2T}}\frac{\sigma \norm{\mathbf{e}_i - \mathbf{e}_j}_{*,E}}{L} \\
    &\geq \sqrt{\frac{K}{2T}}\frac{\sigma \min_{a,a' \in [K]}\norm{\mathbf{e}_a - \mathbf{e}_{a'}}_{*,E}}{L}. \\
\end{aligned}
\end{equation}
\end{proof}

Then, for the first half of Theorem~\ref{thm:lower_bound_smooth} we have this simple proof.

\begin{proof}
Combining Lemmas~\ref{lemma:shared_norm_lb} and \ref{lemma:delta_norm_pi_lb}, we get
\begin{equation}
\begin{aligned}
  \min_{\mathcal{F}} \max_{\phi \in \mathcal{SG}(\sigma^2)} \E_\phi\left[\norm{\pirec- \pistar_\phi }\right] &\geq \sqrt{\frac{K}{2T}}\frac{\sigma \min_{a,a' \in [K]}\norm{\mathbf{e}_a - \mathbf{e}_{a'}}_{*,E}}{8L}. \\
  \end{aligned}
\end{equation}
\end{proof}

\subsection{Bregman Divergence policy recommender regret (Second half of Theorem~\ref{thm:lower_bound_smooth})}
\label{app:sec:breg_policy_regret_proof}
We now demonstrate the lower bound for the Bregman divergence policy recommender regret when Assumption~\ref{ass:smooth} is respected. For policy recommender regret we have $\fstar = \pistar_\Omega$ and $\frec = \pirec$.

\begin{proof}
%For this proof, we choose $\epsilon = \sqrt{\frac{2K}{T}}\sigma \mathbf{e}_i$. 
Starting from Lemma~\ref{lemma:breg_lemma}, we must lower bound $\Delta_{B, \pi}$. However, notice that $\Delta_{B, \pi} = \Delta_{\norm{\cdot}, \pi}^2$. Using Lemma \ref{lemma:delta_norm_pi_lb} with $\epsilon = \sqrt{\frac{K}{2T}}\sigma \left(\mathbf{e}_i- \mathbf{e}_j\right)$, we obtain
\begin{equation}
\begin{aligned}
    \Delta_{B, \pi} =\Delta_{\norm{\cdot}, \pi}^2 &\geq \frac{K}{2T}\frac{\sigma^2 \min_{a,a' \in [K], a\neq a'}\norm{\mathbf{e}_a- \mathbf{e}_{a'}}_{*,E}^2}{L^2} \\
\end{aligned}
\end{equation}
directly, as taking the square is monotonically increasing on the range $[0, \infty]$.

% \begin{equation}
%     \begin{aligned}
%     \Delta &= \|\pistar_\mu - \pistar_{\mu + \epsilon}\|^2 \\ 
%     &\geq \frac{\norm{\epsilon}_*^2}{L^2} \\ 
%     &= \frac{2KR^2 \norm{\mathbf{e}_i}_*^2}{TL^2} \\   
%     &\geq \frac{2KR^2 \min_{j\in [K]}\norm{\mathbf{e}_j}_*^2}{TL^2}   
%     \end{aligned}
% \end{equation}
Combining with Lemma~\ref{lemma:breg_lemma}, we get our result
\begin{equation}
\begin{aligned}
  \min_{\mathcal{F}} \max_{\phi \in \mathcal{SG}(\sigma^2)} \E_\phi\left[B_\Omega(\pirec, \pistar_\phi)\right] &\geq \frac{K}{T}\frac{\beta \sigma^2  \min_{a,a' \in [K], a\neq a'}\norm{\mathbf{e}_a- \mathbf{e}_{a'}}_*^2}{64L^2e}  . \\
  \end{aligned}  
\end{equation}
\end{proof}

% \section{Proof of Proposition~\ref{prop:sufficient_cond_smooth}}
\label{app:sufficient_cond_smooth}

\section{Proof of Proposition~\ref{prop:suf_cond_lb_pol}}
\label{app:suff_cond_lb_pol}

% {\color{blue}

% $\Omega^*(q) = \sup_{\pi \in \Delta} \langle q, \pi \rangle - \Omega(\pi) $

% $\Omega^*(q + c \textbf{1}) = \sup_{\pi \in \Delta} \langle q + c \textbf{1}, \pi \rangle - \Omega(\pi) 
% = c + \sup_{\pi \in \Delta} \langle q, \pi \rangle - \Omega(\pi)
% = c + \Omega^*(q).
% $

% }
Let us first recall the statement of the proposition:

Given a set $\mathcal{Q} \subset \mathbb{R}^\mathcal{A}$, such that $\Pi := \nabla \Omega^*(\mathcal{Q})$ is a subset of $\text{ri}(\Delta_\mathcal{A})$,\footnote{For a convex set $C$, the notation $\text{ri}(C)$ stands for its relative interior.} if, for all $\pi \in \text{Hull}(\Pi)$ and $v \in E$, we have $\langle\nabla_E^2 \Omega(\pi)v,v\rangle \leq L \|v\|^2$, then for any $Q,Q' \in \mathcal{Q}$, we have $\langle\nabla \Omega^*(Q)-\nabla \Omega^*(Q'), Q- Q'\rangle \geq \frac{1}{L} \|Q -Q'\|_{*,E}^2$.

\begin{proof}
    Let us consider $\pi = \nabla \Omega^*(Q)$ and $\pi' = \nabla \Omega^*(Q')$. By~\citep[Proposition 1]{geist2019theory_regularisers} we have that $\pi = \arg \max_{\pi' \in \Delta_A} \langle \pi', Q\rangle - \Omega(\pi')  $ and since $\pi$ is in the relative interior of the simplex we have that the first order optimality condition becomes
    \begin{equation}
        \nabla_E \Omega( \pi) = PQ
    \end{equation}
    where $P=I-\tfrac{1}{|\mathcal A|}\mathbf 1\mathbf 1^\top$ is the orthogonal projection onto $E$. Similarly, we have $ \nabla_E \Omega( \pi') = PQ'$. Now let us consider
    \begin{equation}
        P(Q-Q') = \nabla_E \Omega( \pi) - \nabla_E \Omega( \pi') = \int_0^1 \nabla^2_E \Omega (\pi' + t (\pi - \pi'))(\pi - \pi') dt = H d
    \end{equation}
    where we called $H=  \nabla_E \Omega( \pi) - \nabla_E \Omega( \pi') = \int_0^1 \nabla^2_E \Omega (\pi' + t (\pi - \pi'))dt$ and $d =\pi - \pi'$.
    Now, using Cauchy-Schwartz inequality for the inner product $\langle H\cdot,\cdot\rangle$ we get 
    \begin{equation}
        \langle H d,v\rangle^2 \leq \langle Hd,d\rangle \langle Hv,v\rangle \leq  \langle Hd,d\rangle L \|v\|^2 
    \end{equation}
    where for the last inequality we used for all $\pi \in \text{Hull}(\Pi)$ and $v \in E$, we have $\langle\nabla^2_E \Omega(\pi)v,v\rangle \leq L \|v\|^2 $. Taking the supremum over $v \in E, \|v\|=1$ we get 
    \begin{equation}
        \| H d\|^2_{\star,E} \leq  \langle Hd,d\rangle L 
    \end{equation}
    Now we can look at,
    \begin{align}
        \langle\nabla \Omega^*(Q)-\nabla \Omega^*(Q'), Q- Q'\rangle 
        &= \langle\pi-\pi', Q- Q'\rangle \\
        &= \langle P(\pi-\pi'), Q- Q'\rangle \\
        &= \langle\pi-\pi', P(Q- Q')\rangle \\
        &= \langle d, Hd\rangle \\
        & \geq \frac{1}{L}\| H d\|^2_{\star,E}  \\
        & = \frac{1}{L}\| Q- Q'\|^2_{\star,E} 
    \end{align}
    which concludes the proof. 
\end{proof}

\section{Supporting theoretical results}
    \begin{lemma}[Q-value error bounds value and policy error] \label{lemma:q_val_convergence_implies_value_policy}
        Given a regulariser $\Omega: \Delta_\mathcal{A} \to \mathbb{R}$ that is $\beta$-strongly convex with respect to a norm $\norm{\cdot}$, and whose dual is the regularised RL value function, 
            $$V_\Omega(s) = \Omega^*(Q(s, \cdot)) = \max_\pi \left\{ \expectwrt{Q(s,a)}{a\sim \pi(\cdot|s)} - \frac{1}{\omega}\Omega(\pi(\cdot|s)) \right\}.$$
        The policy and value errors can be bounded as follows:
        \begin{equation} \label{equ:prop_q_values_bound_everything}
            \begin{split}
                &\norm{\hat \pi - \pi^\star} = \norm{\nabla \Omega^*(\hat Q) - \nabla \Omega^*(Q^\star)} \leq \frac{\omega}{\beta} \norm{\hat Q - Q^\star}_\star \\
                &D_\Omega(\hat \pi, \pi^\star) \leq \frac{\omega^2}{2\beta} \norm{\hat Q - Q^\star}^2_\star \\
                &\left| \hat{V} - V^\star \right| = \left| \Omega^*(\hat Q) - \Omega^*(Q^\star) \right| \leq \norm{\hat Q - Q^\star}_\infty, \\
            \end{split}
        \end{equation}
        where $(\hat \pi, \hat Q)$ and $(\pi^\star, Q^\star)$ are primal-dual pairs, with the $Q$s not necessarily being unique.
    \end{lemma}
    \begin{proof}[Proof of Lemma~\ref{lemma:q_val_convergence_implies_value_policy}]
        The policy error bounds are both the result of dual smoothness. In short, dual smoothness guarantees that when the primal function $\Omega$ is closed and convex, the dual function, $\Omega^*$, is smooth. From \citep[Lemma 15(3)]{shalev_shwartz_phd_thesis_online_learning_convexity} and \citep[Theorem 3]{kakade2012regularization_dual_smooth_2}, we have that $\frac{1}{\omega} \Omega$ is $\beta/\omega$-convex, thus admitting a dual smoothness constant $\omega/\beta$: 
        $$\norm{\nabla \Omega^*(\hat Q) - \nabla \Omega^*(Q^\star)} \leq \frac{\omega}{\beta} \norm{\hat Q - Q^\star}_\star.$$
        Similarly, the Bregman divergence can be upper-bounded using the equivalency between the dual and primal Bregman divergence, and the $\frac{\omega}{\beta}$-smoothness property of $\Omega^*$ via \citep[Theorem 3.7(5)]{bauschke1997legendre} and \citep[Lemma 3.4]{bubeck2015convex_Lsmooth}, respectively. Therefore, we have that
        \begin{align*}
           \frac{1}{\omega}D_\Omega(\hat \pi, \pi^\star) &=  D_{\Omega^*}(Q^\star, \hat Q) \\
        &= \Omega^*(Q^\star) - \Omega^*(\hat Q) - \inner{\nabla \Omega^*(\hat Q), Q^\star - \hat Q} \\
            &\leq \frac{\omega}{2\beta} \norm{\hat Q - Q^\star}^2_\star,
        \end{align*}
        where the first line states the primal-dual Bregman divergence equivalency, and the third line uses the smoothness upper bound. The bound is finished by multiplying through by $\omega$. Note that we have to divide by $\omega$ in the first place to state the divergence in terms of the function $\frac{1}{\omega}\Omega$.

        Finally, we know that the regularised value functions $V = \Omega^*(Q)$ are non-expansive mappings via \citep[Lemma 8, Corollary 9]{nachum2017bridging}, whose results hold for any strongly convex regulariser, despite only being proven for negative entropy. This gives:
        $$\left| \hat{V} - V^\star \right| = \left| \Omega^*(\hat Q) - \Omega^*(Q^\star) \right| \leq \norm{\hat Q - Q^\star}_\infty.$$
    \end{proof}

     \begin{lemma}[Policies with non-zero probability mass achieve optimal rates on the Q-value regret] \label{lemma:any_policy_achieves_optimal_bounds}
            Given some sampling policy $d_t$ such that $d_t(a) \geq \zeta>0, \forall a \in \mathcal{A}, \forall t \in [1, T]$, where $\hat{Q}_t = \hat{r}_t$ (\autoref{equ:reg_rl_optimality_bandit_equs}), and where the reward distributions $\phi_a$ are $\sigma$-sub-Gaussian, and $\norm{\mathbf e'}_\star = \max_a \norm{\mathbf{e}_a}_\star$ where $\mathbf{e}_a$ is the canonical basis vector indexed by $a$, the following Q-value error bound holds for any $t \geq \frac{2\ln(2K/\delta)}{\zeta^2}$ with probability $1-\delta$ when $\norm{\cdot}_\star = \norm{\cdot}_\infty$
            \begin{equation} \label{equ:lemma_any_policy_zeta_bound_infinity}
                \norm{\hat{Q}_T - Q^\star}_\infty \leq \sqrt{\frac{4 \sigma^2\ln(4KT / \delta)}{\zeta T}},
            \end{equation}
            and otherwise for any $\norm{\cdot}_\star$ with probability $1-\delta$
            \begin{equation} \label{equ:lemma_any_policy_zeta_bound}
                \norm{\hat{Q}_T - Q^\star}_\star \leq K \norm{\mathbf e'}_\star \sqrt{\frac{4 \sigma^2\ln(4KT / \delta)}{\zeta T}}.
            \end{equation}
        \end{lemma}
        \begin{proof}[Proof of Lemma~\ref{lemma:any_policy_achieves_optimal_bounds}]

            We begin by the bounding the action selection count using a high-probability proof technique via the Azuma-Hoeffding concentration inequality \citep{wainwright2019high_concentration_inequalities}. For arm $a$ up to time $t$, let $N_t(a) = \sum_{i=1}^{t} \mathbb{I}(a_i = a)$ be its empirical tally of samples, and let $\hat{N}_t(a) = \sum_{i=1}^{t} d_i(a) \geq \zeta t$ be the sum of its empirical sampling probabilities. To bound the difference between these two terms, we have two options for modelling action selection: 1) as a Bernoulli distribution at each time step using the Hoeffding inequality, and 2) as a martingale difference sequence using the Azuma-Hoeffding inequality. Both reach the same result, but we opt for the latter. 

            To use Azuma-Hoeffding we must prove that $X_t = d_t(a|\mathcal{H}_t) - \mathbb{I}(a = a_t)$ is a martingale difference sequence. This is immediate from $|X_t| \leq 1$ and
            $$\expect{X_t | \mathcal{H}_t} = \expect{d_t(a|\mathcal{H}_t) - \mathbb{I}(a = a_t) | \mathcal{H}_t} = d_t(a|\mathcal{H}_t) - d_t(a|\mathcal{H}_t) = 0.$$ 
            Applying the one-sided Azuma-Hoeffding inequality (with $X_0 = 0$) implies that for a single arm $a$:
            \begin{equation*}
                \mathbb{P}\left( \hat{N}_t(a) - N_t(a) > \epsilon_1 \right) \leq \exp \left( - \frac{2\epsilon_1^2}{t} \right).
            \end{equation*}
            Setting $\epsilon_1 =  \sqrt{\frac{t \ln(1/\delta_1)}{2}}$, this inequality implies that with probability $1-\delta_1$ (taking the complement event),
            $$N_t(a) \geq \hat{N}_t(a) - \sqrt{\frac{t \ln(1/\delta_1)}{2}} \geq \zeta t - \sqrt{\frac{t \ln(1/\delta_1)}{2}}.$$
            This implies that for any $t \geq T_0 =  \frac{2\ln(1/\delta_1)}{\zeta^2}$, the action count will be bound by $N_t(a) \geq \frac{\zeta t}{2}$.

            Defining the event $E_a = \left\{ \hat{N}_t(a) - N_t(a) \leq \sqrt{\frac{t \ln(1/\delta_1)}{2}} \right\}$ to represent these action count probabilities, and the more general event over all actions $E = \bigcap_{a\in\mathcal{A}}E_a$, we have that $\prob{E} \geq 1 - K\delta_1$.

            We now begin deriving the Q-value error bound. Recall that $\hat{Q}_t(a) = \sum_{i=1}^{t} r_{i, a_i} \mathbb{I}(a_i = a) / N_t(a)$, and that $Q^\star(a) = \mu_a$. This leads to the following relation:
            \begin{align*}
                \left|\hat{Q}_t(a) - Q^\star(a)\right| &= \left|\sum_{i=1}^{t} \frac{r_{i, a_i} \mathbb{I}(a_i = a)}{N_t(a)} - \mu_a \right| \\
                &= \frac{1}{N_t(a)}\left|\sum_{i=1}^{t} (r_{i, a_i} - \mu_a) \mathbb{I}(a_i = a) \right|
            \end{align*}
            which can be used in a Hoeffding inequality, by rearranging the equation. By setting the error $\epsilon_2 = \sqrt{\frac{2\sigma^2 \ln(1/\delta_2)}{N_t(a)}}$, we have the following:
            \begin{align*}
                \prob{ |\hat{Q}_t(a) - Q^\star(a)| > \epsilon_2} &= \prob{ \left|\sum_{i=1}^{t} (r_{i, a_i} - \mu_a) \mathbb{I}(a_i = a) \right| > N_t(a)\epsilon_2} \\
                &\leq 2 \exp \left( - \frac{\epsilon_2^2 N_t(a)}{2 \sigma^2} \right) = 2 \delta_2.
            \end{align*}
            To bound the absolute error across all actions and time steps (which we must do in order to take into account all the entire history $\mathcal{H}_t$), we can define an event on the reward error $G_{a,i} = \left\{ \left|\hat{Q}_i(a) - Q^\star(a)\right| \leq  \sqrt{\frac{2\sigma^2 \ln(1/\delta_2)}{N_i(a)}} \right\}$, and a general event $G = \bigcap_{a\in \mathcal{A}} \bigcap_{i\in[1,t]} G_{a,i}$. It is now clear that $\prob{G} \geq 1 - 2Kt\delta_2$.

            Finally, we can solve for $\delta_1$ and $\delta_2$, so that we can get a $1-\delta$ statement.
            \begin{equation} \label{equ:E_cap_G_event_prob}
                \prob{E \cap G} \geq 1 - \prob{E^c} - \prob{G^c} \geq 1 - K\delta_1 - 2Kt\delta_2 = 1-\delta,
            \end{equation}
            setting $\delta_1 = \frac{\delta}{2K}$ and $\delta_2 = \frac{\delta}{4Kt}$. Thus under the intersection event, we have that for each $a \in \mathcal{A}$, for $t \geq T_0$:
            \begin{equation} \label{equ:q_error_hoeffding_bound_lemma}
                |\hat{Q}_t(a) - Q^\star(a)| \leq \sqrt{\frac{2\sigma^2 \ln(4Kt/\delta)}{N_t(a)}} \leq \sqrt{\frac{4\sigma^2 \ln(4Kt/\delta)}{\zeta t}}.
            \end{equation}
            This statement is enough to prove the special case where $\norm{\cdot}_\star = \norm{\cdot}_\infty$, since we don't need to convert the dual norm to a suitable form for bounding under the concentration argument, above.

            Finally, we can map this back to our dual-norm regret, let $\mathbf{e}_a$ represent the standard basis vector for arm $a$. By the triangle inequality, any error vector can be decomposed via its coordinate-wise $\ell_\infty$ maximum:
            \begin{equation*}
                \norm{\hat{Q}_t - Q^\star}_\star = \norm{\sum_{a \in \mathcal{A}} (\hat{Q}_t(a) - Q^\star(a))\mathbf{e}_a}_\star \leq \norm{\hat{Q}_t - Q^\star}_\infty \sum_{a \in \mathcal{A}} \norm{\mathbf{e}_a}_\star \leq  K \norm{\mathbf e'}_\star \norm{\hat{Q}_t - Q^\star}_\infty
            \end{equation*}
            which proves the lemma after applying the Q-error inequality from \Eqref{equ:q_error_hoeffding_bound_lemma}, and with $\norm{\mathbf e'}_\star = \max_a \norm{\mathbf{e}_a}_\star$.
        \end{proof}

    \section{Proof of \autoref{theorem:general_upper_bound}} \label{appendix:proof_bandit_theorem}

        % \begin{lemma}[Conversion of instance-dependent upper bounds to expectation] \label{lemma:upper_bound_conversion_to_expectation}
        %     Under the conditions of Lemma~\ref{lemma:q_val_convergence_implies_value_policy} and Lemma~\ref{lemma:any_policy_achieves_optimal_bounds}, setting $\delta = \mathcal{O}(1/T^2)$ yields expected simple recommender regret bounds of order:
        %     $$\expect{\mathcal{SRR}^{f, D}_T} \in \Tilde{O}\left( \sqrt{\frac{1}{\zeta T}} \right), \quad \text{and} \quad \expect{\mathcal{SRR}^{\pi, B_\Omega}_T} \in \Tilde{O}\left( \frac{1}{\zeta T} \right), $$
        %     for non-Bregman divergence-based and Bregman divergence-based regrets, respectively. Where $\Tilde O$ leaves out slow-growing logarithmic terms, as was done in \citet{pedramfar2026multiarmedsampling}.
        % \end{lemma}
        \setcounter{theorem_custom_num}{4}
        \begin{theorem_custom_num}[Sampling policies with non-zero probability mass achieve optimal regret bounds]
            Suppose $\Omega$ is $\beta$-strongly convex with dual $V_\Omega(s) = \Omega^*(Q(s, \cdot))$, with $\sigma$-sub-Gaussian reward distributions $\phi_a, \forall a \in \mathcal{A}$. For any $T \geq \frac{2\ln(2K/\delta)}{\zeta^2}$, with $\delta \in O(1/T^2)$, and any algorithm with sampling policy $d_t$ satisfying $d_t(a) \geq \zeta > 0$ $\forall a \in \mathcal{A}, \forall t \in [1, T]$, the following expected regret bounds hold:
            % \begin{equation} \label{equ:theorem_regret_bounds}
            %     \expect{\mathcal{SRR}^{f, D}_T} \in \Tilde{O} \left( \sqrt{\frac{1}{\zeta T}}\right);\quad \quad \expect{\mathcal{SRR}^{\pi, B_\Omega}_T} \in \Tilde{O} \left(\frac{1}{\zeta T}\right),
            % \end{equation}

            \begin{enumerate}[leftmargin=*, nosep]
                \item When $\norm{\cdot}_\star = \norm{\cdot}_\infty$ or when $f = V$, then
                \begin{equation*} 
                    \expect{\mathcal{SRR}^{f, D}_T} \in \tilde{O} \left( \sqrt{\frac{1}{\zeta T}}\right), \qquad \expect{\mathcal{SRR}^{\pi, B_\Omega}_T} \in \tilde{O} \left(\frac{1}{\zeta T}\right).
                \end{equation*}
                
                \item When $\norm{\cdot}_\star$ is a general dual norm with $\norm{\mathbf e'}_\star = \max_{a} \norm{\mathbf{e}_a}_\star$:
                \begin{equation*}
                    \expect{\mathcal{SRR}^{f, D}_T} \in \tilde{O} \left( K \norm{\mathbf e'}_\star \sqrt{\frac{1}{\zeta T}}\right), \qquad \expect{\mathcal{SRR}^{\pi, B_\Omega}_T} \in \tilde{O} \left( K^2 \norm{\mathbf e'}_\star^2 \frac{1}{\zeta T}\right),
                \end{equation*}
            \end{enumerate}
            where the terms on the left bound norm-based policy, value, and Q-value regrets, and the terms on the right bound Bregman divergence-based regrets. Where $\Tilde O$ leaves out slow-growing logarithmic terms, as was done in \citet{pedramfar2026multiarmedsampling}.
        \end{theorem_custom_num}
        \begin{proof}
            From Lemma~\ref{lemma:q_val_convergence_implies_value_policy} and Lemma~\ref{lemma:any_policy_achieves_optimal_bounds}, we have high probability upper bounds on instances of the simple recommender regret, as reflected in \autoref{equ:lemma_instance_regret_bounds} for general dual norms $\norm{\cdot}_\star$. 
            \begin{equation} \label{equ:lemma_instance_regret_bounds}
                \begin{split}
                    \mathcal{SRR}^{Q, \norm{\cdot}_\star}_T &= \norm{\hat{Q}_{\Omega,T} - Q^\star}_\star \leq K \norm{\mathbf e'}_\star \sqrt{\frac{4 \sigma^2\ln(4KT / \delta)}{\zeta T}} \\
                    \mathcal{SRR}^{\pi, \norm{\cdot}}_T &= \norm{\hat \pi_{\Omega,T} - \pi^\star} \leq \frac{\omega}{\beta} K \norm{\mathbf e'}_\star \sqrt{\frac{4\sigma^2 \ln(4KT/\delta)}{\zeta T}} \\
                    \mathcal{SRR}^{\pi, B_\Omega}_T &= B_\Omega(\hat \pi_{\Omega,T}, \pi^\star) \leq \frac{\omega^2}{2\beta} K^2 \norm{\mathbf e'}_\star^2 \frac{4\sigma^2 \ln(4KT/\delta)}{\zeta T} \\
                    \mathcal{SRR}^{V, |\cdot|}_T &= \left| \hat{V}_{\Omega,T} - V^\star \right| \leq \sqrt{\frac{4\sigma^2 \ln(4KT/\delta)}{\zeta T}}.
                \end{split}
            \end{equation}

            Note that when $\norm{\cdot}_\star = \norm{\cdot}_\infty$, each of the above bounds become tighter, as is evident from Lemma~\ref{lemma:any_policy_achieves_optimal_bounds}. \autoref{equ:lemma_instance_regret_bounds_infinity} lists the bounds for the $L_\infty$ dual norm.
            \begin{equation} \label{equ:lemma_instance_regret_bounds_infinity}
                \begin{split}
                    \mathcal{SRR}^{Q, \norm{\cdot}_\infty}_T &= \norm{\hat{Q}_{\Omega,T} - Q^\star}_\infty \leq \sqrt{\frac{4 \sigma^2\ln(4KT / \delta)}{\zeta T}} \\
                    \mathcal{SRR}^{\pi, \norm{\cdot}}_T &= \norm{\hat \pi_{\Omega,T} - \pi^\star} \leq \frac{\omega}{\beta} \sqrt{\frac{4\sigma^2 \ln(4KT/\delta)}{\zeta T}} \\
                    \mathcal{SRR}^{\pi, B_\Omega}_T &= B_\Omega(\hat \pi_{\Omega,T}, \pi^\star) \leq \frac{\omega^2}{2\beta} \frac{4\sigma^2 \ln(4KT/\delta)}{\zeta T} \\
                    \mathcal{SRR}^{V, |\cdot|}_T &= \left| \hat{V}_{\Omega,T} - V^\star \right| \leq \sqrt{\frac{4\sigma^2 \ln(4KT/\delta)}{\zeta T}}.
                \end{split}
            \end{equation}

            The principal difference between the regret bounds of \autoref{equ:lemma_instance_regret_bounds} and \autoref{equ:lemma_instance_regret_bounds_infinity} are the constants in front of the Q-error concentration bound. We will henceforth use the constant $C$, to denote the constant, which can depend on $K$, as is evident, above.

            These bounds are conditioned on the fact that, with probability $1-\delta$, the event $E \cap G$ holds, see \autoref{equ:E_cap_G_event_prob}. We re-write the events here, for clarity:
            $$G = \bigcap_{a\in \mathcal{A}} \bigcap_{i\in[1,T]} \left\{ \left|\hat{Q}_i(a) - Q^\star(a)\right| \leq  \sqrt{\frac{2\sigma^2 \ln(4KT/\delta)}{N_i(a)}} \right\}$$
            $$E = \bigcap_{a\in \mathcal{A}}  \left\{ \hat{N}_T(a) - N_T(a) \leq \sqrt{\frac{T \ln(2K/\delta)}{2}} \right\}.$$

            Then, for the non-Bregman divergences, we have that for some $T \geq T_0$, there exists a constant $C$ that depends on the type of recommender regret and the conversion between the dual norm to the $L_\infty$ norm, such that
            \begin{align*}
                \expect{\mathcal{SRR}^{f, D}_T} &= \expect{\mathcal{SRR}^{f, D}_T| E \cap G} \prob{E \cap G} + \expect{\mathcal{SRR}^{f, D}_T | E^c \cup G^c} \prob{E^c \cup G^c} \\
                &\leq C \sqrt{\frac{4\sigma^2 \ln(4KT/\delta)}{\zeta T}} + \delta C \expect{ \norm{\hat{Q}_T - Q^\star}_\infty \Bigg| E^c \cup G^c}.
            \end{align*}
            The term on the right is irreducible under the complement event, since we have no relations on $N_T(a)$, nor $\norm{\hat{Q}_T - Q^\star}_\infty$ that provide bounds on the term in the expectation. Thus, we use the following properties of independent sub-Gaussian random variables:
            \begin{itemize}[nosep]
                \item We know that random variables $Y_{i, a} \sim \text{subG}(\sigma^2), i \in [1, N_T(a)]$ are sub-Guassian by definition if $\expect{Y_{i,a}} = 0$ 
                $$\mathbb{E}\left[ \exp \left(\lambda Y_{i,a} \right) \right] \leq \exp \left(\frac{\lambda^2 \sigma^2}{2} \right);\forall \lambda \in \mathbb{R}.$$
                We may then trivially achieve zero-mean sub-Gaussian random variables on our reward distributions by setting $Y_{i,a} = r_{i,a} - \mu_a$.
                \citep[Definition 1.2]{high_dim_statistic_MIT_lecture_notes}
                \item We also know that sums of sub-Gaussian random variables are sub-Gaussian. The following is a specific case of \citep[Theorem 1.6]{high_dim_statistic_MIT_lecture_notes}. For a random vector of sub-Gaussian random variables $Y_a = (Y_{i,a_i} : i \in [1, T], a_i = a)$ and a vector of constant factors $z_a = \frac{1}{N_T(a)} \mathbbm{1}$, such that $z_a^\intercal Y_a = \frac{1}{N_T(a)}\sum_{i=1}^{N_T(a)}Y_{i, a}$:
                $$\expect{\exp \left( \lambda z_a^\intercal Y_a \right)} = \prod_{i = 1}^{N_T(a)} \expect{\exp \left( \frac{\lambda Y_{i,a}}{N_T(a)} \right)} \leq \prod_{i = 1}^{N_T(a)} \exp \left(\frac{\lambda^2 \sigma^2}{2N_T^2(a)} \right) \leq \exp \left(\frac{\lambda^2 \sigma^2}{2} \right),$$
                where the last inequality follows by $1/N_T(a) \leq 1$.
                \item Finally, following \citep[Theorem 1.14]{high_dim_statistic_MIT_lecture_notes}, define $\hat{Y}_T = (z_a^\intercal Y_a : a \in \mathcal{A})$, which, by the above, is a sub-Gaussian random vector with variance proxy $\sigma^2$. We have that 
                $$\expect{\max_{a \in \mathcal{A}} |\hat Y_{T,a}|} \leq \sigma \sqrt{2\log(2K)}; \quad \prob{\max_{a \in \mathcal{A}} |\hat Y_{T,a}| >j} \leq 2K \exp \left(-\frac{j^2}{2\sigma^2} \right)$$
            \end{itemize}
            
            We can now continue by defining $\hat Y_{T,a} = \hat Q_{T}(a) - Q^\star(a) = \frac{1}{N_T(a)} \sum_{i=1}^{T} \mathbb{I}(a_i = a) \left( r_{i, a_i} - \mu_a \right)$.
            Note that we will have to evaluate moments of $\max_{a \in \mathcal{A}} |\hat Y_{T,a}|$, so the following will prove useful.
            \begin{align*}
                \expect{\norm{\hat{Q}_T - Q^\star}^{2p}_\infty} &=
                \expect{\max_{a \in \mathcal{A}} \left|\hat Y_{T,a} \right|^{2p}} \\
                &= \int_0^\infty \prob{\max_{a \in \mathcal{A}} \left|\hat Y_{T,a} \right|^{2p} > j} dj \\
                &= \int_0^\infty \prob{\max_{a \in \mathcal{A}} \left|\hat Y_{T,a} \right| > j^{1/2p}} dj \\
                &= \int_0^u \prob{\max_{a \in \mathcal{A}} \left|\hat Y_{T,a} \right| > j^{1/2p}} dj + \int_u^\infty \prob{\max_{a \in \mathcal{A}} \left|\hat Y_{T,a} \right| > j^{1/2p}} dj  \\
                &\leq \int_0^u 1dt + 2K \int_u^\infty \exp \left(-\frac{j^{1/p}}{2\sigma^2} \right) dj  \\
                &= u + 2K \int_u^\infty \exp \left(-\frac{j^{1/p}}{2\sigma^2} \right) dj.
            \end{align*}

            We choose $u$ to be a convenient constant: $u = (2\sigma^2 \ln(2K))^p$. Then via a change of variables, we will evaluate the remaining integral term: $h = j^{1/p}$ such that $dj = p h^{p-1}dh$, with the lower integration bound becoming $u^{1/p} = 2\sigma^2 \ln(2K)$.
            \begin{align*}
                \int_u^\infty \exp \left(-\frac{j^{1/p}}{2\sigma^2} \right) dj
                &= \int_{2\sigma^2 \ln(2K)}^\infty \exp \left(-\frac{h}{2\sigma^2} \right) p h^{p-1}dh. 
            \end{align*}

            Once again changing the integration bounds to recover a gamma-like function, let $x = \frac{h}{2\sigma^2}$, and $dh = 2\sigma^2 dx$, with the lower integration bound now being $\ln(2K)$:
            \begin{align*}
                \int_{2\sigma^2 \ln(2K)}^\infty \exp \left(-\frac{h}{2\sigma^2} \right) p h^{p-1}dh 
                &= \int_{\ln(2K)}^\infty \exp \left(-x \right) p (2\sigma^2 x)^{p-1} 2\sigma^2 dx \\ 
                &= p (2\sigma^2)^p \int_{\ln(2K)}^\infty x^{p-1} \exp \left(-x \right) dx \\
                &\leq p (2\sigma^2)^p e^{-\ln(2K)} \sum_{k=0}^{p-1} \frac{(p-1)!}{k!} \ln(2K)^k \\
                &=  (2\sigma^2)^p \frac{1}{2K} \sum_{k=0}^{p-1} \frac{p!}{k!} \ln(2K)^k, \\
            \end{align*}
            which is a similar result to what is shown in \citep[Theorem 2.6]{wainwright2019high}.

            We only care about powers of 2 and 4, since these will be used later, so taking $p=1,2$, we find the following bounds:
            \begin{align*}
                \expect{\norm{\hat{Q}_T - Q^\star}^{2}_\infty} &\leq 2\sigma^2 \left(\ln(2K) + 1\right) \\
                \expect{\norm{\hat{Q}_T - Q^\star}^{4}_\infty} &\leq 4\sigma^4 \left(\ln(2K)^2 + 2\ln(2K) + 2\right).
            \end{align*}
            
            % Noticing that $\norm{\hat{Q}_T - Q^\star}_\infty = \max_{a \in \mathcal{A}} |\hat{Q}_T(a) - Q^\star(a)|$, the following holds:
            % $$\expect{\norm{\hat{Q}_T - Q^\star}_\infty} = \expect{\max_{a \in \mathcal{A}} \left|\hat{Q}_T(a) - Q^\star(a)\right|} \leq \sigma \sqrt{2\log(2K)}$$

            Back to our difficult-to-reduce expectation, the following holds, using Cauchy-Schwartz:
            \begin{align*}
                \expect{ \norm{\hat{Q}_T - Q^\star}_\infty \Bigg| E^c \cup G^c} \prob{E^c \cup G^c}
                &= \expect{ \norm{\hat{Q}_T - Q^\star}_\infty \mathbb{I}(E^c \cup G^c)} \\
                &\leq \sqrt{\expect{ \norm{\hat{Q}_T - Q^\star}^2_\infty}} \sqrt{\expect{\mathbb{I}(E^c \cup G^c)}} \\
                &\leq \sigma \sqrt{2\delta(\log(2K) + 1)}.
            \end{align*}
            Thus, we have that 
            \begin{align*}
                \expect{\mathcal{SRR}^{f, D}_T} &= \expect{\mathcal{SRR}^{f, D}_T| E \cap G} \prob{E \cap G} + \expect{\mathcal{SRR}^{f, D}_T | E^c \cup G^c} \prob{E^c \cup G^c} \\
                &\leq C \sqrt{\frac{4\sigma^2 \ln(4KT/\delta)}{\zeta T}} + \sigma C \sqrt{2\delta(\log(2K)+1)}.
            \end{align*}
            
            For Bregman divergences, we follow the same proof technique, just with squared Q-value error, leading to a similar result.
            \begin{align*}
                \expect{\mathcal{SRR}^{\pi, B_\Omega}_T} &= \expect{\mathcal{SRR}^{\pi, B_\Omega}_T| E \cap G} \prob{E \cap G} + \expect{\mathcal{SRR}^{\pi, B_\Omega}_T | E^c \cup G^c} \prob{E^c \cup G^c} \\
                &\leq C \frac{4\sigma^2 \ln(4KT/\delta)}{\zeta T} + 2 \sigma^2 C \sqrt{\delta\left(\ln(2K)^2 + 2\ln(2K) + 2\right)}.
            \end{align*}
            Allowing us to conclude that, so long as $\delta \in O(1/T^2)$, $\expect{\mathcal{SRR}^{f, D}_T} \in \Tilde O \left( C \sqrt{\frac{1}{\zeta T}} \right)$ for the non-Bregman divergences, and $\expect{\mathcal{SRR}^{\pi, B_\Omega}_T} \in \Tilde O \left( \frac{C}{\zeta T} \right)$ for the Bregman divergence-based regret. By choosing the dual norm, we can finally achieve the statements in the theorem, by cross-referencing \autoref{equ:lemma_instance_regret_bounds} and \autoref{equ:lemma_instance_regret_bounds_infinity}.

            Note that we chose the above proof technique to avoid bounding the reward distributions. If instead we bounded each $|r_{i,a_i}| \leq R_\text{max} < \infty$, we would quite simply have that 
            \begin{align*}
                \expect{\mathcal{SRR}^{f, D}_T | E^c \cup G^c} \prob{E^c \cup G^c} 
                &\leq \delta C \expect{ \norm{\hat{Q}_T - Q^\star}_\infty \Bigg| E^c \cup G^c} \\
                &\leq 2CR_\text{max}\delta.
            \end{align*}
            Then with $\delta \in O(1/T)$ we have both bounds presented above.
        \end{proof}

\section{Proof of Theorem~\ref{thrm:MCTS_SRR_high_prob_bound}} \label{appendix:MCTS_THEOREM_PROOF}
        
        \begin{proof}[Proof of Theorem~\ref{thrm:MCTS_SRR_high_prob_bound}]
            Using Lemma~\ref{lemma:q_val_convergence_implies_value_policy} with the fact that $\norm{\hat{Q}_{\Omega,t} - Q^\star}_\star \leq  C_\mathcal{A}\norm{\hat{Q}_{\Omega,t} - Q^\star}_\infty$ for some constant $C_\mathcal{A}$ that depends on the dual norm and action space, we know that there exists a constant $C_3$ such that
            \begin{align*}
                \mathcal{SRR}^{f, D}_t \leq \begin{cases}
                    C_3^2\norm{\hat{Q}_{\Omega,t}(s_0, \cdot) - Q^\star(s_0, \cdot)}_\infty^2 & \text{if } D \in \left\{ B_\Omega \right\} \\
                    C_3\norm{\hat{Q}_{\Omega,t}(s_0, \cdot) - Q^\star(s_0, \cdot)}_\infty & \text{else.}
                \end{cases}
            \end{align*}
            Without loss of generality, we proceed with the linear bound, noting that for Bregman divergences, we can afford a higher $\mathcal{SRR}$ lower bound in the probability statement: $\epsilon^2$, which yields the same exponential concentration rate, after square rooting both sides of the inequality in the probability.
    
            Our goal is now to analyse the following term at various tree depths $H$, for an inductive proof:
            \begin{align*}
                \prob{\mathcal{SRR}^{f, D}_t \geq \epsilon} 
                &\leq \prob{\norm{\hat{Q}_{\Omega,t}(s_0, \cdot) - Q^\star(s_0, \cdot)}_\infty \geq \frac{\epsilon}{C_3}}. 
            \end{align*}
    
            Namely, we seek to prove that $\exists C_1, C_2 \in \mathbb{R}_{>0}$ at the root $s_0$ of a search tree such that
            $$\prob{\norm{\hat{Q}_{\Omega,t}(s_0, \cdot) - Q^\star(s_0, \cdot)}_\infty \geq \frac{\epsilon}{C_3}} \leq C_1 \exp \left(- \frac{N_t(s_0) \epsilon^2}{C_2} \right).$$
    
            Let $N_t(s,a) = \sum_{i=1}^{t} \mathbb{I}(a_i = a)\mathbb{I}(s_i = s)$ be the empirical tally of samples, $N_t(s_0) := \sum_{a\in\mathcal{A}}N_t(s_0, a) = t$, and let $\hat{N}_t(s,a) = \sum_{i=1}^{t} d_i(a|s) \geq \zeta t$ be the sum of sampling policy $d_i$ for a given state-action pair.
    
            \textbf{Base case, with $H=1$}
            In the bandit setting, we can simply use Hoeffding's inequality to achieve the desired concentration bound. We will reuse elements of the proof technique from Lemma~\ref{lemma:any_policy_achieves_optimal_bounds}, specifically the concentration of action selection counts via the Azuma-Hoeffding inequality on martingale difference sequences \citep{wainwright2019high_concentration_inequalities}. Define the event $E_a = \left\{ \hat{N}_t(s_0,a) - N_t(s_0,a) \leq \epsilon_1 \right\}$ and the more general event over all actions $E = \bigcap_{a\in\mathcal{A}}E_a$. When conditioning on $E_a$, it will be useful to set $\epsilon_1 = \zeta N_t(s_0)/2$, so that we have $N_t(s_0, a) \geq \zeta N_t(s_0)/2$ (since $\hat{N}_t(s_0, a) \geq \zeta N_t(s_0)$). It is clear that 
            \begin{align*}
                \prob{E_a^c} &\leq \exp \left(- \frac{2\epsilon_1^2}{N_t(s_0)} \right) \\
                &= \exp \left(- \frac{2}{N_t(s_0)} \frac{\zeta^2 (N_t(s_0))^2}{4} \right) \\
                &= \exp \left(- \frac{\zeta^2 N_t(s_0)}{2} \right).
            \end{align*}
    
            Define another event on the reward error $G_{a,t} = \left\{ \left|\hat{Q}_{\Omega,t}(s_0,a) - Q^\star_\Omega(s_0,a)\right| \geq  \epsilon_2 \right\}$, and a general event $G_t = \bigcup_{a\in \mathcal{A}} G_{a,t}$. Notice how $\prob{\norm{\hat{Q}_{\Omega,t}(s_0, \cdot) - Q^\star(s_0, \cdot)}_\infty \geq \epsilon_2} \leq \prob{G_t}$. Using Hoeffding's inequality (under the assumption that $N_t(s_0, a)$ is not a random variable), it is clear that
            \begin{align*}
                \prob{ G_{a, t}} &= \prob{ \left|\sum_{i=1}^{t} (r_{i, a_i} - Q^\star(s_0, a)) \mathbb{I}(a_i = a) \right| > N_t(s_0, a)\epsilon_2} \\
                &\leq 2 \exp \left( - \frac{\epsilon_2^2 N_t(s_0, a)}{2 \sigma^2} \right).
            \end{align*}

            To finally bound $\prob{G_t}$, we have to use the event $E_a$ to convert between child and parent action counts, and we have to convert those action counts from random variables to deterministic sample counts. Since we know that $N_t(s_0, a) \in [\zeta N_t(s_0)/2, N_t(s_0)]$ under the event $E_a$, we can sum over those action counts (through the union bound) to make the count deterministic:
            \begin{align*}
                \prob{ G_{a, t} | E_a}
                &= \prob{ \left|\hat{Q}_{\Omega,t}(s_0,a) - Q^\star_\Omega(s_0,a)\right| \geq  \epsilon_2 \Bigg| E_a} \\
                &= \prob{ \exists n = N_t(s_0, a) \in \left[ \frac{\zeta N_t(s_0)}{2}, N_t(s_0) \right]  : \left|\hat{Q}_{\Omega,t}(s_0,a) - Q^\star_\Omega(s_0,a)\right| \geq  \epsilon_2 \Bigg| E_a} \\
                &\leq \prob{ \bigcup_{n=\zeta N_t(s_0)/2}^{N_t(s_0)} \left\{ n = N_t(s_0, a) :  \left|\hat{Q}_{\Omega,t}(s_0,a) - Q^\star_\Omega(s_0,a)\right| \geq  \epsilon_2 \right\} \Bigg| E_a} \\
                &\leq \sum_{n=\zeta N_t(s_0)/2}^{N_t(s_0)} \prob{ \left\{ n = N_t(s_0, a) :  \left|\hat{Q}_{\Omega,t}(s_0,a) - Q^\star_\Omega(s_0,a)\right| \geq  \epsilon_2 \right\}} \\
                &\leq 2 \sum_{n=\zeta N_t(s_0)/2}^{N_t(s_0)} \exp \left( - \frac{\epsilon_2^2 n}{2 \sigma^2} \right) \\
                &\leq 2 N_t(s_0) \exp \left( - \frac{\epsilon_2^2 \zeta N_t(s_0)}{4 \sigma^2} \right).
            \end{align*}

            Completing the bound on $\prob{G_t}$:
            \begin{align*}
                \prob{G_t} &= \prob{G_t|E}\prob{E} + \prob{G_t|E^c}\prob{E^c} \\
                &\leq \prob{G_t|E} + \prob{E^c} \\
                &= \sum_{a\in\mathcal{A}} \prob{ G_{a, t} | E_a} + \prob{E^c_a} \\
                &\leq 2 K N_t(s_0) \exp \left( - \frac{\epsilon_2^2 \zeta N_t(s_0)}{4 \sigma^2} \right) + K \exp \left(- \frac{\zeta^2 N_t(s_0)}{2} \right) \\
                &\leq 2 K N_t(s_0) \exp \left( - \frac{\epsilon_2^2 \zeta N_t(s_0)}{4 \sigma^2} \right) + K N_t(s_0) \exp \left(- \frac{\zeta^2 N_t(s_0)}{2} \right).
            \end{align*}
    
            From here, we need to draw some relations between $\zeta$ and $\epsilon_2$ in order to get the bound into the correct form. We set $\epsilon_2 = \epsilon/C_3 \leq \zeta$ thus setting the assumption in the theorem statement. This invariably affects the size that $T$ must take in order to satisfy the upper bound (i.e. $T$ must be very large since $\epsilon_2$ is smaller than $\zeta$). Thus, using the facts that $-1 \leq -\zeta$ and that $-\zeta \leq -\epsilon_2$, we have the following bound
            \begin{align*}
                \prob{G_t} 
                &\leq 2 K N_t(s_0) \exp \left( - \frac{\epsilon_2^2 \zeta N_t(s_0)}{4 \sigma^2} \right) + K N_t(s_0) \exp \left(- \frac{\zeta^2 N_t(s_0)}{2} \right) \\
                &\leq 2 K N_t(s_0) \exp \left( - \frac{\epsilon_2^2 \zeta N_t(s_0)}{4 \sigma^2} \right) + K N_t(s_0) \exp \left(- \frac{\epsilon_2^2 \zeta N_t(s_0)}{2} \right) \\
                &\leq 3K N_t(s_0) \exp \left( - \frac{\epsilon_2^2 \zeta N_t(s_0)}{\max(4 \sigma^2, 2)} \right) \\
                &\leq 3K N_t(s_0) \exp \left( - \frac{\epsilon_2^2 \zeta N_t(s_0)}{2(2 \sigma^2+1)} \right),
            \end{align*}
            where the last inequality come from the fact that $a+b \geq \max(a,b), \forall a, b \in \mathbb{R}_{\geq 0}$.
    
            Finally, we use the relation between $G_t$ and the $L_\infty$ Q-value error to show the base case:
            \begin{align*}
                \prob{\norm{\hat{Q}_{\Omega,t}(s_0, \cdot) - Q^\star(s_0, \cdot)}_\infty \geq \frac{\epsilon}{C_3}}
                &\leq \prob{G_t}, \quad (\epsilon_2 = \epsilon/C_3) \\
                &\leq 3K N_t(s_0) \exp \left( - \frac{\epsilon^2 \zeta N_t(s_0)}{2C_3^2(2 \sigma^2+1)} \right) \\
                &= Z_1^{(1)} \exp \left( - \frac{\epsilon^2 N_t(s_0)}{Z_2^{(1)} C_3^2} \right),
            \end{align*}
            with $Z_1^{(1)} = 3K N_t(s_0)$ and $Z_2^{(1)} = \frac{\zeta}{2(2 \sigma^2+1)}$.
            
            \textbf{Inductive step}
            Assuming that the inequality holds for each depth $H-1$ tree, we will prove that it also holds for a search tree of depth $H$. We will reuse the events $E_a$ and $E$, since they were applied at the root of our bandit, but equally apply to the root of a search tree. Let both Q-value estimators use the same deterministic reward, where $s_a$ is the next state from state $s_0$ after $a$ is taken: $\hat{Q}_{\Omega,t}(s_0, a) = r(s_0, a) + \gamma \Omega^*(\hat{Q}_{\Omega,t}(s_a, \cdot))$ and $Q^\star_{\Omega,t}(s_0, a) = r(s_0, a) + \gamma \Omega^*(Q^\star_{\Omega,t}(s_a, \cdot))$. Therefore, we have that
            \begin{align*}
                \prob{\norm{\hat{Q}_{\Omega,t}(s_0, \cdot) - Q^\star(s_0, \cdot)}_\infty \geq \frac{\epsilon}{C_3}}
                &\leq \prob{\sum_{a\in\mathcal{A}} \left| \hat{Q}_{\Omega,t}(s_0, a) - Q^\star(s_0, a) \right| \geq \frac{\epsilon}{C_3} } \\
                &= \prob{\sum_{a\in\mathcal{A}} \gamma \left| \Omega^*(\hat{Q}_{\Omega,t}(s_a, \cdot))) - \Omega^* (Q^\star(s_a, \cdot)) \right| \geq \frac{\epsilon}{C_3} } \\
                &\leq \prob{\sum_{a\in\mathcal{A}}  \norm{\hat{Q}_{\Omega,t}(s_a, \cdot) - Q^\star(s_a, \cdot)}_\infty \geq \frac{\epsilon}{C_3\gamma } } \\
                &\leq \prob{\bigcup_{a\in\mathcal{A}} \left\{ \norm{\hat{Q}_{\Omega,t}(s_a, \cdot) - Q^\star(s_a, \cdot)}_\infty \geq \frac{\epsilon}{KC_3\gamma}\right\} } \\
                &\leq \sum_{a\in\mathcal{A}} \prob{ \norm{\hat{Q}_{\Omega,t}(s_a, \cdot) - Q^\star(s_a, \cdot)}_\infty \geq \frac{\epsilon}{C_3 K \gamma}},
            \end{align*}
            where the second line follows from the Q-value definitions, the third follows from Lemma~\ref{lemma:q_val_convergence_implies_value_policy}, the fourth follows from the union bound (setting the resulting constant on the bound to $1/K$ so that the constants sum to 1), and the final inequality follows from the union bound.
    
            Now, let $G'_{s_a} = \left\{ \norm{\hat{Q}_{\Omega,t}(s_a, \cdot) - Q^\star(s_a, \cdot)}_\infty \geq \frac{\epsilon}{C_3 K \gamma} \right\}$. We follow the same proof technique from the base case, conditioning on $E$ and solving for the upper bound:
            \begin{align*}
                \prob{G'_{s_a}} &= \prob{G'_{s_a}|E}\prob{E} + \prob{G'_{s_a}|E^c}\prob{E^c} \\
                &\leq \prob{G'_{s_a}|E} + \prob{E^c} \\
                &\leq Z_1^{(H-1)} N_t(s_0) \exp \left(- \frac{ \zeta N_t(s_0) \epsilon^2}{2 Z_2^{(H-1)} C_3^2 K^2 \gamma^2} \right) + K \exp \left(- \frac{\zeta^2 N_t(s_0)}{2} \right) \\
                &\leq Z_1^{(H-1)} N_t(s_0) \exp \left(- \frac{\zeta N_t(s_0) \epsilon^2}{2 Z_2^{(H-1)} C_3^2 K^2 \gamma^2} \right) + K N_t(s_0) \exp \left(- \frac{\zeta N_t(s_0) \epsilon^2}{2} \right) & (\text{using } \zeta^2 \geq \epsilon^2) \\
                &\leq (Z_1^{(H-1)}+K) N_t(s_0) \exp \left(- \frac{\zeta N_t(s_0) \epsilon^2}{2 (Z_2^{(H-1)} C_3^2 K^2 \gamma^2 + 1)}\right).
            \end{align*}

            Finally, we have that the root Q-value error concentrates given that the sub-trees concentrate:
            \begin{align*}
                \prob{\norm{\hat{Q}_{\Omega,t}(s_0, \cdot) - Q^\star(s_0, \cdot)}_\infty \geq \frac{\epsilon}{C_3}}
                &\leq \sum_{a \in \mathcal{A}} \prob{G'_{s_a}} \\
                &\leq K(Z_1^{(H-1)}+K) N_t(s_0) \exp \left(- \frac{\zeta N_t(s_0) \epsilon^2}{2 (Z_2^{(H-1)} C_3^2 K^2 \gamma^2 + 1)} \right) \\
                &\leq Z_1^{(H)} \exp \left(- \frac{N_t(s_0) \epsilon^2}{Z_2^{(H)}} \right),
            \end{align*}
            for $Z_1^{(H)} = K(Z_1^{(H-1)} + K) N_t(s_0)$ and $Z_2^{(H)} = \frac{2(Z_2^{(H-1)} C_3^2 K^2 \gamma^2 + 1)}{\zeta}$. Analysing these relations for arbitrary tree depths $H$, we find the following relation, which will be used to construct our final bounds, later:
            \begin{itemize}
                \item \textbf{For } $\mathbf{Z_1^{(H)}}$: 
                    \begin{align*}
                        Z_1^{(H)} &= K(Z_1^{(H-1)} + K) t \propto Kt Z_1^{(H-1)} \\
                        &= K(K(Z_1^{(H-2)} + K) t + K) t \propto K^2t^2 Z_1^{(H-2)} \\
                        &\quad \vdots \\
                        &\propto (Kt)^H
                    \end{align*}
                \item \textbf{For } $\mathbf{Z_2^{(H)}}$: 
                    \begin{align*}
                        Z_2^{(H)} &= \frac{2(Z_2^{(H-1)} C_3^2 K^2 \gamma^2 + 1)}{\zeta} \propto \frac{K^2}{\zeta} Z_2^{(H-1)} \\
                        &= \frac{2(\frac{2(Z_2^{(H-2)} C_3^2 K^2 \gamma^2 + 1)}{\zeta} C_3^2 K^2 \gamma^2 + 1)}{\zeta} \propto \frac{K^4}{\zeta^2} Z_2^{(H-2)} \\
                        &\quad \vdots \\
                        &\propto \frac{K^{2H}}{\zeta^H}
                    \end{align*}
            \end{itemize}

            \textbf{Conclusion}
            The $L_\infty$ Q-value error concentrates at the root of a search tree. Since this is the case, we also have that the recommender regret concentrates at the root of a search tree:
            \begin{align*}
                \prob{\mathcal{SRR}^{f, D}_t \geq \epsilon} 
                &\leq \prob{\norm{\hat{Q}_{\Omega,t}(s_0, \cdot) - Q^\star(s_0, \cdot)}_\infty \geq \frac{\epsilon}{C_3}} \\
                &\leq C_1 (Kt)^H \exp \left(- \frac{\zeta^H N_t(s_0) \epsilon^2}{C_2 K^{2H}} \right).
            \end{align*}
            For some $C_1, C_2 \in \mathbb{R}_{>0}$ that depend on $K, H, \gamma, \zeta, \sigma, \text{and } C_3$. Specifically, $\exists C_1 : Z_1^{(H)} \leq C_1 (Kt)^H$ and $\exists C_2: Z_2^{(H)} \leq C_2 \frac{K^{2H}}{\zeta^H}$. These exponential relations indicate that the tree must be well-explored prior to concentration occurring, and that $\zeta$ should be maximised to minimise the upper bound.
        \end{proof}

        \begin{remark}[Comparison with prior regularised MCTS bounds] \label{remark:mcts_bounds_kinda_suck}
            Prior regularised MCTS proofs produced exponential rates of order $O(T \exp(-T/\log^3 T))$, where depth dependence of $O(K^H)$ was hidden in constants. However, union bounding over execution histories $\sum_{i=1}^T \mathbb{P}(E_{r_i}^c)$ (\cf, \citep[Lemma 3, 4]{convex_regular_tree_search}) incorporates non-decaying early-stage estimation errors ($\prob{E_{r_1}^c} = \Omega(1)$, making history-based concentration bounds vacuous. By explicitly conditioning on valid visit allocations at each depth layer $h \in [1, H]$, our proof guarantees a rigorous concentration bound of $O(T^H \exp(-C T))$.
        \end{remark}
        \begin{proof}[An explanation of Remark~\ref{remark:mcts_bounds_kinda_suck}: the issue with prior proofs]
            \citep[Lemma 3]{convex_regular_tree_search} states that 
            $$\prob{|\hat{N}_t(s,a) - N_t^*(s,a)| > \frac{Ct}{\log t}} \leq \hat{C}|A|t \exp\left\{ -\frac{t}{(\log t)^3} \right\}.$$

            To prove this statement, the following event is defined which concentrates according to the bound that follows.
            $$E_{r_k} = \left\{ \norm{ r(s', \cdot) - \hat{r}_k(s', \cdot) }_\infty < \frac{2\sigma}{\log(2+k)} \right\}$$
            $$\prob{ \norm{ r - \hat{r}_t }_\infty \geq \frac{2\sigma}{\log(2+t)} } \leq 4|A| \exp\left( -\frac{t}{(\log(2+t))^3} \right).$$

            After some analysis, a relation is arrived at which serves as the basis for the lemma statement:
            $$\mathbb{P}\left(|\hat{N}_t(s,a) - N_t^*(s,a)| \geq \frac{Ct}{\log t}\right) \leq \sum_{k=1}^t \mathbb{P}(E_{r_k}^c).$$

            It is clear that at the $t$-th step, we have concentration with regards to $t$. However, at some arbitrary $k$-th step, $k \in [1, t]$, all we have is an irreducible constant term:
            $$\prob{E_{r_k}^c} \leq 4|A| \exp \left(- \frac{k}{(\log(2+k))^3} \right).$$

            Since there is no relation between $k$ and $t$ that allows us to convert the bound to one on $t$ (nor should there be, since at some time step $k$ we had a fixed error probability), we must conclude that the probability bound achieved is actually $\Omega(1)$, and does not prove concentration. This affects every theorem down the line, by making it impossible to prove tree convergence at any step, since the lemma is used frequently to convert child action counts to parent action counts. Note that this issue is inherited from \citep[Theorem 3]{xiao2019maximum_MENTS}, which uses the same logic to draw relations between child and parent action counts.
        \end{proof}

\section{Deferred proofs} \label{appendix:deffered_proofs}
    \subsection{Proof of Corollary~\ref{cor:uniform_beats_e3w_Q_upper_bounds}}
        \begin{corollary}[Uniform sampling minimises simple recommender regret bounds] \label{cor:uniform_beats_e3w_Q_upper_bounds}
            Among all sampling policies maintaining full support, uniform sampling ($d_t(a) = 1/K$) minimises the upper bounds in Theorem~\ref{theorem:general_upper_bound} across all notions of recommender regret. \autoref{tab:regret_bounds_comparison} lists the bounds for each notion of simple regret for an instance (\ie, not in expectation, but these results translate directly to the expectation).
            % Uniform sampling achieves the tightest Q-value regret bound, even when compared with the E3W policy in \Eqref{equ:E3W_action_selection}.
        \end{corollary}

        \begin{table}[htbp]
            \centering
            \caption{Regret bounds comparing Uniform and E3W sampling policies.}
            \begin{tabular}{lcc}
            % \hline
            \textbf{Bounded Term} & \textbf{Uniform} & \textbf{E3W} \\
            \hline
            $\norm{\hat{Q}_t - Q^\star}_\star$ & 
            $2K^{3/2} \norm{\mathbf e_1}_\star \sqrt{\frac{ \sigma^2\ln(4Kt / \delta)}{t}}$ & 
            $4K^{3/2} \norm{\mathbf e_1}_\star \sqrt{\frac{ \sigma^2 \ln(2+t) \ln(4Kt / \delta)}{t}}$ \\
            
            $\norm{\hat \pi - \pi^\star}$ & 
            $\frac{1}{\beta} 2K^{3/2} \norm{\mathbf e_1}_\star \sqrt{\frac{\sigma^2 \ln(4Kt/\delta)}{t}}$ & 
            $\frac{1}{\beta} 4K^{3/2} \norm{\mathbf e_1}_\star \sqrt{\frac{\sigma^2 \ln(2+t) \ln(4Kt/\delta)}{t}}$ \\
            
            $D_\Omega(\hat \pi, \pi^\star)$ & 
            $\frac{1}{\beta} K^3 \norm{\mathbf e_1}_\star^2 \frac{\sigma^2 \ln(4Kt/\delta)}{t}$ & 
            $\frac{2}{\beta} K^3 \norm{\mathbf e_1}_\star^2 \frac{\sigma^2 \ln(2+t) \ln(4Kt/\delta)}{t}$ \\
            
            $\left| \hat{V} - V^\star \right|$ & 
            $2\sigma\sqrt{\frac{K \ln(4Kt/\delta)}{t}}$ & 
            $4\sigma\sqrt{\frac{K \ln(2+t) \ln(4Kt/\delta)}{t}}$ \\
            \hline
            \end{tabular}
            \label{tab:regret_bounds_comparison}
        \end{table}

        \begin{proof}[Proof of Corollary~\ref{cor:uniform_beats_e3w_Q_upper_bounds}]
            It is evident in \Eqref{equ:lemma_any_policy_zeta_bound} that $\zeta$ is the principle quantity which controls the constants of the bound. If we want to minimise the bound, we should take $\zeta$ to its maximum possible value, across all actions. This implies that uniform sampling should produce the tightest Q-value error bound.
            
            We start by proving the Q-error bound under E3W-based action sampling, and then prove a similar result for uniform sampling. Both results work off of the proof in Lemma~\ref{lemma:any_policy_achieves_optimal_bounds}.

            \textbf{E3W-based sampling}
            \citet[Lemma 1]{convex_regular_tree_search} and \citet[Lemma 4]{xiao2019maximum_MENTS} show that $\hat{N}_t(a) \geq \frac{t}{2\ln(t+2)}$ for E3W-based action selection. Thus, with probability $1-\delta_1$ we have that 
            $$N_t(a) \geq \hat{N}_t(a) - \sqrt{\frac{t \ln(1/\delta_1)}{2}} \geq \frac{t}{2\ln(t+2)} - \sqrt{\frac{t \ln(1/\delta_1)}{2}}$$
            Then, with $\frac{t}{(\ln(2+t))^2} \geq 2 \ln(1/\delta_1)$, we see that $N_t(a) \geq \frac{t}{4 \ln(t+2)}$. Following the derivation in Lemma~\ref{lemma:any_policy_achieves_optimal_bounds}, we can now achieve the following bound, under the intersection event $E \cap G$, which holds for each $a \in \mathcal{A}$ for $t \geq T_0$. 
            \begin{equation} \label{equ:q_error_hoeffding_bound_e3w}
                |\hat{Q}_t(a) - Q^\star(a)| \leq \sqrt{\frac{2\sigma^2 \ln(4Kt/\delta)}{N_t(a)}} \leq \sqrt{\frac{16 K \ln(t+2) \sigma^2 \ln(4Kt/\delta)}{t}}.
            \end{equation}

            \textbf{Uniform sampling}
            We may proceed trivially by setting $N_t(a) = t/K$, which is achieved on the $K$-th time step, $t=nK$, with a deterministic uniform sampling policy that samples each action sequentially over $K$ time steps. The following bound holds, under the event $G$, for each $a \in \mathcal{A}$.
            \begin{equation} \label{equ:q_error_hoeffding_bound_uniform}
                |\hat{Q}_t(a) - Q^\star(a)| \leq \sqrt{\frac{2\sigma^2 \ln(4Kt/\delta)}{N_t(a)}} = \sqrt{\frac{2\sigma^2 K\ln(4Kt/\delta)}{t}}.
            \end{equation}

            % \textbf{Results in Table~\ref{tab:regret_bounds_comparison}}
            % The results in Table~\ref{tab:regret_bounds_comparison} are generated by combining the Q-error upper bounds for uniform sampling (\Eqref{equ:q_error_hoeffding_bound_uniform}) and E3W (\Eqref{equ:q_error_hoeffding_bound_e3w}) with Lemma~\ref{lemma:q_val_convergence_implies_value_policy} and Lemma~\ref{lemma:any_policy_achieves_optimal_bounds}, similarly to how the theorem was proven.
        \end{proof}
        
    \subsection{Proof of Corollary~\ref{cor:MCTS_uniform_regret}} \label{appendix:MCTS_COR_UNIF_PROOF}
        This is pushed to the appendix because it does not directly give the expected regret bound, but is close.
        
        \begin{corollary}[Recommender MCTS optimal regret bound] \label{cor:MCTS_uniform_regret}
                Under the hypotheses of Theorem~\ref{thrm:MCTS_SRR_high_prob_bound}, the uniform sampling policy $d_t(a|s) = 1/K$ maximizes $\zeta^* = 1/K$. This choice minimizes both $C_2 / \zeta^H$ and the search horizon $T$. In particular, for error probability $\delta \in (0, \min(C_1 (TK)^H,1)]$ and search horizon $T \geq \frac{C_2 K^{2+3H}}{C_3^2} \ln \left( \frac{C_1 (TK)^H}{\delta} \right)$, uniform sampling guarantees that with probability at least $1-\delta$:
                \begin{equation} \label{equ:MCTS_optimal_regret_bound}
                    \mathcal{SRR}^{f, D}_T \leq \sqrt{\frac{C_2 K^{3H}}{T} \ln \left( \frac{C_1 (TK)^H}{\delta} \right)}.
                \end{equation}
        \end{corollary}
        \begin{proof}[Proof of Corollary~\ref{cor:MCTS_uniform_regret}]
            Summing the allocation constraint $d_t(a|s) \geq \zeta$ over all $K = |\mathcal{A}|$ actions yields $1 \geq K \zeta$, establishing that $\zeta \leq 1/K$. The maximum threshold $\zeta^* = 1/K$ is uniquely attained by the uniform sampling policy $d_t(a|s) = 1/K$. Choosing $\zeta^* = 1/K$ minimises $C_2/\zeta^H$, which simultaneously minimises the regret upper bound and the required search horizon.
        
            From Theorem~\ref{thrm:MCTS_SRR_high_prob_bound}, we have the failure probability 
            \begin{equation*}
                \begin{aligned}
                    \delta &= C_1 (TK)^H \exp \left( - \frac{ \zeta^H T \epsilon^2}{C_2 K^{2H}} \right) \\
                     &= C_1 (TK)^H \exp \left( - \frac{ T \epsilon^2}{C_2 K^{3H}} \right). \quad & (\zeta^* = 1/K)
                \end{aligned}
            \end{equation*}

% Since $0<\delta <1$, then $T\geq \ln\left(C (KT)^H\right)\frac{C_2 K^{3H}}{\epsilon^2}$. 
Solving for $\epsilon$ gives
            \begin{equation*}
                \epsilon = \sqrt{\frac{C_2 K^{3H}}{ T} \ln \left( \frac{C_1 (TK)^H}{\delta} \right)}.
            \end{equation*}

            % which
            % 
            By Theorem~\ref{thrm:MCTS_SRR_high_prob_bound}, this concentration bound holds with probability at least $1 - \delta$. Notice that for the bound to be non-vacuous, then $T$ must be relatively large. Provided $\epsilon \leq C_3 \zeta^* = C_3 / K$, we can substitute our expression for $\epsilon$ into this threshold and obtain
            \begin{align*}
                \sqrt{\frac{C_2 K^{3H}}{T} \ln \left( \frac{C_1 (TK)^H}{\delta} \right)}
                &\leq \frac{C_3}{K} \\
                \frac{C_2 K^{2+3H}}{C_3^2} \ln \left( \frac{C_1 (TK)^H}{\delta} \right) &\leq T. \\
                % \frac{C_2 K^2}{C_3^2} \left( H\ln(T) + \ln \left( \frac{C_1 K^H}{\delta} \right) \right) &\leq T \\
            \end{align*}
            % Let $a = \frac{C_2 K^2 H}{C_3^2}$ and $b =  \frac{C_2 K^2}{C_3^2}\ln \left( \frac{C_1 K^H}{\delta} \right)$. We now have $a\ln(T) + b \leq T$, from which we can bound $T$ via \citep[Lemma A.2]{shalev2014understanding}, under the conditions that 
            % $$a = \frac{C_2 K^2 H}{C_3^2} \geq \frac{K^{H}H}{\norm{\mathbf e'}_\star} \geq 1,$$
            % which holds via the constants of the tree search, and 
            % $$b =  \frac{C_2 K^2}{C_3^2}\ln \left( \frac{C_1 K^H}{\delta} \right) > 0,$$
            % which holds because $ \frac{C_1 K^H}{\delta} > 1$ since each term in the numerator is greater than one, and $0 < \delta < 1$ via the restriction $\delta \in (0, \min(C_1 (TK)^H, 1))$ which ensures both that $\delta$ is a valid probability ($\delta < 1$) and that $\ln(C_1 (TK)^H/\delta) > 0$. With these conditions on $a$ and $b$, we have via \citep[Lemma A.2]{shalev2014understanding} that the following is sufficient to satisfy the conditions above.
            % $$T \geq 4a \log(2a) + 2b = \frac{2C_2 K^2}{C_3^2} \left( 2H \log \left( \frac{2C_2 K^2 H}{C_3^2}  \right) + \ln \left( \frac{C_1 K^H}{\delta} \right) \right)$$
        \end{proof}

    \subsection{Proof of Corollary~\ref{cor:MCTS_expected_regret}} \label{appendix:proof_of_mcts_Expected_regret}
        Expected regret bound for MCTS.
        \begin{proof}[Proof of Corollary~\ref{cor:MCTS_expected_regret}]
        Because $\mathcal{SRR}^{f, D}_T$ is a non-negative random variable strictly bounded by 
        $$\mathcal{SRR}^{f, D}_T \leq C_3 \norm{\hat Q_t - Q^\star}_\infty \leq 2C_3R_\text{max},$$
        according to the assumptions of the Theorem, its expectation can be computed by integrating its cumulative distribution function:
        \begin{equation*}
            \mathbb{E}[\mathcal{SRR}^{f, D}_T] = \int_0^{2C_3R_\text{max}} \prob{\mathcal{SRR}^{f, D}_T \geq \epsilon} d\epsilon.
        \end{equation*}
        
        Following Corollary~\ref{cor:MCTS_uniform_regret}, which bounds the instance-dependent $\mathcal{SRR}$, we choose $d_t(a|s) = 1/K$ to maximise $1/\zeta$, which consequently minimises the required search horizon for the bounds to hold: 
        $$T \geq \frac{C_2 K^{2+3H}}{C_3^2} \ln \left( \frac{C_1 (TK)^H}{\delta} \right),$$
        which guarantees that 
        $$u = \sqrt{\frac{C_2 K^{3H}}{T} \ln(C_1 (TK)^H)} \leq C_3 \zeta^* = \frac{C_3}{K}.$$
        
        We can then set $\delta = 1$ in the high-probability horizon bound from Corollary~\ref{cor:MCTS_uniform_regret} to arrive at our choice for $u$: this means that when $\epsilon < u$, the error probability of this event is $\leq \delta = 1$. Thus, we partition the integration into three intervals: $[0, u]$, $[u, C_3/K]$, and $[C_3/K, 2C_3R_\text{max}]$.
        
        \textbf{Interval 1: $[0, u]$} \\
        We trivially bound the probability by $1$:
        \begin{equation*}
            \int_0^u \prob{\mathcal{SRR}^{f, D}_T \geq \epsilon} d\epsilon \leq \int_0^u 1 \,d\epsilon = u = \sqrt{\frac{C_2 K^{3H}}{T} \ln(C_1 (TK)^H)}.
        \end{equation*}
    
        \textbf{Interval 2: $[u, C_3/K]$} \\
        In this region, $\epsilon \leq C_3/K$, so Theorem~\ref{thrm:MCTS_SRR_high_prob_bound} applies. Extending the upper limit to infinity yields an upper bound. Applying the standard Gaussian tail inequality $\int_u^\infty e^{-ax^2} dx \leq \frac{1}{2au}e^{-au^2}$ with $a = T/C_2$:
        \begin{align*}
            \int_u^{\frac{C_3}{K}} \prob{\mathcal{SRR}^{f, D}_T \geq \epsilon} d\epsilon 
            &\leq \int_u^{\infty} C_1 (TK)^H \exp\left(-\frac{T\epsilon^2}{C_2 K^{3H}}\right) d\epsilon \\
            &\leq C_1 (TK)^H \left( \frac{C_2 K^{3H}}{2Tu} \right) \exp\left(-\frac{Tu^2}{C_2 K^{3H}}\right).
        \end{align*}
        By our definition of $u$, the term $\exp(-Tu^2/C_2 K^{3H}) = \exp(-\ln(C_1 (TK)^H)) = \frac{1}{C_1 (TK)^H}$. Substituting this back cancels the leading coefficient:
        \begin{equation*}
            \int_u^{\frac{C_3}{K}} \prob{\mathcal{SRR}^{f, D}_T \geq \epsilon} d\epsilon \leq \frac{C_2 K^{3H}}{2Tu} = \frac{1}{2} \sqrt{\frac{C_2 K^{3H}}{T \ln \left( C_1 (TK)^H \right)}}.
        \end{equation*}
    
        \textbf{Interval 3: $[C_3/K, 2C_3R_\text{max}]$} \\
        For $\epsilon > C_3/K$, notice that the probability is monotonically decreasing, so it is bounded by its value at the threshold $\epsilon = C_3/K$, allowing us to leverage Theorem~\ref{thrm:MCTS_SRR_high_prob_bound}. The length of this interval is strictly less than $2C_3R_\text{max}$, providing
        \begin{align*}
            \int_{\frac{C_3}{K}}^{2C_3R_\text{max}} \prob{\mathcal{SRR}^{f, D}_T \geq \epsilon} d\epsilon 
            &\leq \int_{0}^{2C_3R_\text{max}} \prob{\mathcal{SRR}^{f, D}_T \geq \frac{C_3}{K}} d\epsilon \\
            & \leq 2C_3R_\text{max} \cdot C_1(TK)^H \exp\left(-\frac{T C_3^2}{C_2 K^{2+3H}}\right).
        \end{align*}
    
        Summing the bounds from the three intervals yields the final result in \autoref{equ:MCTS_expected_regret_full_bound}. As $T \to \infty$, the first term scales as $\tilde{O}(\sqrt{\frac{K^{3H}}{T}})$ and dominates, while the second term scales as $\tilde{O}((T \ln T)^{-1/2})$, and the final term decays exponentially.
        \begin{equation} \label{equ:MCTS_expected_regret_full_bound}
            \begin{split}
                \mathbb{E}[\mathcal{SRR}^{f, D}_T] \leq &\sqrt{\frac{C_2 K^{3H}}{T} \ln \left( C_1 (TK)^H \right)}\\
                &+ \frac{1}{2} \sqrt{\frac{C_2 K^{3H}}{T \ln \left( C_1 (TK)^H \right)}} \\
                &+ 2C_3R_\text{max} C_1(TK)^H \exp\left(-\frac{T C_3^2}{C_2 K^{2+3H}}\right).
            \end{split}
        \end{equation}

        It is also of interest to bound $\mathcal{SRR}^{\pi,B}_T$ for the Bregman divergence $B_\Omega$. Noting the result in Theorem~\ref{thrm:MCTS_SRR_high_prob_bound}, we know that the Bregman divergence-based regret is bound by $C_3^2 \norm{\hat Q_t - Q^\star}^2_\infty$, which means that we can set the error lower bound in the probability statement to $\epsilon^2$. Doing so yields a slightly modified regret bound and proof, but one of the expected rate: $\frac{1}{T}$.

        \begin{align*}
            \mathbb{E}[\mathcal{SRR}^{\pi, B_\Omega}_T] &= \int_0^{4C^2_3R^2_{\max}} \prob{\mathcal{SRR}^{\pi, B_\Omega}_T \geq \epsilon} d\epsilon \\
            &\leq \int_0^{4C^2_3R^2_{\max}} \prob{C_3^2 \norm{\hat Q_t - Q^\star}^2_\infty \geq \epsilon} d\epsilon \\
            &= \int_0^{4C_3^2R^2_{\max}} \prob{C_3 \norm{\hat Q_t - Q^\star}_\infty \geq \sqrt\epsilon} d\epsilon \\
        \end{align*}
        We can now choose
        $$u = {\frac{C_2 K^{3H}}{T} \ln(C_1 (TK)^H)} \leq \frac{C^2_3}{K^2}.$$
        Our high-probability bound from Corollary~\ref{cor:MCTS_uniform_regret} still holds under $\sqrt{\epsilon}$, and we can still set $\delta=1$. This means that the only difference in the above proof is the limits of integration, and what they resolve to. The following completes the proof.
        
        \textbf{Interval 1: $[0, u]$} \\
        We trivially bound the probability by $1$:
        \begin{equation*}
            \int_0^u \prob{\mathcal{SRR}^{\pi, B_\Omega}_T \geq \epsilon} d\epsilon \leq \int_0^u 1 \,d\epsilon = u = \frac{C_2 K^{3H}}{T} \ln(C_1 (TK)^H).
        \end{equation*}
    
        \textbf{Interval 2: $[u, C^2_3/K^2]$} \\
        In this region, $\sqrt\epsilon \leq C_3/K$, so Theorem~\ref{thrm:MCTS_SRR_high_prob_bound} applies. Extending the upper limit to infinity yields an upper bound.
        \begin{align*}
            \int_u^{\frac{C^2_3}{K^2}} \prob{\mathcal{SRR}^{\pi, B_\Omega}_T \geq \epsilon} d\epsilon 
            &\leq \int_u^{\infty} C_1 (TK)^H \exp\left(-\frac{T\epsilon}{C_2 K^{3H}}\right) d\epsilon \\
            &\leq -C_1 (TK)^H \left( \frac{C_2 K^{3H}}{T} \right) \exp\left(-\frac{T\epsilon}{C_2 K^{3H}}\right) \Bigg|^\infty_u \\
            &= C_1 (TK)^H \left( \frac{C_2 K^{3H}}{T} \right) \exp\left(-\frac{Tu}{C_2 K^{3H}}\right) \\
            &= \frac{C_2 K^{3H}}{T}
        \end{align*}
    
        \textbf{Interval 3: $[C^2_3/K^2, 4C_3^2R^2_{\max}]$} \\
        For $\sqrt\epsilon > C_3/K$, the probability is monotonically decreasing, so it is bounded by its value at the threshold $\sqrt\epsilon = C_3/K$. The length of this interval is strictly less than $4C_3^2R^2_{\max}$, providing:
        \begin{align*}
            \int_{\frac{C^2_3}{K^2}}^{ 4C_3^2R^2_{\max}} \prob{\mathcal{SRR}^{\pi, B_\Omega}_T \geq \epsilon} d\epsilon 
            &\leq \int_{0}^{ 4C_3^2R^2_{\max}} \prob{\mathcal{SRR}^{\pi, B_\Omega}_T \geq \frac{C^2_3}{K^2}} d\epsilon \\
            & \leq  4C_3^2R^2_{\max} \cdot C_1(TK)^H \exp\left(-\frac{T C_3^2}{C_2 K^{2+3H}}\right).
        \end{align*}
        Summing these terms, we get a bound of $\Tilde{O} \left( \frac{K^H}{T} \right)$:
        \begin{equation} \label{equ:MCTS_expected_regret_full_bound_bregman}
            \begin{split}
                \mathbb{E}[\mathcal{SRR}^{\pi, B_\Omega}_T] \leq &\frac{C_2 K^{3H}}{T} \ln(C_1 (TK)^H)\\
                &+ \frac{C_2 K^{3H}}{T} \\
                &+ 4C_3^2R^2_{\max} \cdot C_1(TK)^H \exp\left(-\frac{T C_3^2}{C_2 K^{2+3H}}\right).
            \end{split}
        \end{equation}
    \end{proof}

\section{Additional details and plots for synthetic tree experiments}
\label{appendix:synthetic_tree_plots}
    The purpose of this experiment is to highlight the differences between E3W and uniform selection.
    To adequately compare E3W action selection to uniform action selection, we set the exploration parameter $\epsilon = 0.1$. In our testing, there was little difference in performance for different selections of $\epsilon$, excluding of course when $\epsilon$ is large, which means E3W performs similarly to uniform sampling, for some time. 
    We also change $\lambda_{s_t} = \frac{\epsilon}{\log(N_t(s_t) + 2)}$ since this forces $\lambda \in [0,1]$ (when $\lambda_{s_t} = \frac{\epsilon K}{\log(N_t(s_t) + 2)} > 1$ like in \citet{convex_regular_tree_search} we no longer have a convex combination in $d_t^\text{E3W}$).
    % In fact, due to our results, we conjecture that the $K$ factor in \citet{convex_regular_tree_search} along with a $\lambda \in [0,1]$ cutoff would produce uniform sampling for some time, causing E3W to perform better in the settings we identified.
    
    We present complete plots for the synthetic tree experiments. These plots include all the depth and action space size settings that we tested. Note that any plots labelled ``theoretical'' use the primal or dual norm that is associated with the regulariser. For instance, entropy is strongly convex with respect to the L1 norm, and Tsallis entropy with respect to the L2 norm. Thus, the dual norm for entropy-regularised Q-values is the L-infinity norm, while for Tsallis it is L2.

    Note how in the ``theoretical'' plots, Tsallis entropy and UCT often have significantly different error values to the other regularised RL methods. In the Bregman divergence plots, this is due to the KL divergence being unbounded, while the L2 policy error is bounded. In the Q-value error plots, this is because the L2 error is significantly higher than the L-infinity error.

    In the main body, we elected to present plots that were similarly scaled through the use of the same error function. This meant that we sometimes couldn't accurately represent the regret curve of the bandit / tree search, but a proxy of that curve through another error function. This choice was made to avoid unnecessary confusion regarding which method was best, since the use of different norms would skew the positioning of the curves.

    \textbf{``Why does UCT drop suddenly?" and other related questions answered}
        When we retrieve Q-value estimates from the tree, we always zero out the unvisited action's Q-values, because otherwise the default value is negative infinity. This means that we suffer constant regret on unvisited arms for those Q-value predictions. While papers like \citet{danihelka2022policy_gumbel_alphazero} would use a sort of mixing operator to replace unvisited Q-values with the value of the parent, we do not use such a device. This may be useful for future investigation, but would likely not yield terribly different plots. 
        Note that ``constant regret on unvisited arms" means that UCT will especially have difficulty minimising Q-value error since it doesn't visit all arms until some high number of simulations, in our results around 4500 or 5000 simulations. Even the greedier sampling policies like E3W suffer this fate, since they tend to sample high performing arms more often (as is expected by the greedy policy $\hat \pi_t$).

    The following lists the plots in this appendix.
    \begin{itemize}
        \item \textbf{Figure~\ref{fig:appendix_Q_vals_MSE}} illustrates the Q-value mean squared error plots in synthetic trees.
        \item \textbf{Figure~\ref{fig:appendix_value_error_mse}} shows the value mean squared error plots in synthetic trees.
        \item \textbf{Figure~\ref{fig:appendix_policy_error_l2}} displays the L2 norm-based policy error plots in synthetic trees.
        \item \textbf{Figure~\ref{fig:appendix_policy_error_theoretical}} depicts the theoretical norm-based policy error plots in synthetic trees.
        \item \textbf{Figure~\ref{fig:appendix_policy_bregman_kl}} demonstrates the KL divergence-based value error plots in synthetic trees.
        \item \textbf{Figure~\ref{fig:appendix_policy_bregman_theoretical}} highlights the theoretical Bregman divergence-based value error plots in synthetic trees.
        \item \textbf{Figure~\ref{fig:4x5_q_error_synthetic_tree}} shows the Q-value error for each seed (random environment over which 100 random tree searches were run)
        \item \textbf{Figure~\ref{fig:4x5_value_error_synthetic_tree}} shows the value error for each seed (random environment over which 100 random tree searches were run)
        \item \textbf{Figure~\ref{fig:4x5_norm_policy_error_synthetic_tree}} shows the norm-based policy error for each seed (random environment over which 100 random tree searches were run)
        \item \textbf{Figure~\ref{fig:4x5_bregman_policy_error_synthetic_tree}} shows the Bregman-based policy error for each seed (random environment over which 100 random tree searches were run)
\end{itemize}

    \begin{figure}[ht]
        \centering
        \includegraphics[width=\linewidth]{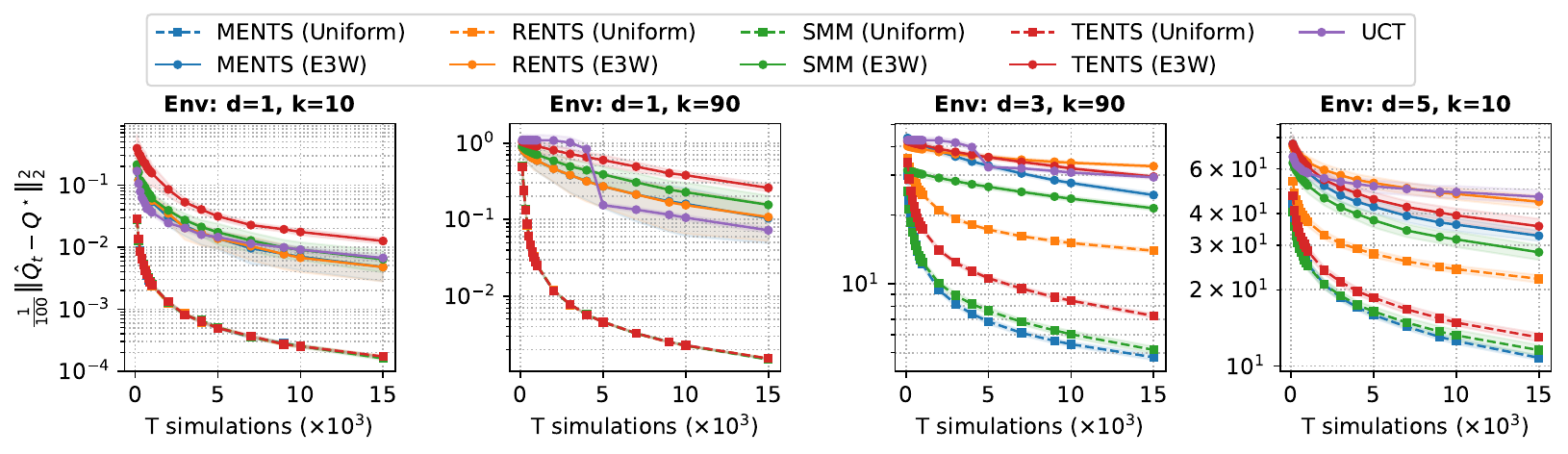}
        \caption{Q-value mean squared error plots in synthetic trees.}
        \label{fig:appendix_Q_vals_MSE}
    \end{figure}

    \begin{figure}[ht]
        \centering
        \includegraphics[width=\linewidth]{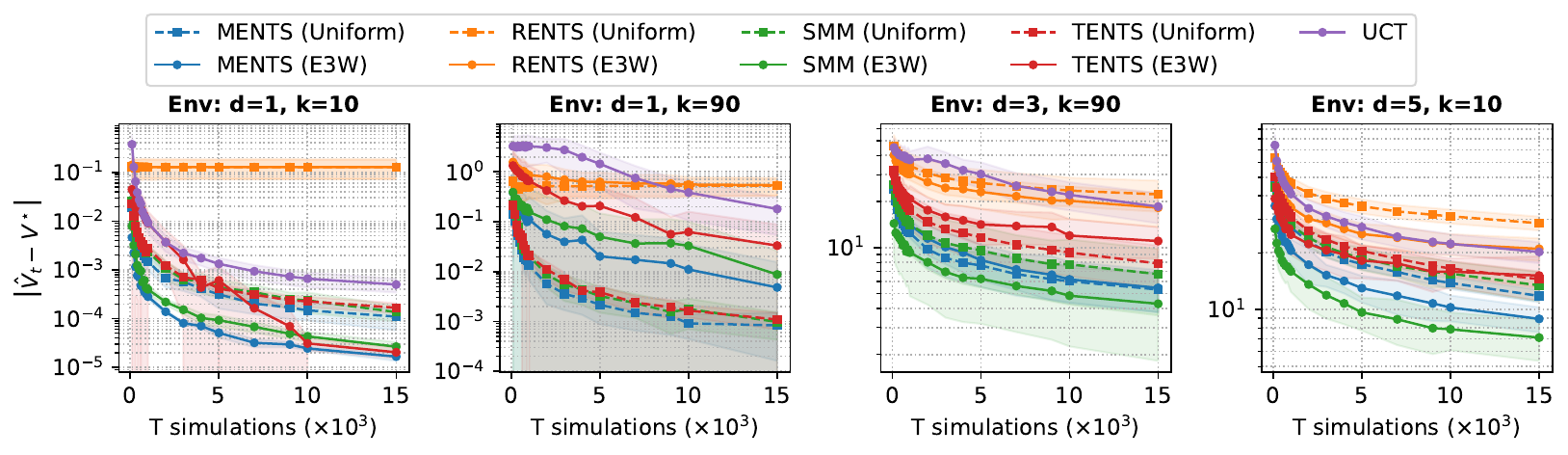}
        \caption{Value mean squared error plots in synthetic trees.}
        \label{fig:appendix_value_error_mse}
    \end{figure}

    \begin{figure}[ht]
        \centering
        \includegraphics[width=\linewidth]{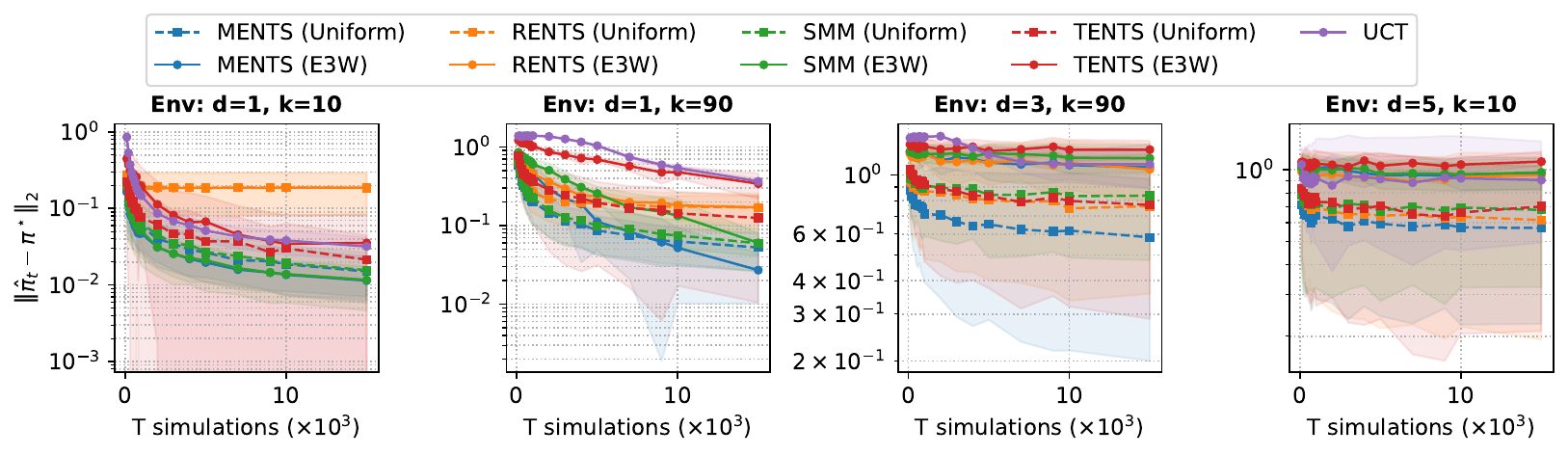}
        \caption{L2 norm-based policy error plots in synthetic trees.}
        \label{fig:appendix_policy_error_l2}
    \end{figure}

    \begin{figure}[ht]
        \centering
        \includegraphics[width=\linewidth]{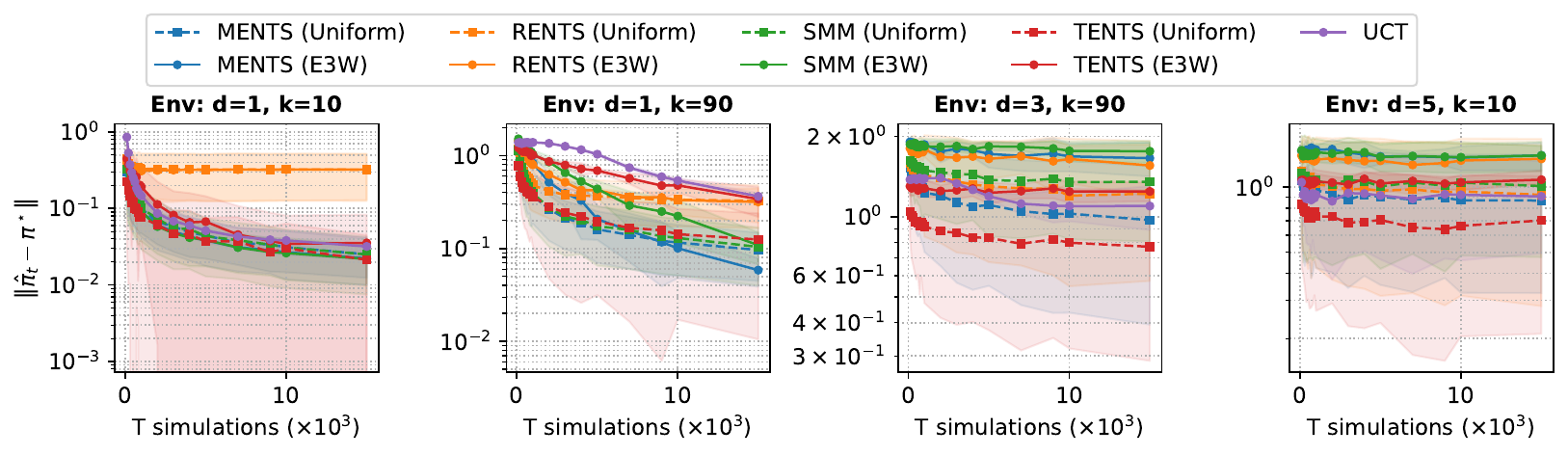}
        \caption{Theoretical norm-based policy error plots in synthetic trees.}
        \label{fig:appendix_policy_error_theoretical}
    \end{figure}

    \begin{figure}[ht]
        \centering
        \includegraphics[width=\linewidth]{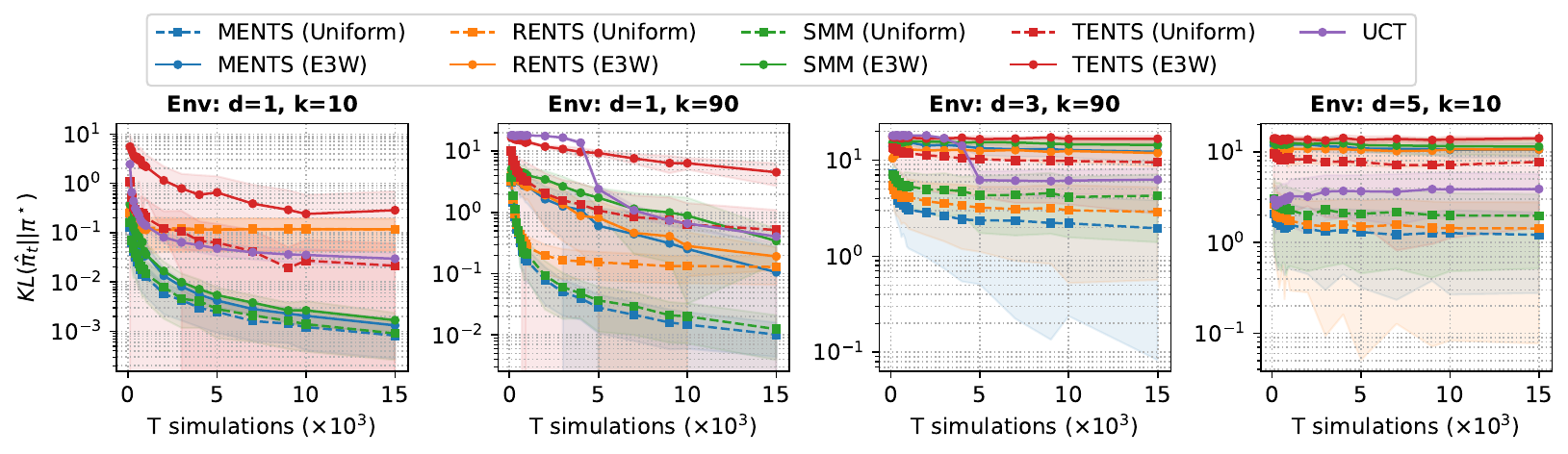}
        \caption{KL divergence-based error plots in synthetic trees.}
        \label{fig:appendix_policy_bregman_kl}
    \end{figure}

    \begin{figure}[ht]
        \centering
        \includegraphics[width=\linewidth]{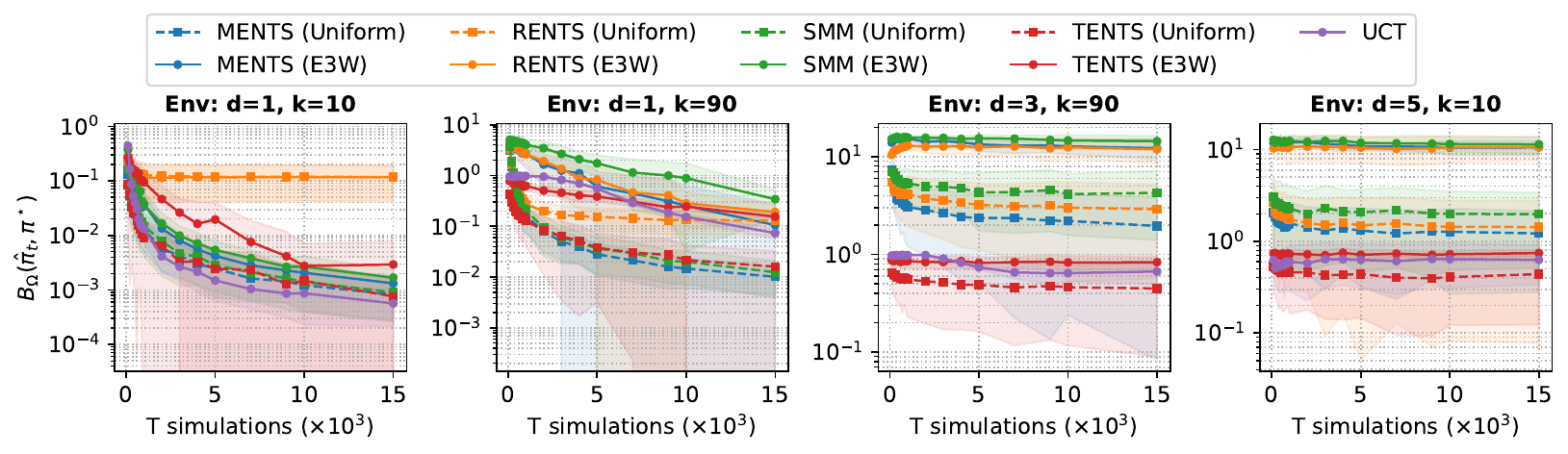}
        \caption{Theoretical Bregman divergence-based value error plots in synthetic trees.}
        \label{fig:appendix_policy_bregman_theoretical}
    \end{figure}

    \begin{figure}
        \centering
        \includegraphics[width=\linewidth]{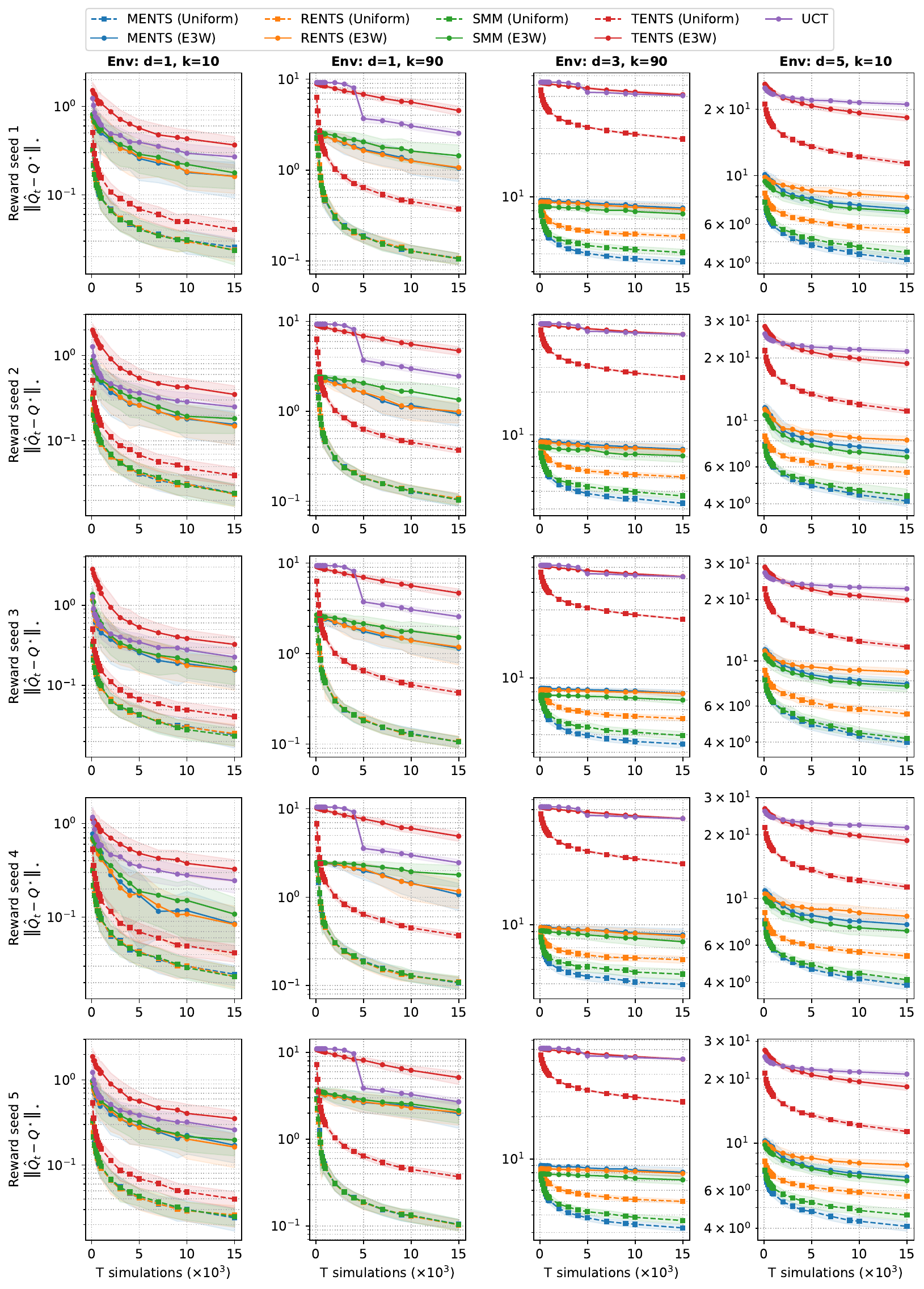}
        \caption{Q-error graphs per seed (5 seeds one per row)}
        \label{fig:4x5_q_error_synthetic_tree}
    \end{figure}

    \begin{figure}
        \centering
        \includegraphics[width=\linewidth]{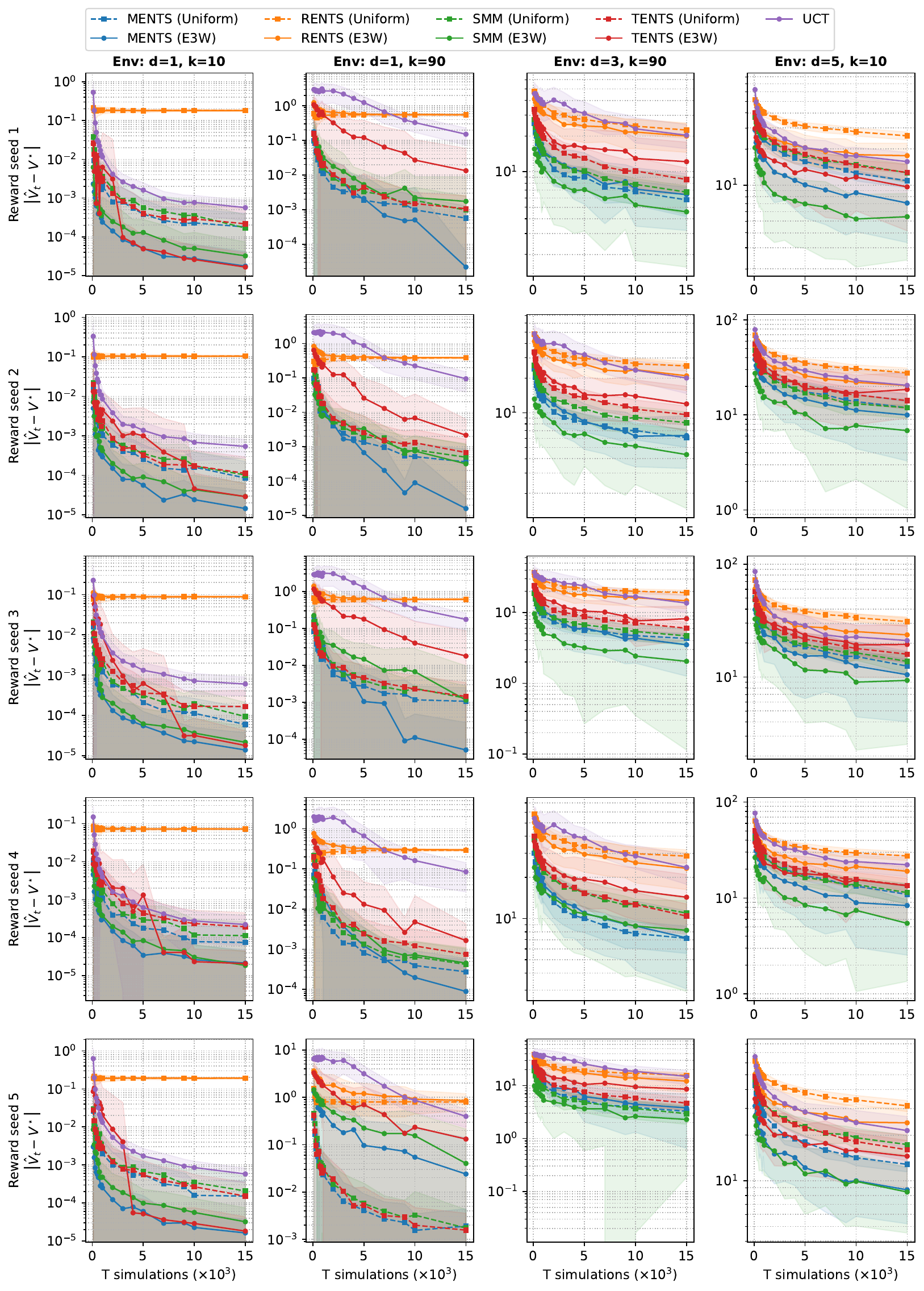}
        \caption{Value error graphs per seed (5 seeds one per row)}
        \label{fig:4x5_value_error_synthetic_tree}
    \end{figure}

    \begin{figure}
        \centering
        \includegraphics[width=\linewidth]{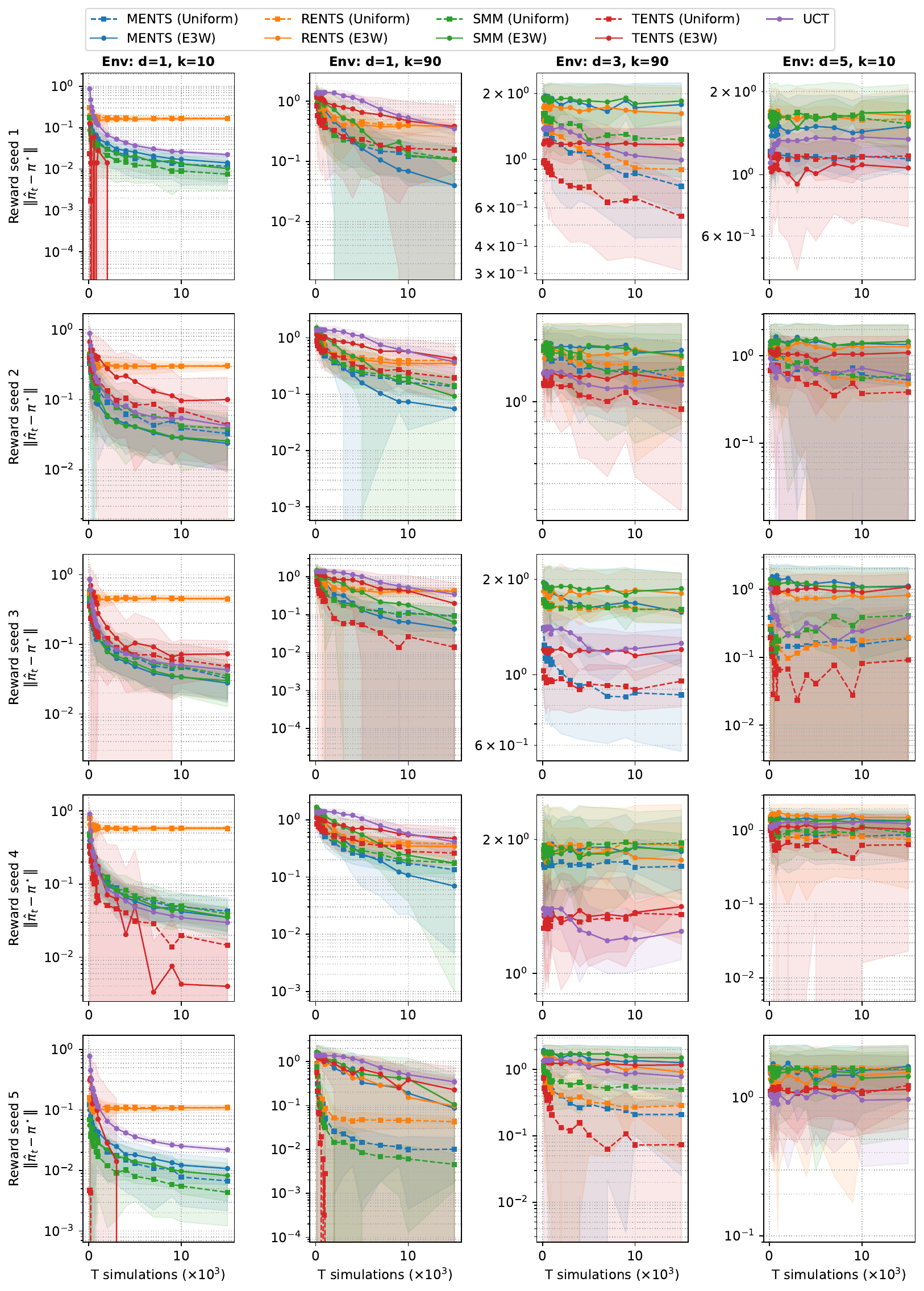}
        \caption{Norm-based policy error graphs per seed (5 seeds one per row)}
        \label{fig:4x5_norm_policy_error_synthetic_tree}
    \end{figure}

    \begin{figure}
        \centering
        \includegraphics[width=\linewidth]{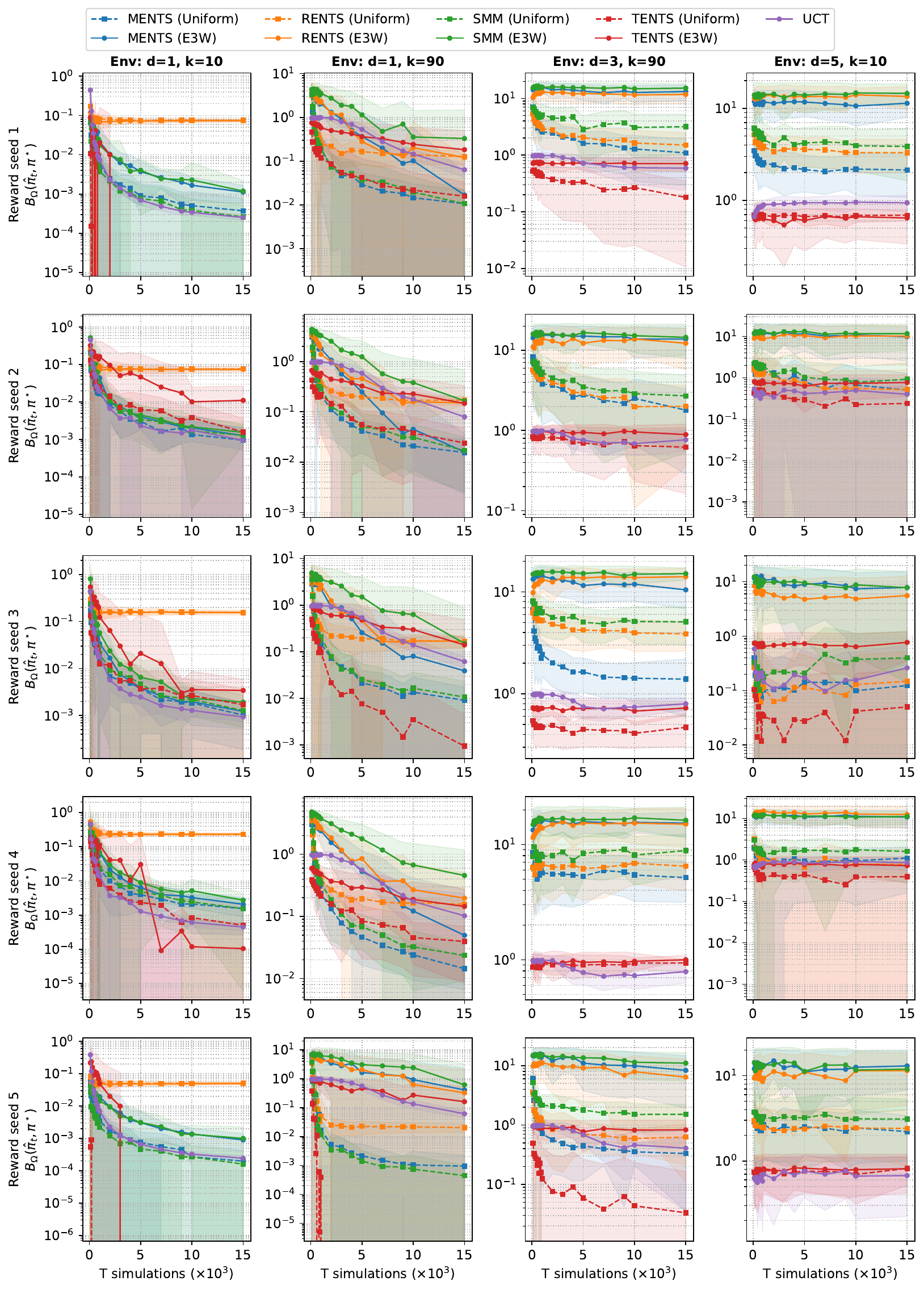}
        \caption{Bregman-based policy error graphs per seed (5 seeds one per row)}
        \label{fig:4x5_bregman_policy_error_synthetic_tree}
    \end{figure}

\clearpage
\section{Details of the OCTAX experiments} \label{appendix:additional_octax_details}
    We follow a PQN-style \citep{gallici2025simplifying_pqn} implementation: 1024 vectorized environments, 16-step rollouts, 3 gradient updates over 16 minibatches per rollout, for a total of $10^7$ environment steps. The network is a small convolutional encoder (two $3\times3$, stride-2 convolutions with 16 and 32 channels, each followed by LayerNorm and ReLU, into a 64-dimensional dense encoding) with a linear Q-head, trained with an $\epsilon$-greedy behaviour policy over a 4-frame stack. We generate Q-value estimates from the Q-head by running $\hat V(s) = \Omega^*(Q_\theta(s, \cdot))$.

    We swept the inverse temperature $\omega \in {5, 10, 20, 50, 100}$ on 15 of the 19 games for the regularised tree searches and UCT. We discovered that performance didn't change considerably over these wider $\omega$ settings: an algorithm that outperforms the baseline with $\omega=5$ still outperforms at $\omega=20$, for instance. Since we aren't looking to demonstrate state-of-the-art performance in these tasks, we found it counterproductive, and computationally wasteful, to run full sweeps over five $\omega$ values.
    Tree search depth is capped at 10, and $\gamma=1$ to match full-episode returns, with the number steps per episode fixed based on that game's typical episode length. 
    For UCT specifically, we also swept over the Dirichlet hyperparameters to see how performance might improve with more or less exploration, but again did not observe significant changes in performance. In particular, all settings of $\omega$ and the exploration hyperparameters that we swept had returns within one standard deviation of one another, indicating that there was no significant differences in performance. 
    For the OCTAX experiments, we followed the hyperparameters of \citet{xiao2019maximum_MENTS} and \citet{convex_regular_tree_search} for their arcade game experiments: $\epsilon = 0.1$, with $\lambda_{s_t} = \frac{\epsilon K}{\log(N_t(s_t)+2)}$.

    The tasks that we filtered out were done so to choose which ones to run for the 512 simulation budget. These experiments took many days (\cf, \autoref{fig:appendix_octax_long_sims_execution_time}), so we had to be sure which ones to run. 

\section{Additional plots for OCTAX experiments}
\label{appendix:octax_plots}

    We first present the plots for environments where we ran 32 to 512 MCTS simulations per environment step. Then, we present the plots for all the environments together, where we show the results for 32 to 128 MCTS simulations per step.
    \begin{itemize}
        \item \textbf{Figure~\ref{fig:appendix_octax_long_sims_execution_time}} illustrates the execution times for the large-sim task set. 
        \item \textbf{Figure~\ref{fig:appendix_octax_long_sims_episode_lengths}} illustrates the episode lengths for the large-sim task set.
        \item \textbf{Figure~\ref{fig:appendix_octax_episode_returns}} illustrates the episodic return for all tasks, with 32 to 128 MCTS simulations per step.
        \item \textbf{Figure~\ref{fig:appendix_octax_episode_lengths}} illustrates the episode lengths for all tasks, with 32 to 128 MCTS simulations per step.
        \item \textbf{Figure~\ref{fig:appendix_octax_execution_times}} illustrates the execution times for all tasks, with 32 to 128 MCTS simulations per step.
    \end{itemize}

    \begin{figure}[ht]
        \centering
        \includegraphics[width=\linewidth]{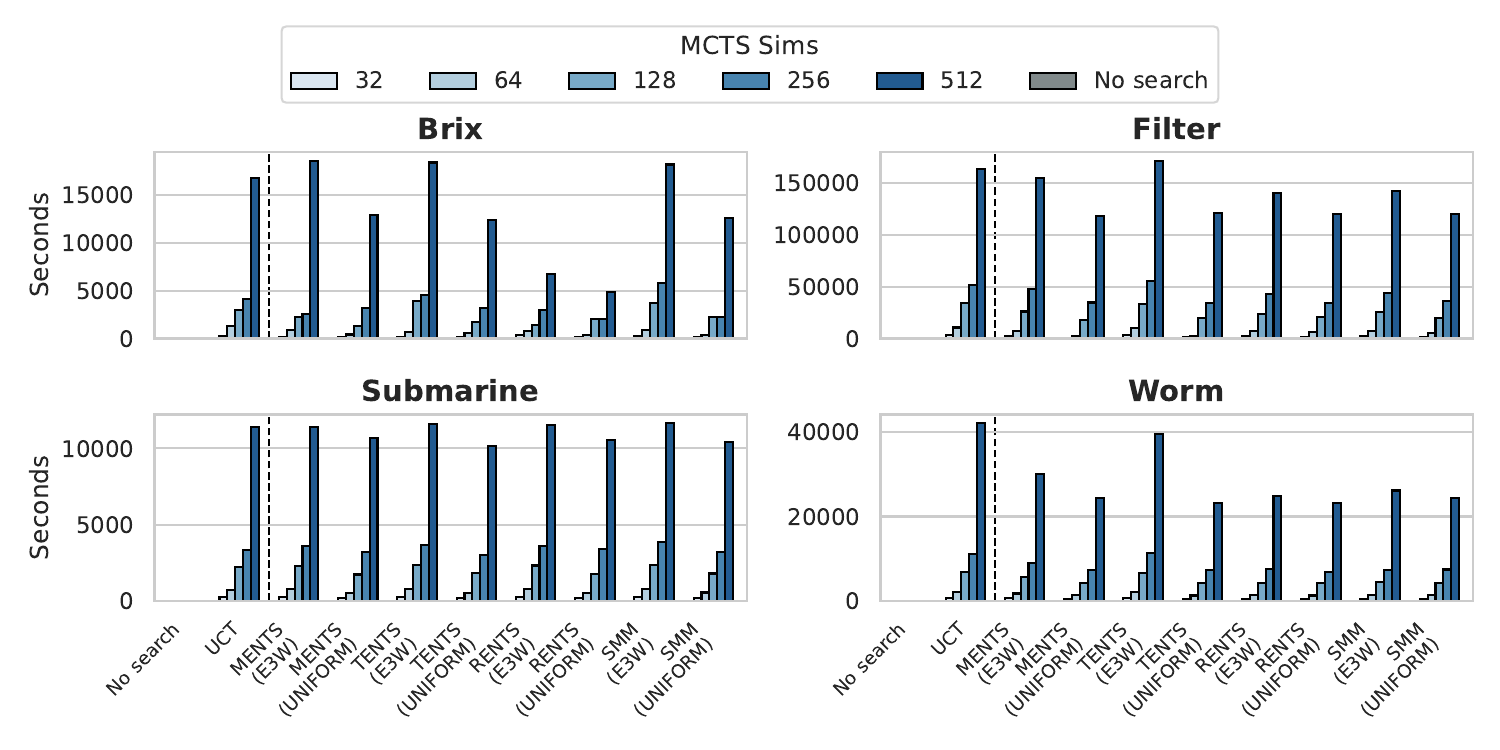}
        \caption{Execution times increase quickly with sims per step.}
        \label{fig:appendix_octax_long_sims_execution_time}
    \end{figure}

    \begin{figure}[ht]
        \centering
        \includegraphics[width=\linewidth]{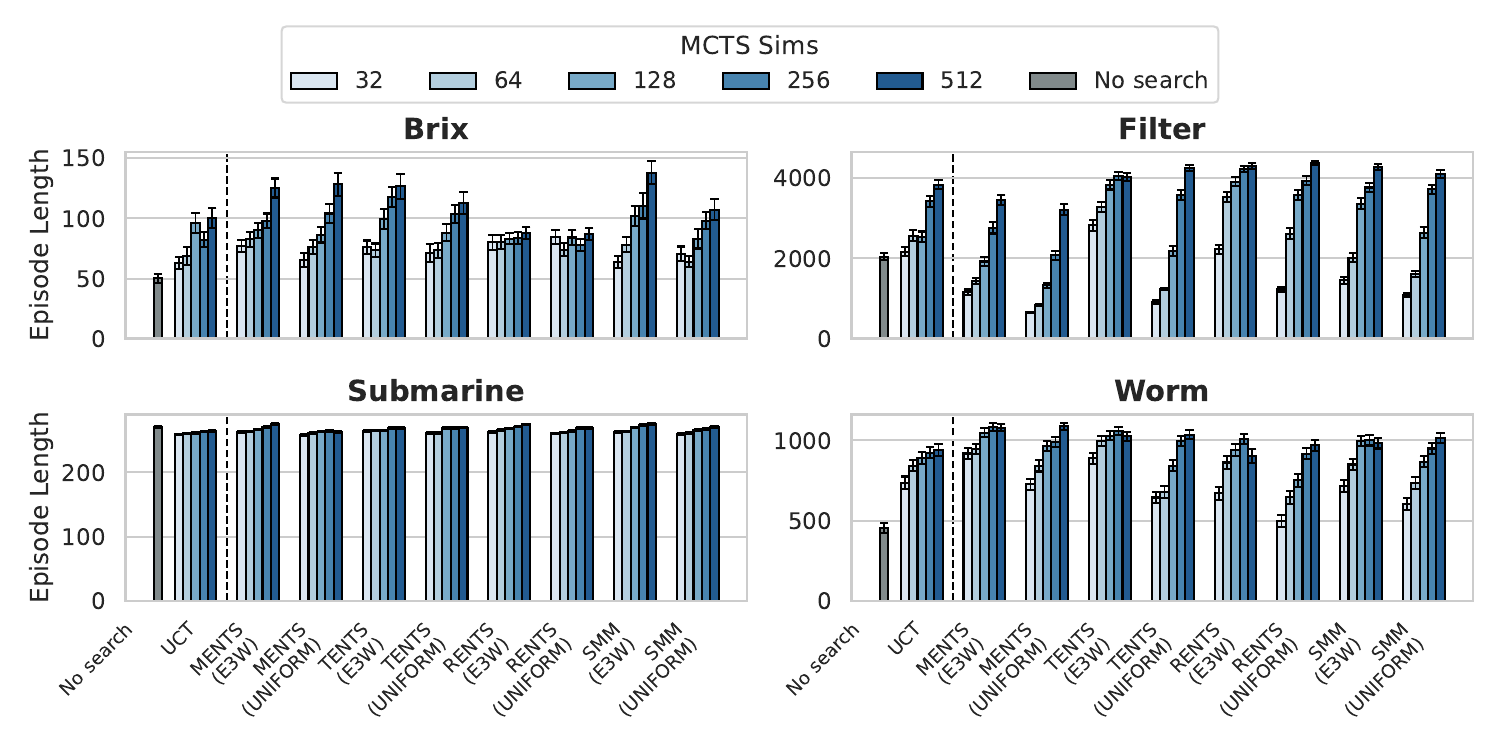}
        \caption{Episode lengths stay relatively similar with increased sims per step.}
        \label{fig:appendix_octax_long_sims_episode_lengths}
    \end{figure}

    \begin{figure}[ht]
        \centering
        \includegraphics[width=\linewidth]{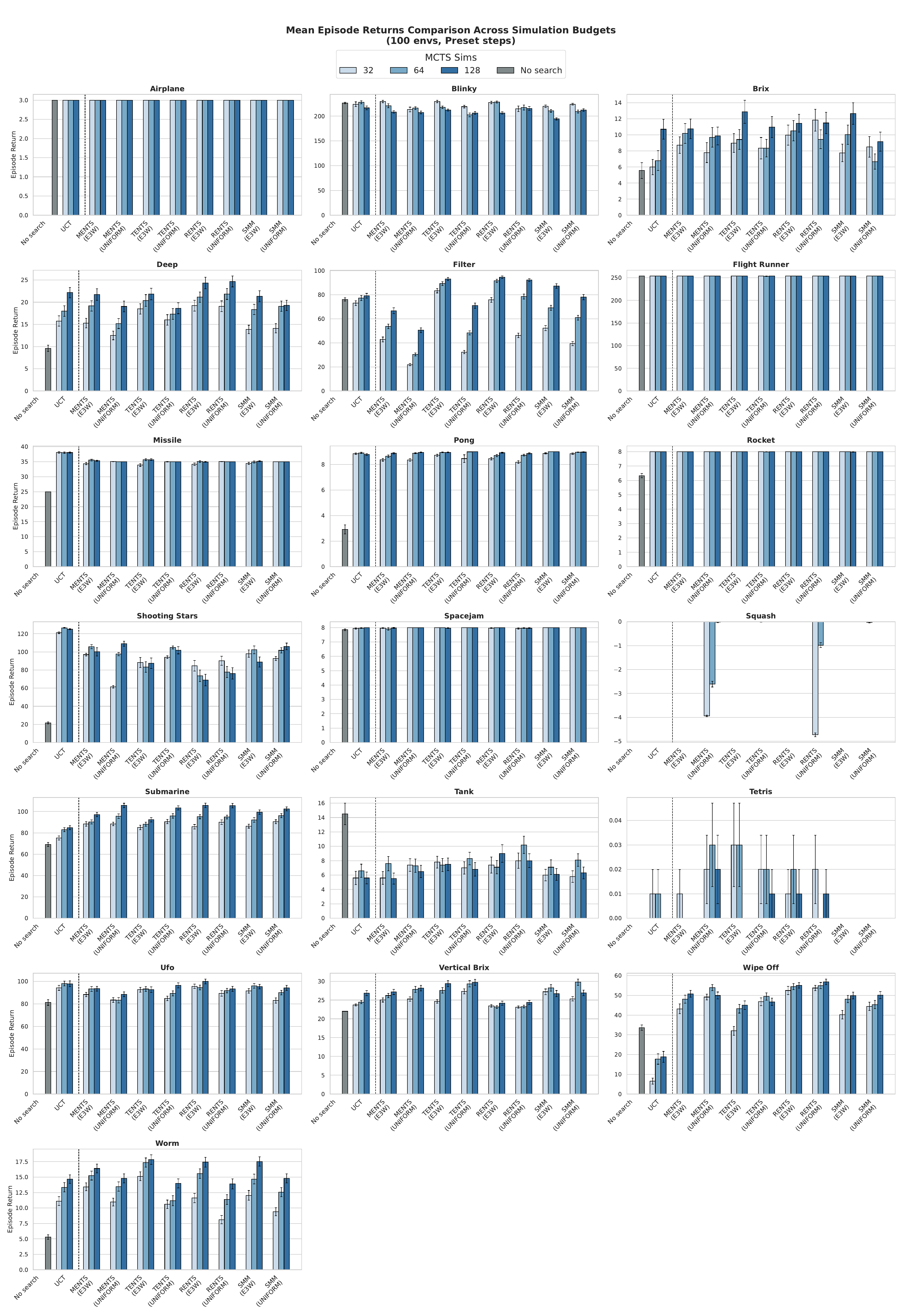}
        \caption{Episodic returns for all OCTAX tasks.}
        \label{fig:appendix_octax_episode_returns}
    \end{figure}

    \begin{figure}[ht]
        \centering
        \includegraphics[width=\linewidth]{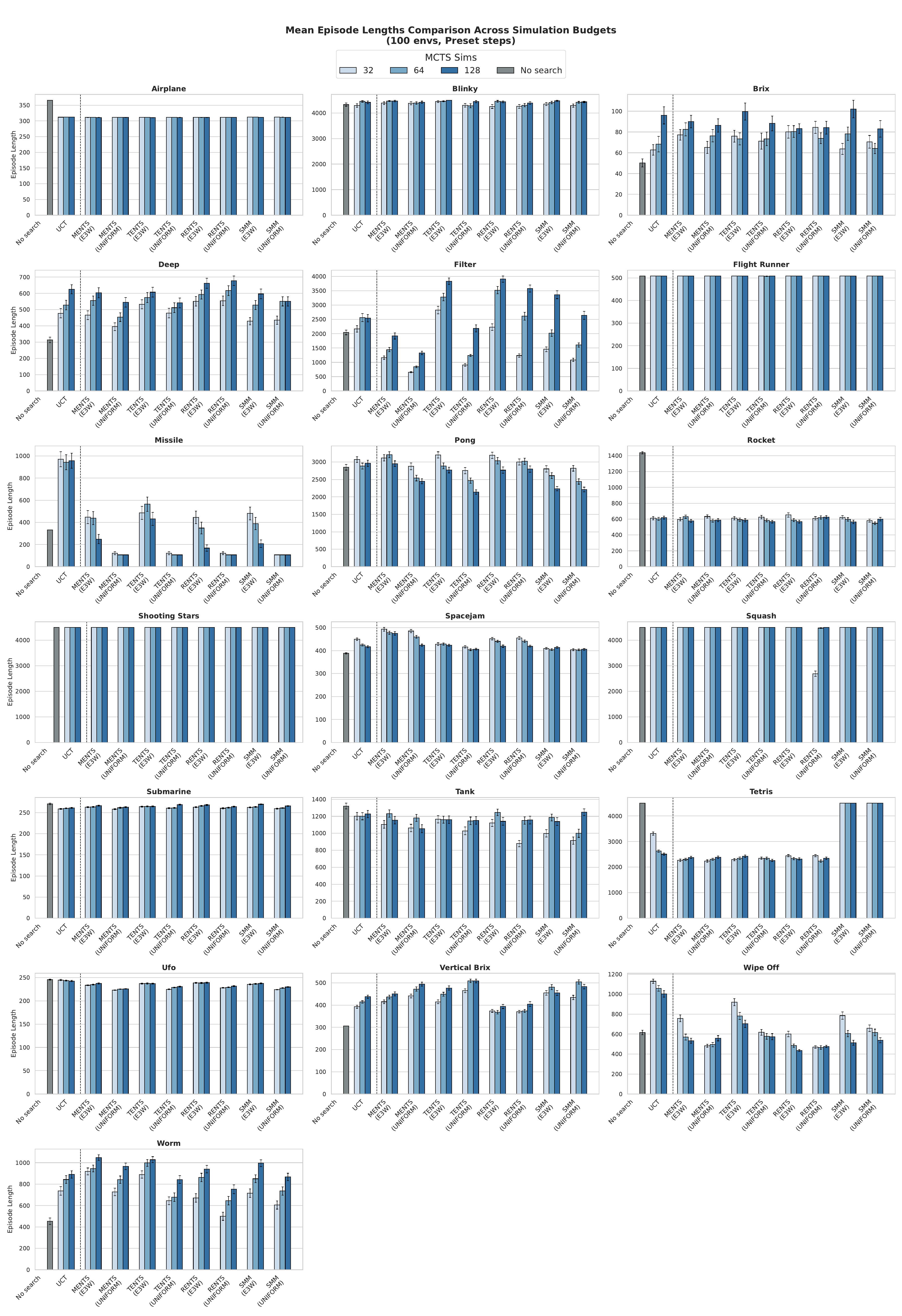}
        \caption{Episode lengths for all OCTAX tasks.}
        \label{fig:appendix_octax_episode_lengths}
    \end{figure}

    \begin{figure}[ht]
        \centering
        \includegraphics[width=\linewidth]{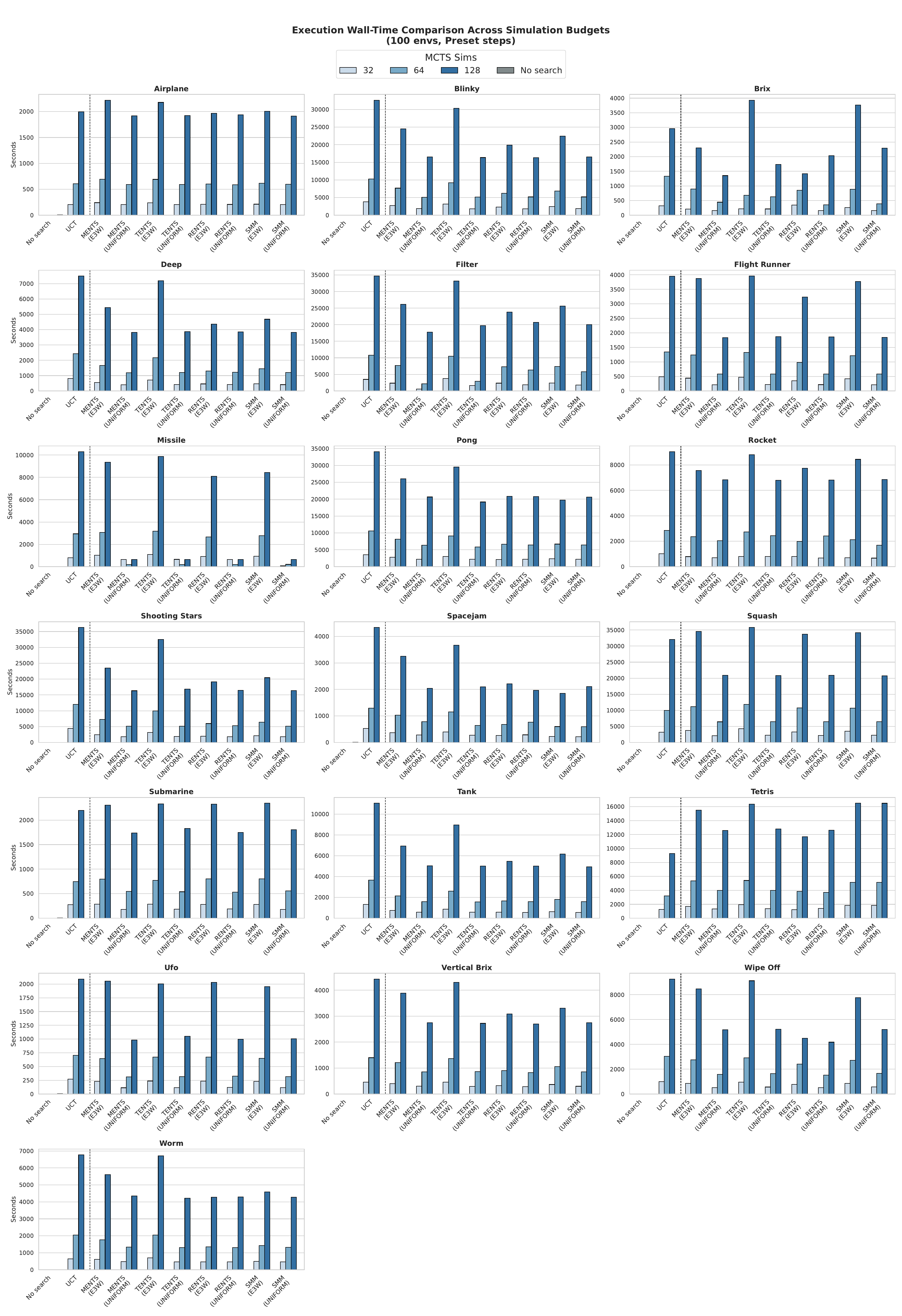}
        \caption{Execution times for all OCTAX tasks.}
        \label{fig:appendix_octax_execution_times}
    \end{figure}

\clearpage
\section{Details of the scientific discovery experiments} \label{appendix:details_sci_disc}

    \textbf{On the top-k metric}
    Because a policy that repeatedly regenerates one high-reward sequence is far less useful than one that discovers many distinct high-reward sequences, we report a diversity-aware top-$k$ reward: sequences are ranked by reward and greedily added to a selected set only if they are at least a task-specific Levenshtein distance $\delta$ from every sequence already selected ($\delta=10$ for AMP, $\delta=60$ for GFP, $\delta=13$ for UTR); we report the mean reward of the selected set at $k \in \{1, 10, 20, 50, 100\}$. This measures each method's ability to cover distinct modes of the reward landscape rather than only its best-case sample.

    \textbf{Details about the models.}
    We re-use the code from \citet{TGM_2025_jiralerspong} to train the proxy reward models and to train the TGM models. 
    For each task we train a causal sequence-Transformer policy (3 layers, 8 attention heads, 64-dimensional embeddings/hidden size, dropout 0.1, sinusoidal positional encoding) with a soft, regularization-consistent trajectory-balance-style objective; the network has no separate value head; its next-token logits are treated as prior Q-values for search, and the state value is reconstructed analytically from the same causal forward pass (summing the trained model's per-token credit along the prefix generated so far). Note that this last part is required for tree search since we need value estimates, not just Q-values estimates. 
    According to \citep[Table 2]{TGM_2025_jiralerspong}, we select the top-performing hyperparameters for each task, and use those to train our baselines: set $q = 0.75$ for UTR and AMP, and $q=0.25$ for GFP. 

    We did notice that our proxy reward model and TGM models did not give as high of numbers as in \citet{TGM_2025_jiralerspong}. This likely just boils down to implementation differences, different seeds being employed, or just differences between the trained proxy reward models, which are used in evaluation. 
    
    \textbf{Details about the tree search.}
    Rather than sweeping a single $\omega$, we run a 17-point sweep over $\omega \in \{0.025, 0.05, 0.1, 0.25, 0.5, 0.75, 1, 2, 4, 8, 10, 15, 20, 25, 30, 35, 40\}$ for every algorithm/simulation-budget configuration, so that we can separately report both peak achievable reward and sensitivity to this hyperparameter. Each configuration generates 100 sequences in parallel, repeated over 5 random seeds for the regularized algorithms and their no-search and UCT baselines. This sweep is done to generate a larger set of diverse sequences, so that we can evaluate the diversity of high-reward objects generated by each configuration. Note that this is standard procedure for these tasks in the GFlowNet-adjacent literature, see specifically \citet{trajectory_balance_GFN_2022} and \citet{TGM_2025_jiralerspong}. It is also standard to compute the mean reward of the top-100 diverse sequences according to the thresholds on the Levenshtein distance. Finally, we acknowledge that the sweep over UCT's temperature hyperparameter is not necessarily the best way to search a reward landscape, but we also note that UCT is not designed for such tasks, so adapting it to fit scientific discovery better is not within the scope of this investigation.

    Like in the OCTAX experiments, we set $\epsilon = 0.1$ with $\lambda_{s_t} = \frac{\epsilon K}{\log(N_t(s_t) + 2)}$ for E3W, like prior work.
    Tree depth is capped at 9 for all tasks.
    For GFP, whose much longer episodes make search at every generation step prohibitively expensive, we instead run full tree search only every 10th token and sample directly from the trained policy otherwise, restricting that task's simulation budget to 32.
    For sake of transparency, running MCTS for each of the 237 token of the GFP task took roughly 20-30 minutes for a single $\omega$ at just 32 simulations per token.
    Scaling this up to every $\omega$ we swept over, and each simulation count, the task would take far too much compute, a quick estimate states roughly a GPU-month of compute.  
    As such, our GFP results are limited by time and compute considerations. 
    Note that it is standard practice in the research literature to run MCTS as reduced intervals. \citet{xiao2019maximum_MENTS} did so in the beam rider task in Atari since the sequences were too long to run MCTS on each step.

\section{Additional results for the scientific discovery experiments} \label{appendix:addition_sci_disc_plots}

    \autoref{tab:top100_diverse_reward} lists the performance of all tree search methods and the no-search baseline for all 3 scientific discovery tasks. It is clear that uniform sampling performs better than E3W for almost all search methods, for almost all tasks. Uniform sampling with Tsallis entropy-based tree search (TENTS) is particularly performant in the scientific discovery tasks. 
    \begin{table}[t]
    \centering
    \caption{Top-100 diverse mean reward (mean $\pm$ std across seeds) for each search algorithm and action-selection rule, at each task's highest-performance simulation budget: AMP/UTR: 1024 simulations for most methods, other than UCT; GFP: 32 simulations for all methods. Bold marks a configuration whose mean exceeds the no-search baseline for that task. $\dagger$: fewer than 5 seeds had completed; mean and std are over the seeds available.}
    \label{tab:top100_diverse_reward}
    \begin{tabular}{lccc}
    \toprule
    Method & AMP & UTR & GFP \\
    \midrule
    Baseline (no search) & 10.17 $\pm$ 1.76 & 4.13 $\pm$ 0.02 & \textbf{1.61 $\pm$ 0.02} \\
    SMM (E3W) & 6.63 $\pm$ 0.23 & 3.47 $\pm$ 0.02 & 1.53 $\pm$ 0.01 \\
    SMM (Uniform) & 7.69 $\pm$ 0.74 & 3.66 $\pm$ 0.01 & 1.53 $\pm$ 0.01 \\
    MENTS (E3W) & 8.56 $\pm$ 0.48 & 3.65 $\pm$ 0.01 & 1.53 $\pm$ 0.01 \\
    MENTS (Uniform) & 9.13 $\pm$ 0.98 & 3.65 $\pm$ 0.01 & 1.53 $\pm$ 0.01 \\
    RENTS (E3W) & 7.57 $\pm$ 0.10$^{\dagger}$ & 4.12 $\pm$ 0.01 & 1.58 $\pm$ 0.01 \\
    RENTS (Uniform) & 7.57 $\pm$ 0.05 & \textbf{4.24 $\pm$ 0.01$^{\dagger}$} & 1.60 $\pm$ 0.01 \\
    TENTS (E3W) & 9.49 $\pm$ 0.07$^{\dagger}$ & 4.09 $\pm$ 0.02$^{\dagger}$ & \textbf{1.61 $\pm$ 0.00} \\
    TENTS (Uniform) & \textbf{10.64 $\pm$ 0.14} & \textbf{4.25 $\pm$ 0.01} & \textbf{1.62 $\pm$ 0.01} \\
    UCT & 7.94 $\pm$ 1.45 & \textbf{4.20 $\pm$ 0.02} & \textbf{1.61 $\pm$ 0.04} \\
    \bottomrule
    \end{tabular}
    \end{table}

    \autoref{fig:MCTS_sci_disc_perf_over_sims} shows the performance of the algorithms on AMP and UTR, as the number of simulations increases from 32 to 1024. Performance clearly improves with the number of simulations. We also clearly see that UCT doesn't finish at the 1024 mark for either of the tasks. 
    \begin{figure}
        \centering
        \includegraphics[width=\linewidth]{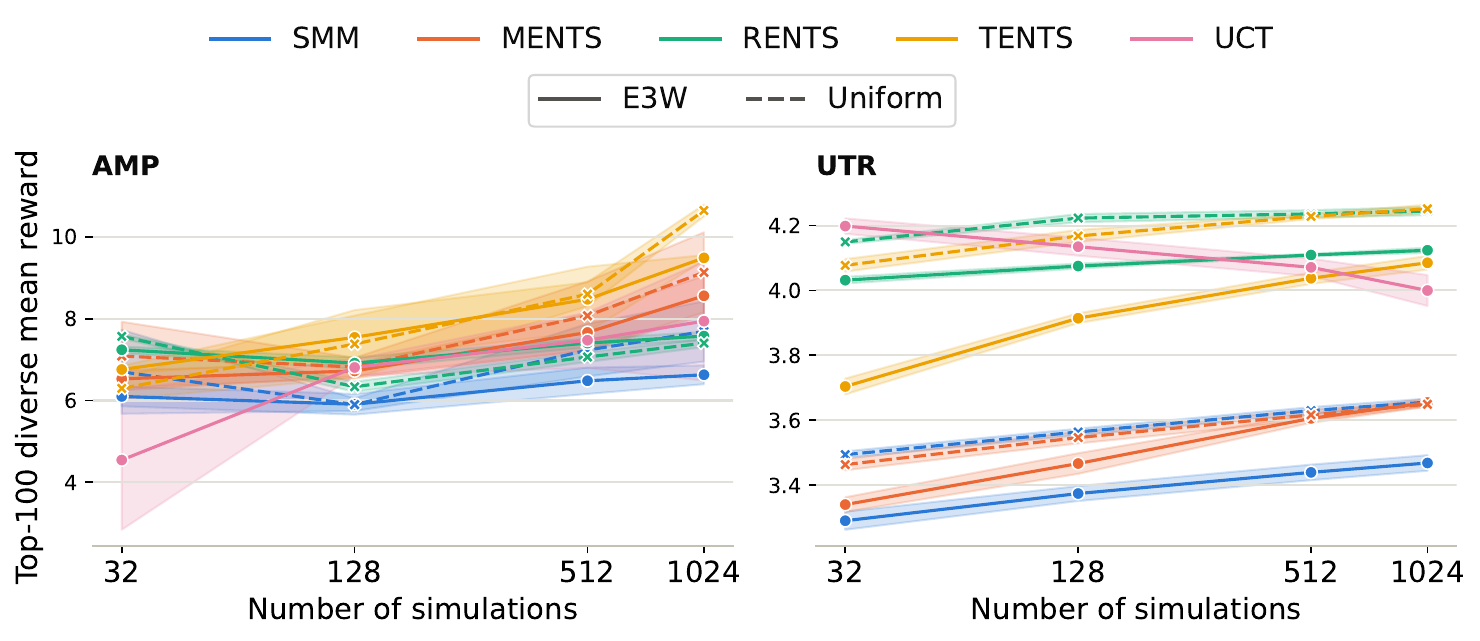}
        \caption{Performance of tree search in the AMP and UTR tasks, with increasing MCTS simulations per environment step. The top-100 diverse mean reward consistently increases for all tasks, with uniform sampling (dotted lines) often outperforming E3W. UCT did not finish at the 1024 simulation budget. The mean and standard deviation are computed over 5 random seeds of the generation process, sweeping inverse temperatures between $0.025$ to $40$. GFP is not presented on this figure because it takes too long to run with higher simulation budgets.}
        \label{fig:MCTS_sci_disc_perf_over_sims}
    \end{figure}

    \autoref{fig:MCTS_sci_disc_all_diverse_rewards} shows plots across top-K settings \textbf{for a single seed}. It is clear that the baseline model is very strong at lower values of K: 1, 10, and 20 are virtually unbeaten. For top-50 and top-100 mean rewards, it is clear that regularised MCTS algorithms are able to produce higher average rewards, especially as the number of simulations increases. AMP is the only task where no tree search method seems to clear the maximum average reward, but as is visible in the box-and-whisker plots of \autoref{fig:MCTS_sci_disc_box_plots}, TENTS and RENTS are able to achieve competitive, or higher, rewards at 1024 simulations per step. Like in \autoref{fig:MCTS_sci_disc_perf_over_sims}, we see that performance improves as sims increase. Finally, it is interesting to note that uniform sampling often outperforms E3W in UTR, but not in AMP, except at 1024 simulations. We do not have an explanation of this phenomenon.

    \begin{figure}[ht]
        \centering
        \includegraphics[width=\linewidth]{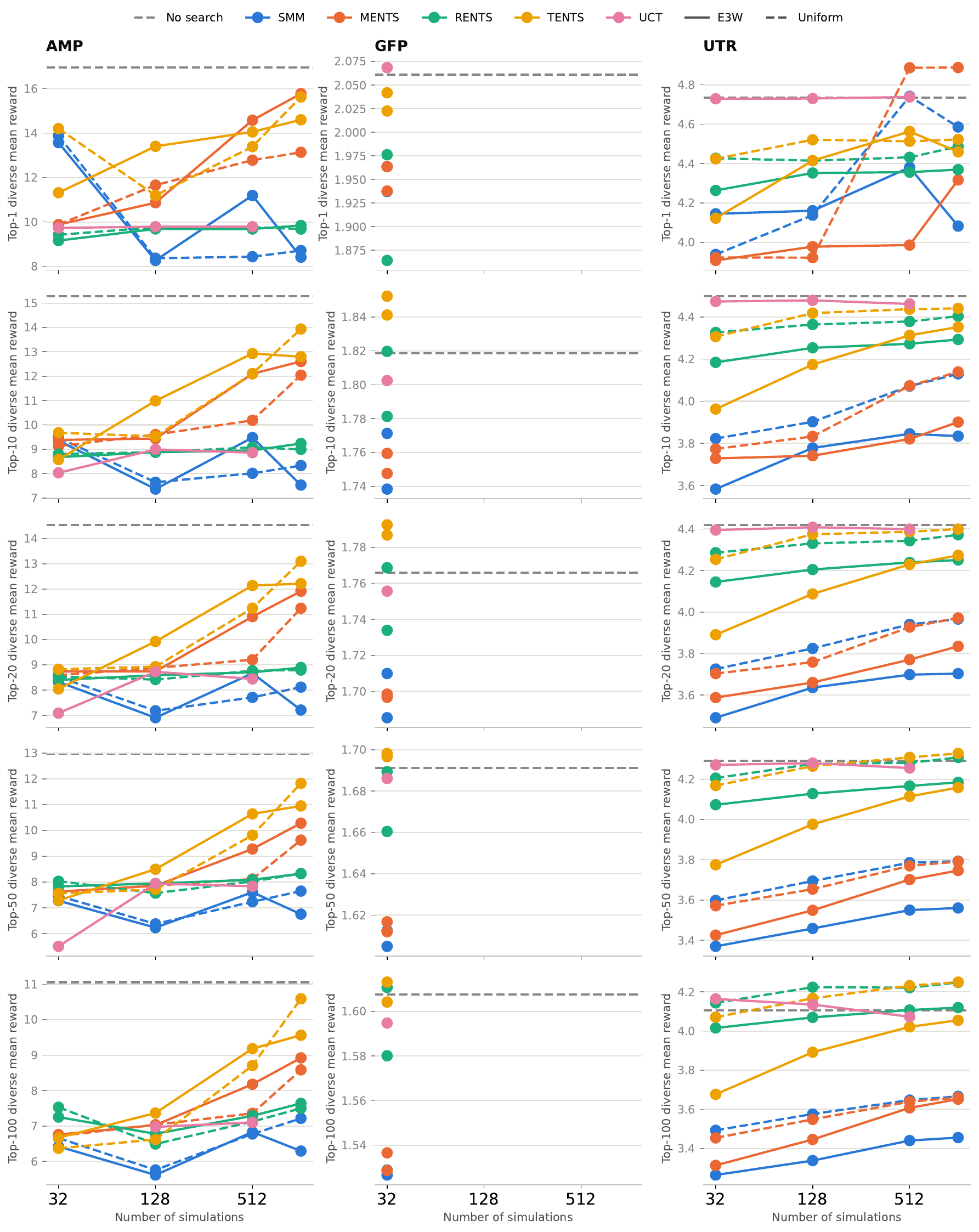}
        \caption{Performance of tree search across all scientific discovery tasks at different top$-k$ levels, with increasing simulation budgets.}
        \label{fig:MCTS_sci_disc_all_diverse_rewards}
    \end{figure}
    
\end{document}